\pdfoutput=1
\documentclass{article}
\usepackage{arxiv}
\usepackage[T1]{fontenc}
\usepackage{microtype}
\usepackage[breakable,skins]{tcolorbox}
\makeatletter
\g@addto@macro\normalsize{%
  \setlength\abovedisplayskip{10pt plus 3pt minus 3pt}\setlength\belowdisplayskip{10pt plus 3pt minus 3pt}%
  \setlength\abovedisplayshortskip{6pt plus 3pt}\setlength\belowdisplayshortskip{8pt plus 3pt minus 2pt}}
\makeatother
\newtcolorbox{keyresult}{enhanced jigsaw, breakable, width=0.95\linewidth, center, colback=white, colframe=black, boxrule=0.8pt, sharp corners,
  left=12pt, right=12pt, top=8pt, bottom=8pt, toprule at break=0pt, bottomrule at break=0pt,
  before skip=12pt, after skip=12pt, before upper={\setlength{\topsep}{0pt}\setlength{\partopsep}{0pt}}}
\renewcommand{\shorttitle}{Stability of Outer Nesterov Momentum}

\usepackage[round]{natbib}

\usepackage{graphicx}   %
\usepackage{xurl}       %
\usepackage{algorithm}   %
\usepackage{algpseudocode}
\usepackage{tikz}   %
\usetikzlibrary{tikzmark,calc}
\algrenewcommand{\algorithmiccomment}[1]{\hfill\makebox[4cm][l]{$\triangleright$\ #1}}   %
\newcommand{\phasebracket}[4]{%
  \draw[thick] ($(pic cs:#1)+(8pt,8pt)$) -- ++(3pt,0) coordinate (t)
    -- ($(t|-{pic cs:#2})+(0,-2pt)$) coordinate (b) -- ++(-3pt,0);
  \draw[thick,-latex] ($(t)!0.5!(b)$) -- ++(9pt,0) coordinate (e);
  \node[circle,fill=black,text=white,inner sep=1pt,font=\scriptsize\bfseries,anchor=west] (c) at (e) {#3};
  \node[anchor=west,align=left,font=\small,inner sep=1pt] at ($(c.east)+(2pt,0)$) {#4};
}
\usepackage{hyperref}
\usepackage{url}
\usepackage{bm} %
\usepackage{booktabs}
\usepackage{mathtools}
\usepackage{amssymb}
\usepackage{xcolor}
\usepackage{longtable}
\usepackage{colortbl}   %
\usepackage{xltabular}   %
\newlength{\tabtermwidth}   %
\usepackage{placeins}   %
\usepackage[normalem]{ulem}   %
\definecolor{trianglered}{HTML}{D96868}   %
\definecolor{purple}{HTML}{8000C8}        %
\definecolor{uncgreen}{HTML}{00A040}      %
\colorlet{CYAN}{cyan}                      %
\definecolor{citegreen}{HTML}{064F43}
\definecolor{drawteal}{HTML}{17937C}
\definecolor{covlight}{HTML}{E4DEEF}
\definecolor{unstablered}{HTML}{A30F37}
\hypersetup{colorlinks=true,allcolors=black,citecolor=citegreen,pdftitle={Understanding Stability of Outer Nesterov Momentum on Stochastic Pseudo-Gradients},pdfauthor={Nishant Sharma, Laura Gomezjurado Gonzalez, Ganesh Talluri, Sarath Chandar, Hiroki Naganuma}}

\newcommand{\E}{\mathbb{E}}

\newcommand{\Cov}{\operatorname{Cov}}
\newcommand{\Var}{\operatorname{Var}}
\newcommand{\vecop}{\operatorname{vec}}
\newcommand{\kron}{\otimes}
\newcommand{\opnorm}[1]{\left\lVert#1\right\rVert}
\newcommand{\T}{\mathsf{T}}

\newcommand{\Nk}{\mathcal{N}_K}
\newcommand{\dcrit}{d_{2,\mathrm{crit}}}

\newtheorem{theorem}{Theorem}
\newtheorem{proposition}[theorem]{Proposition}
\newenvironment{proof}{\par\noindent\textit{Proof.}\ }{\hfill$\square$\par}
\newtheorem{corollary}[theorem]{Corollary}
\newtheorem{example}[theorem]{Example}   %
\newtheorem{assumption}{Assumption}   %
\newtheorem{definition}{Definition}
\makeatletter
\def\@opargbegintheorem#1#2#3{\trivlist
      \item[\hskip \labelsep{\bfseries #1\ #2}]{\bfseries (#3)}\hskip\labelsep\itshape}
\makeatother

\newcommand{\Lop}{\mathcal{L}}   %
\newcommand{\Vop}{\mathcal{V}}   %
\newcommand{\Gop}{\mathcal{G}}   %
\newcommand{\Qop}{\mathcal{Q}}   %
\newcommand{\Bm}{\mathsf{B}}     %
\newcommand{\Em}{\mathsf{E}}     %
\newcommand{\Sm}{\mathsf{S}}     %
\newcommand{\smom}{\mathsf{s}}   %
\newcommand{\Sop}{\mathcal{S}}   %
\newcommand{\Cop}{\mathcal{C}}   %

\begin{document}

\title{Understanding Stability of Outer Nesterov Momentum on Stochastic Pseudo-Gradients}
\author{%
  \textbf{Nishant Sharma}$^{1,2,3}$ \quad \textbf{Laura Gomezjurado Gonzalez}$^{4}$ \quad \textbf{Ganesh Talluri}$^{5}$\\
  \textbf{Sarath Chandar}$^{1,2,6,7}$ \quad \textbf{Hiroki Naganuma}$^{2}$\thanks{This work was conducted while Hiroki Naganuma was at Mila.}\\[1.5ex]
  $^{1}$Chandar Research Lab \quad $^{2}$Mila -- Quebec AI Institute \quad $^{3}$Universit\'e de Montr\'eal\\
  $^{4}$Stanford University \quad $^{5}$Arizona State University \quad $^{6}$Polytechnique Montr\'eal\\
  $^{7}$Canada CIFAR AI Chair\\[1ex]
  {\small\texttt{nishant.sharma@mila.quebec}, \texttt{lpgomez@stanford.edu}, \texttt{talluri.ganeshsai@gmail.com},}\\
  {\small\texttt{sarath.chandar@mila.quebec}, \texttt{naganuma.hiroki@mila.quebec}}
}
\date{}
\maketitle

\begin{abstract}
Large neural networks are commonly trained across many workers, and communication between them is a major cost. Workers therefore take several local optimization steps before sending their net parameter displacements, or \textit{pseudo-gradients}, to an outer optimizer, which often applies Nesterov momentum. We study how the outer momentum, the number of local steps and the averaging over workers control the training dynamics, and in particular their stability.  For random quadratics, we derive an exact mean-square stability criterion in any dimension, which reveals that the inner batch size determines whether increasing the number of local steps eventually stabilizes or destabilizes the outer iteration. Moreover, while averaging the pseudo-gradients of independent workers can make larger outer momentum values admissible, the number of workers required for stability can grow exponentially with the number of local steps. Finally, we show that the mean-square dynamics governing stability can be probed directly at intermediate neural-network checkpoints using shared-minibatch perturbations. We test our theoretical findings across networks, and measure their consequences for training loss.

\end{abstract}
\vspace{1.5em}
\begin{figure}[H]
\centering
\includegraphics[width=0.73\linewidth]{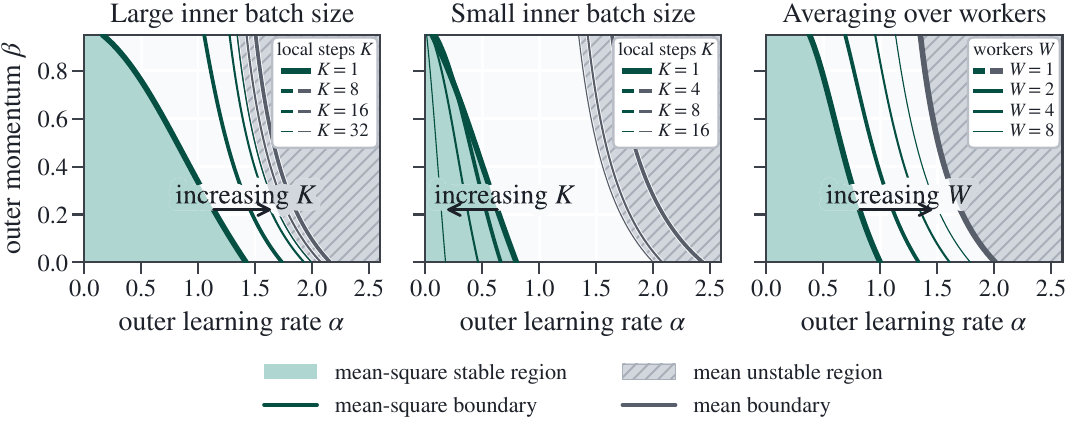}   %
\caption{\textbf{Mean-square stable regions in $(\alpha,\beta)$.} For inner batch size $b$ and inner-batch threshold $b_c$ (as defined in Section~\ref{sec:cons-K}), more local steps enlarge the mean-square stable region when $b>b_c$ (\uline{left}), and shrink it when $b<b_c$ (\uline{middle}). Averaging the pseudo-gradient over $W$ workers enlarges the mean-square stable region (\uline{right}), but never past the mean boundary (gray), which averaging leaves unchanged.}
\label{fig:phase}
\end{figure}
\clearpage

\addtocontents{toc}{\protect\setcounter{tocdepth}{-1}}

\section{Introduction}
\label{sec:intro}

Deep neural networks, and large language models (LLMs) in particular, are now trained on many workers, often spread across compute nodes or even data centers. In such distributed training, each worker holds a copy of the model, and gradients are exchanged at every step to keep the copies in sync. This communication has become one of the main drivers of training costs at scale \citep{douillard2023diloco}. A common solution is to let each worker train locally for $\boldsymbol{K}$ optimization steps and send only its displacement, the \textit{pseudo-gradient}, to an outer update \citep{mcmahan2017fedavg}. This outer update often applies Nesterov momentum \citep{nesterov1983method}, which has been found to accelerate training on one or many workers \citep{kallusky2025snoo,agarwala2026twophase} and to outperform SGD and Adam as the outer optimizer \citep{douillard2023diloco}.

Concretely, at outer step $t$, each of $\boldsymbol{W}$ workers starts from the shared parameters $\theta_t$ and runs $K$ steps of minibatch stochastic gradient descent (SGD) with batch size $b$ and inner learning rate $\eta$ (its \textit{inner trajectory}). The outer optimizer averages the workers' pseudo-gradients and applies Nesterov momentum to the average, with outer learning rate $\alpha$ and outer momentum $\beta$ (Algorithm~\ref{alg:outer}).
\begin{algorithm}[!ht]
\caption{Two-phase optimization with outer Nesterov momentum}
\label{alg:outer}

\begin{minipage}{0.83\linewidth}
\begin{algorithmic}[1]
\Require start $\theta_0$, inner rate $\eta$, inner batch size $b$, outer rate $\alpha>0$, outer momentum $0\le\beta<1$, number of local steps $K$, number of workers $W$
\State $u_{-1}\gets0$
\For{$t=0,1,2,\ldots$}
  \For{$w=1,\ldots,W$} \Comment{in parallel}\tikzmark{algOuterInA}
    \State \parbox[t]{0.66\textwidth}{run $K$ inner SGD steps at rate $\eta$ with batch size $b$ from $\theta_t$ on fresh minibatches, ending at $\theta_{t,K}^{(w)}$}
  \EndFor\hfill\tikzmark{algOuterInB}
  \State $P_t\gets\theta_t-\frac1W\sum_{w=1}^W\theta_{t,K}^{(w)}$ \Comment{pseudo-gradient}\tikzmark{algOuterOutA}
  \State $u_t\gets\beta u_{t-1}+P_t$ \Comment{buffer updated first}
  \State $\theta_{t+1}\gets\theta_t-\alpha\bigl(\beta u_t+P_t\bigr)$ \Comment{uses the updated buffer}\tikzmark{algOuterOutB}
\EndFor
\end{algorithmic}
\end{minipage}
\begin{tikzpicture}[remember picture,overlay]
\phasebracket{algOuterInA}{algOuterInB}{1}{Inner\\phase}
\phasebracket{algOuterOutA}{algOuterOutB}{2}{Outer\\step}
\end{tikzpicture}
\end{algorithm}

Several familiar training methods fit this two-phase template. Local SGD and federated averaging \citep{mcmahan2017fedavg,stich2019local,lin2020localsgd} are the simplest case, $\alpha=1$ and $\beta=0$, in which the outer step just averages the workers' parameters. Lookahead \citep{zhang2019lookahead} keeps $\beta=0$ but takes $\alpha<1$ on a single worker, and SNOO \citep{kallusky2025snoo} adds the outer Nesterov step, still on a single worker. DiLoCo \citep{douillard2023diloco} and, more recently, MuLoCo \citep{therien2025muloco} apply the outer Nesterov step across workers using AdamW \citep{kingma2015adam,loshchilov2019adamw} and Muon \citep{jordan2024muon} as inner optimizers, respectively. SlowMo \citep{wang2020slowmo} also adds momentum to the outer step, adaptive federated optimizers \citep{reddi2021adaptive} use an adaptive outer optimizer, semi-synchronous variants target high-latency links \citep{naganuma2025palsgd}, and asynchronous ones tolerate stragglers and hardware failures \citep{douillard2026decoupled}. 

These methods, which have been scaled to language models with billions of parameters \citep{charles2025scaling,douillard2025streaming}, have many interacting hyperparameters, such as the inner and outer optimizers, the outer momentum, the learning rates, the batch size and the number of local steps $K$, and today they are set by tuning. SNOO, for example, finds that outer momentum $0.9$ is more sensitive to $K$ than smaller values, and that suboptimal outer hyperparameters can substantially degrade performance \citep{kallusky2025snoo}. Among these choices, the number of local steps $K$ is special, because it influences the variance of the pseudo-gradient (Section~\ref{sec:block-moments}), and this variance grows with the error left by earlier outer steps, so it does not average out over the outer steps. As a result, the error can shrink on average while its mean square grows without bound, so an analysis of the mean alone can wrongly predict stability, as in the region between the green and grey boundaries of Figure~\ref{fig:phase}. Moreover, the stability of this two-phase setup is less understood than that of standard SGD \citep{wu2018sgdselection,ma2021linear,granziol2022batchsize,velikanov2023view,ziyin2023saddles,mulayoff2024exact}, because the outer optimizer acts on a stochastic $K$-step pseudo-gradient rather than on a single stochastic gradient. In this work, we give a theory that explains how changing the outer momentum, the number of local steps and the averaging over workers controls the training dynamics, and in particular their mean-square stability, which provides both a deeper mechanistic understanding of these methods and principled guidance for choosing their hyperparameters.

Our contributions are as follows.
\begin{enumerate}
\item For random quadratics, we derive a necessary and sufficient condition for the mean-square stability of outer Nesterov momentum on pseudo-gradients from \(K\) stochastic inner SGD steps (Theorem~\ref{thm:full}).
\item From it we derive four implications: (i) at a fixed $\alpha$, the noise budget can only shrink as $\beta$ grows (Proposition~\ref{prop:momentum}); (ii) the inner batch size decides whether large $K$ is eventually stable or unstable (Proposition~\ref{thm:blocklength}; Figure~\ref{fig:phase}); (iii) averaging $W$ workers divides the variance only by $W$, and the number of workers needed for stability can grow exponentially in $K$ (Corollary~\ref{cor:averaging}; Figure~\ref{fig:phase}); (iv) for fixed total batch size budget, splitting it across workers cannot improve mean-square stability (Proposition~\ref{prop:totalbatch}).

\item We measure stability at an intermediate neural-network checkpoint without knowing a minimizer, by tracking whether a small difference between two runs that share their minibatches grows or shrinks. We apply this measurement at several checkpoints of an MLP and a ResNet-20, both trained on CIFAR-10, and of Pythia language models from 70m to 1b parameters, and test the trends that the four implications predict (Figure~\ref{fig:ckpt-amp}). We also measure the consequences of instability for the training loss (Section~\ref{sec:checkpoint-measure}; Figure~\ref{fig:ckpt-spread-impl}).
\end{enumerate}

\section{Related Work}
\label{sec:related}

\citet{agarwala2026twophase} gives a theory of two-phase optimizers, deriving the second-moment dynamics of Lookahead and LA-DiLoCo on high-dimensional linear regression, and shows in a deterministic analysis that outer Nesterov momentum accelerates convergence on an effective spectrum set by the inner optimizer. \citet{topollai2026restart} show in language-model pretraining that the stable range of outer learning rates and momentum values changes with the communication period, and use deterministic mode dynamics to motivate restarts of the outer momentum that widen this range. \citet{khaled2025outer} prove convergence guarantees for local SGD with an outer learning rate, outer momentum and outer acceleration. However, their worst-case guarantees depend on the gradient noise only through a uniform variance bound, so two problems with the same bound can receive the same guarantee even when one is stable and the other is not (illustrated in Appendix~\ref{app:bounded-noise-example}).

Mathematically, our analysis is closest to \citet{mulayoff2024exact}, who derive an exact mean-square threshold for minibatch SGD. They connect the mean-square stability of SGD to stationary covariance and expected loss. We extend this perspective to outer momentum acting on an entire stochastic trajectory. While their analysis is around a minimum, through shared-minibatch perturbations we also investigate these dynamics at intermediate neural-network checkpoints. Their threshold continues a line of work that uses linear stability to characterize which minima SGD can reach \citep{wu2018sgdselection,ma2021linear,wu2022alignment}. Stochastic momentum is known to be sensitive to multiplicative noise \citep{gupta2024nesterov}. To our knowledge, however, no exact mean-square criterion was available for outer momentum on a stochastic pseudo-gradient. Our criterion is the spectral radius of a map on second moments, which is the form of mean-square stability for jump-linear systems \citep{costa2005mjls}.

A separate line of work studies how training behaves at its stability boundary. Under full-batch gradient descent, the sharpness of a neural network typically rises to $2/\eta$ and hovers there, at the edge of stability \citep{cohen2021eos,damian2023selfstabilization}. Minibatch SGD, with and without momentum, trains at an analogous edge of stochastic stability \citep{andreyev2024edge,liao2026sharpnessgap,andreyev2026momentum}. Among two-phase optimizers, \citet{khandelwal2026scaffold} observe edge-of-stability dynamics in federated averaging, the case $\alpha=1$ and $\beta=0$. Whether training with outer momentum also operates at the edge of stability is, to our knowledge, open (Section~\ref{sec:limitations}).

\section{Setup}
\label{sec:setup}
Here, we make the setup from the introduction precise. With a single worker per outer step (averaging over $W$ workers in Section~\ref{sec:cons-W}), the pseudo-gradient comes from $K$ steps that start at $\theta$ and end at $\theta_{K,\omega}$, where $\omega$ denotes their minibatches $\xi_0,\ldots,\xi_{K-1}$, each of $b$ samples. Written as a function of the starting point $\theta$,
\begin{equation}
P_{K,\omega}(\theta)=\theta-\theta_{K,\omega},
\label{eq:pseudo-gradient}
\end{equation}
we call this map from $\theta$ to its displacement a \textit{block}, and write $P_\omega$ for it when $K$ is fixed (Figure~\ref{fig:setup}(a)). At outer step $t$, the pseudo-gradient is this map evaluated at the current parameters, $P_t=P_{\omega_t}(\theta_t)$, where the minibatches $\omega_t$ are drawn independently at every outer step, and the outer update of Algorithm~\ref{alg:outer} is
\begin{equation}
u_t=\beta u_{t-1}+P_t,\qquad \theta_{t+1}=\theta_t-\alpha\,(\beta u_t+P_t),
\label{eq:snoo}
\end{equation}
with the momentum buffer started at $u_{-1}=0$. Throughout, $\alpha>0$ and $0\le\beta<1$.

\begin{figure}[t]
\centering
\begin{minipage}[t]{0.49\textwidth}\vspace{0pt}\centering
\includegraphics[trim=0 100.1 0 0,clip]{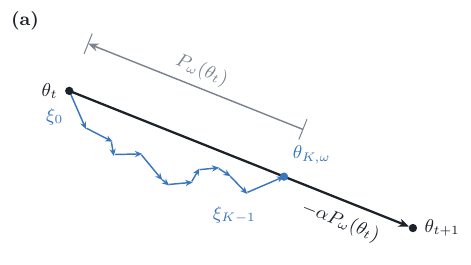}
\end{minipage}\hfill
\begin{minipage}[t]{0.49\textwidth}\vspace{0pt}\centering\setlength{\parskip}{0pt}
\includegraphics[trim=0 79 0 132.5,clip]{figures/fig_setup_arxiv.pdf}\par\vspace{14pt}
\includegraphics[trim=46.5 12.75 4.75 155,clip,scale=1.11]{figures/fig_setup_arxiv.pdf}
\end{minipage}
\caption{\textbf{The pseudo-gradient of $K$ inner SGD steps, and the block operator.} (\uline{a}) From $\theta_t$ the inner trajectory takes $K$ SGD steps on minibatches $\xi_0,\ldots,\xi_{K-1}$ and ends at $\theta_{K,\omega}$. The pseudo-gradient $P_\omega(\theta_t)$ is the displacement from $\theta_{K,\omega}$ to $\theta_t$. The outer step moves $\theta_t$ by $-\alpha P_\omega(\theta_t)$ to $\theta_{t+1}$. (\uline{b}) The block operator $D_{K,\omega}$ maps a displacement $\delta$ of the starting point to the change $D_{K,\omega}\delta$ of the pseudo-gradient. Each \textcolor{drawteal}{green} arrow is the change $D_{K,\omega}\delta$ for one draw of the minibatches $\omega$. The draws have mean $\bar D_K\delta$, and the {\setlength{\fboxsep}{1pt}\colorbox{covlight}{shaded}} ellipse marks their covariance. Mean-square stability depends on $D_{K,\omega}$ only through $\bar D_K$ and the second moment $\Sm_K$ of Section~\ref{sec:block-moments}.}
\label{fig:setup}
\end{figure}

Initially, through Section~\ref{sec:setup-model} we will assume that the minibatch losses share a minimizer $\theta^*$ (an \textit{interpolating minimum} in the terminology of \citealt{mulayoff2024exact}) and as a part of the stability analysis, we track the error $e_t=\theta_t-\theta^*$. With a shared minimizer the error obeys a homogeneous recurrence, so stability means that the error vanishes. Section~\ref{sec:checkpoint-jacobian} drops this assumption, which makes the minimizer of the full loss a \textit{regular minimum} in the terminology of \citet{mulayoff2024exact}. The error then no longer vanishes, but, importantly, the same stability criterion still decides whether it stays bounded.

\subsection{Block operator and stability}
\label{sec:setup-model}
\label{sec:setup-stability}

We model the minibatch losses as quadratics with a common minimizer,
\begin{equation}
\begin{gathered} %
\ell_\xi(\theta)=\tfrac12(\theta-\theta^*)^\top H_\xi(\theta-\theta^*),
\\ H_\xi=H_\xi^\top\succeq0,
\end{gathered}
\label{eq:quadratic-model}
\end{equation}
where $H_\xi$ is the Hessian of the minibatch loss $\ell_\xi$. Since an SGD step at rate $\eta$ sends the error $e$ to $(I-\eta H_\xi)e$, after $K$ steps the error is multiplied by $M_{K,\omega}=Q_{\xi_{K-1}}\cdots Q_{\xi_0}$ with $Q_\xi=I-\eta H_\xi$, so the pseudo-gradient is linear in the error,
\begin{equation}
\begin{gathered} %
P_{K,\omega}(\theta)=D_{K,\omega}(\theta-\theta^*),
\\
D_{K,\omega}=I-M_{K,\omega}.
\end{gathered}
\label{eq:block-linear}
\end{equation}
We call the map $D_{K,\omega}$ the \textit{block operator}. In other words, the inner trajectory replaces the curvature of the loss with the block operator, on which the outer Nesterov step acts. In the deterministic setting its spectrum is the effective spectrum of \citet{agarwala2026twophase}. Note that the block operator is random, because it depends on the minibatches $\omega$, although not on $\theta$, and for $K\ge2$ it is in general not symmetric, so the pseudo-gradient need not be the gradient of any loss. Since it does not depend on $\theta$, displacing the starting point by $\delta$ changes the pseudo-gradient by exactly $D_{K,\omega}\delta$ (Figure~\ref{fig:setup}(b)), which is the reading that Section~\ref{sec:checkpoint-jacobian} carries over to networks.

Substituting \eqref{eq:block-linear} into \eqref{eq:snoo} gives linear recurrence for $z_t=(e_t,u_{t-1})$, the error and the momentum buffer that enter step $t$,
\begin{equation}
\begin{gathered} %
z_{t+1}=\T(D_{K,\omega_t})\,z_t,
\\
\T(D)=
\begin{bmatrix}
I-\alpha(1+\beta)D & -\alpha\beta^2 I\\
D & \beta I
\end{bmatrix}.
\end{gathered}
\label{eq:transition}
\end{equation}

\begin{assumption}[i.i.d.\ block operators]
\label{ass:block}
The block operators $D_t:=D_{K,\omega_t}$, $t\ge0$, of different outer steps are jointly independent and identically distributed (i.i.d.), and independent of the initial state $z_0$. Both have finite second moments, $\E\opnorm{D_t}_F^2<\infty$ and $\E\opnorm{z_0}^2<\infty$.
\end{assumption}

In the model, the assumption holds whenever every outer step draws its $K$ minibatches afresh.\footnote{Independence of the minibatches within a block is used for the block moments of Section~\ref{sec:block-moments}. Propositions~\ref{thm:blocklength} and~\ref{prop:totalbatch} and Corollary~\ref{cor:averaging}(b) use these moments.}

\begin{definition}[Stability]
\label{def:mss}
The recurrence \eqref{eq:transition} is \emph{mean stable} if $\E[z_t]\to0$, and \emph{mean-square stable} if $\E\opnorm{z_t}^2\to0$, for every initial state admitted by Assumption~\ref{ass:block}.
\end{definition}

As one would expect, mean-square stability implies mean stability, because $\opnorm{\E[z_t]}\le(\E\opnorm{z_t}^2)^{1/2}$, but not conversely (Figure~\ref{fig:mean-vs-ms} in Appendix~\ref{app:mean-vs-ms} gives an example).\footnote{If every $H_\xi$ vanishes on a direction, every block leaves it unchanged, and an error along it neither grows nor decays. We therefore carry out the analysis on the orthogonal complement of such common flat directions.}

\subsection{Block moments and how $K$ and $b$ enter them}
\label{sec:block-moments}

Define two averages of the block operator, its mean $\bar D_K=\E[D_{K,\omega}]$ and its second moment $\Sm_K=\E[D_{K,\omega}\kron D_{K,\omega}]$, with $\kron$ the Kronecker product. We will refer to them as the \textit{block moments}. As Section~\ref{sec:theory} shows, the stability criterion depends on the block operator only through these block moments. In the model, both are built from the mean and the second moment of a single inner step (derivation in Appendix~\ref{app:proof-moments}). If the $K$ one-step maps of a block are i.i.d.\ with $\bar Q=\E[Q_\xi]$ and $\mathcal Q=\E[Q_\xi\kron Q_\xi]$, then
\begin{equation}
\begin{gathered} %
\bar D_K=I-\bar Q^K,
\\
\Sm_K=I\kron I-I\kron\bar Q^K-\bar Q^K\kron I+\mathcal Q^K.
\end{gathered}
\label{eq:block-moments}
\end{equation}
Since $(Q_\xi\kron Q_\xi)\vecop(Y)=\vecop(Q_\xi YQ_\xi)$, the matrix $\Qop$ acts on a symmetric matrix $Y$ exactly as the map $Y\mapsto\E[Q_\xi YQ_\xi]$ does, and we use both forms.

Equation~\eqref{eq:block-moments} is precisely where the number of local steps enters. The mean block operator is governed by $\bar Q^K$ alone, whereas the second moment also contains $\Qop^K$, and the two powers vary at different rates in $K$, so an analysis using only $\bar D_K$ does not account for how the noise changes with $K$.

\paragraph{The batch size.} For $b$ samples drawn with replacement from sample curvatures $h_i$ with mean $\bar H$, so that $H_\xi$ is the average of the $b$ sampled $h_i$, the one-step mean $\bar Q=I-\eta\bar H$ does not depend on $b$ and the one-step second-moment map splits as
\begin{equation}
\begin{gathered} %
\Qop_b=\Qop_\infty+\tfrac1b\,\Cop,
\qquad
\Qop_\infty(Y)=\bar QY\bar Q,
\\
\Cop(Y)=\eta^2\,\E_i\bigl[(h_i-\bar H)\,Y\,(h_i-\bar H)\bigr],
\end{gathered}
\label{eq:batch-split}
\end{equation}
with $\Qop_\infty$ the large-batch limit and $\Cop$ positive\footnote{A linear map on symmetric matrices is \emph{positive} if it preserves positive semidefiniteness, and $\mathcal M\preceq\mathcal N$ means that $\mathcal N-\mathcal M$ is positive.} (derivation in Appendix~\ref{app:toy}). The batch therefore enters only through the noise term, as the factor $1/b$.
\section{The Stability Boundary}
\label{sec:theory}

\subsection{Second-moment map and the block noise}
\label{sec:full}

Under Assumption~\ref{ass:block}, the first two moments of $z_{t+1}$ depend only on those of $z_t$. With $D_K$ drawn from the common distribution of the $D_t$, define the \textit{second-moment map} on symmetric matrices
\begin{equation}
\Lop_K(X)=\E\bigl[\T(D_K)\,X\,\T(D_K)^\top\bigr].
\label{eq:lift}
\end{equation}
Then, with $A=\T(\bar D_K)$ the \textit{mean transition} and $\Sigma_t=\E[z_tz_t^\top]$ the second-moment matrix of $z_t$,
\begin{equation}
\E[z_{t+1}]=A\,\E[z_t],
\qquad
\Sigma_{t+1}=\Lop_K(\Sigma_t),
\label{eq:moment-recurrences}
\end{equation}
hence $\Sigma_t=\Lop_K^t(\Sigma_0)$ (proof in Appendix~\ref{app:proof-full}). Crucially, the second-moment map splits as
\begin{equation}
\begin{gathered} %
\Lop_K(X)=AXA^\top+\Bm\,\Vop_K(\Em X\Em^\top)\,\Bm^\top,
\\
\Vop_K(Y)=\E\bigl[(D_K-\bar D_K)\,Y\,(D_K-\bar D_K)^\top\bigr],
\\
\Bm=\begin{bmatrix}-\alpha(1+\beta)I\\ I\end{bmatrix},
\qquad
\Em=\begin{bmatrix}I&0\end{bmatrix},
\end{gathered}
\label{eq:decomposition}
\end{equation}
with $\Vop_K$ the \textit{block noise}, the covariance map of the block operator.
A mean-stability analysis replaces $D_K$ by its mean $\bar D_K$, so it keeps only the first term of \eqref{eq:decomposition}, which propagates the second moment through the mean transition $A$. The second term adds, at every outer step, the block noise applied to the current error. Because this term is a positive map, it can only raise the \textit{spectral radius} $\rho$, so $\rho(\Lop_K)\ge\rho(A)^2$, with equality when there is no block noise. A mean-stability analysis can therefore be wrong only by predicting stability for an unstable setting (example in Appendix~\ref{app:affine-example}).

\subsection{The mean condition and the noise budget}
\label{sec:stability}

By \eqref{eq:decomposition}, $\rho(\Lop_K)$ mixes the mean dynamics and the block noise in a single number. In Theorem~\ref{thm:full} we separate them into a condition on the mean transition and a budget for the noise. When $\rho(A)<1$, let
\begin{equation}
\Gop_{\alpha,\beta}(Y)=\sum_{t\ge0}\Em A^t\Bm\;Y\;\Bm^\top(A^\top)^t\Em^\top
\label{eq:noise-gain}
\end{equation}
be the \textit{noise gain}. If the block noise injects a covariance $Y$ at one outer step, $\Gop_{\alpha,\beta}(Y)$ is the second moment of the error that the mean dynamics build from it over all later steps. It depends on the distribution of the block operator only through its mean $\bar D_K$. The error in turn feeds the noise of the next block through $\Vop_K$, so we call $\rho(\Vop_K\circ\Gop_{\alpha,\beta})$ the \textit{loop gain}. For minibatch SGD, \citet{agarwala2024conservative} split stability in the same way, into a condition on the deterministic dynamics and a noise kernel norm below one, under a high-dimensional approximation of the dynamics. Theorem~\ref{thm:full} makes this split exact for $K$-step blocks with outer momentum. We call the reciprocal of the loop gain, $1/\rho(\Vop_K\circ\Gop_{\alpha,\beta})$, the \textit{noise budget} of $(\alpha,\beta)$. It is the factor by which the block noise could grow before the criterion fails, so the setting is stable when the budget exceeds one, and a larger noise gain means a smaller budget.

\begin{keyresult}
\begin{theorem}[Exact stability criterion and noise budget]
\label{thm:full}
Under Assumption~\ref{ass:block}:
\begin{enumerate}\setlength{\itemsep}{1pt}\setlength{\parskip}{0pt}
\item[(i)] the recurrence \eqref{eq:transition} is mean-square stable if and only if $\rho(\Lop_K)<1$, \newline equivalently \mbox{$\rho(\E[\T(D_K)\kron\T(D_K)])<1$};
\item[(ii)] $\rho(\Lop_K)<1$ if and only if $\rho(A)<1$ and
\begin{equation}
\rho\bigl(\Vop_K\circ\Gop_{\alpha,\beta}\bigr)<1 ;
\label{eq:budget-criterion}
\end{equation}
\item[(iii)] $\rho(A)<1$ if and only if, for every eigenvalue $\lambda$ of $\bar D_K$, the $2\times2$ matrix $\T(\lambda)$ of \eqref{eq:transition} has spectral radius below one; when the eigenvalues of $\bar D_K$ are real it reads
\begin{equation}
0<\alpha\lambda<\frac{2(1+\beta)}{1+2\beta}.
\label{eq:mean-condition}
\end{equation}
\end{enumerate}
\end{theorem}
\end{keyresult}

Mean stability is therefore necessary but not sufficient, because the block noise must also fit within the noise budget that the mean dynamics set. Note that the moment recurrences \eqref{eq:moment-recurrences}, the split \eqref{eq:decomposition} and Theorem~\ref{thm:full} use nothing about $D_K$ beyond Assumption~\ref{ass:block}, so they hold for any i.i.d.\ block operators, of which the quadratic model is one case. Part (i) still holds if we consider only starts with an empty buffer, $u_{-1}=0$, as in the algorithm (proof in Appendix~\ref{app:proof-full}). In one dimension the criterion has a closed form (Appendix~\ref{app:proof-scalar}).

\paragraph{Small outer steps.} Under the mean condition and for a symmetric $\bar D_K$, the criterion with $W$ workers reads, to leading order in $\alpha$,
\begin{equation}
\begin{gathered}
\frac{\alpha\,\kappa_K}{(1-\beta)\,W}<1,
\\
\kappa_K=\rho\bigl(\Vop_K\circ\Sop^{-1}\bigr),\qquad \Sop(X)=\bar D_KX+X\bar D_K,
\end{gathered}
\label{eq:noise-scale}
\end{equation}
which has the form of the noise scale of heavy-ball SGD \citep{smith2018bayesian}, with $\alpha$ as the learning rate, $W$ as the batch size and $\kappa_K$ as the training-set size (Appendix~\ref{app:smallstep}).

At $K=1$ and $(\alpha,\beta)=(1,0)$ the criterion reduces to that of \citet{mulayoff2024exact}. At $(\alpha,\beta)=(1,0)$ and any $K$ the outer step continues the inner trajectory, and the criterion becomes the mean-square stability of plain SGD observed every $K$ steps.
\begin{figure}[!ht]
\centering
\includegraphics[width=\linewidth]{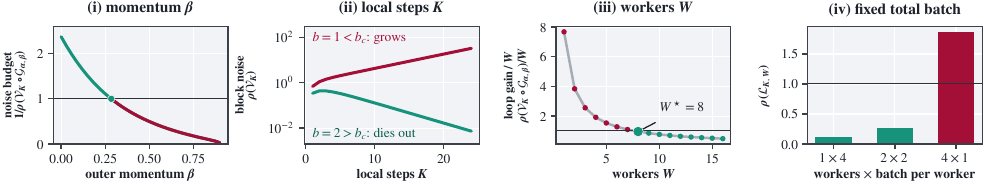}
\caption{\textbf{Mechanism behind each implication, on the model of Figure~\ref{fig:phase}.} (\uline{i}) The noise budget shrinks as $\beta$ grows (Proposition~\ref{prop:momentum}). (\uline{ii}) The block noise grows with $K$ below $b_c$ and dies out above it (Proposition~\ref{thm:blocklength}). (\uline{iii}) Averaging divides the loop gain by $W$ (Corollary~\ref{cor:averaging}). (\uline{iv}) Splitting a fixed total batch raises $\rho(\Lop_{K,W})$ (Proposition~\ref{prop:totalbatch}). \textcolor{drawteal}{Green} marks stable settings and \textcolor{unstablered}{red} marks unstable ones.}\label{fig:implications}
\end{figure}

\section{Implications for $\beta$, $K$ \&\ $W$}\label{sec:implications}

\subsection{Outer momentum shrinks the noise budget}
\label{sec:cons-beta}

While the momentum buffer can accelerate training by summing past pseudo-gradients, it also accumulates their block noise. Proposition~\ref{prop:momentum} shows how stability depends on $\beta$ at fixed $\alpha$, for any distribution of the block operator and in any dimension (Figure~\ref{fig:implications}(i)).

\begin{keyresult}
\begin{proposition}[Momentum monotonicity]
\label{prop:momentum}
Fix $\alpha$ and the distribution of the block operator, and let $\beta\le\beta'$ with $(\alpha,\beta')$ mean stable. Then $(\alpha,\beta)$ is mean stable, $\Gop_{\alpha,\beta}\preceq\Gop_{\alpha,\beta'}$, and $\rho(\Vop_K\circ\Gop_{\alpha,\beta})\le\rho(\Vop_K\circ\Gop_{\alpha,\beta'})$ for every block noise $\Vop_K$. In particular, mean-square stability at $(\alpha,\beta')$ implies mean-square stability at~$(\alpha,\beta)$.
\end{proposition}
\end{keyresult}

Consequently, if a setting is unstable because of its block noise, lowering $\beta$ can only help, and so can adding workers (Section~\ref{sec:cons-W}). Note that the noise budget is monotone in $\beta$, whereas $\rho(\Lop_K)$ itself need not be (Figure~\ref{fig:exact-beta} in Appendix~\ref{app:toy}).

\subsection{Local steps and the inner batch size}
\label{sec:cons-K}

Practitioners prefer the number of local steps to be as large as stability allows, since longer blocks reduce the frequency of communication \citep{stich2019local,wang2019adacomm,douillard2023diloco,topollai2026restart}. By \eqref{eq:block-moments} the block noise is $\Vop_K=\Qop^K-\Qop_\infty^K$, with $\Qop=\Qop_b$ and $\Qop_\infty$ as in \eqref{eq:batch-split}, while the mean block operator saturates, $\bar D_K\to I$ when $\rho(\bar Q)<1$. Whether the block noise dies out or grows with $K$ is therefore decided by $\rho(\Qop)$ alone. Since $\Qop^k$ propagates the second moment of a displacement through $k$ SGD steps, $\rho(\Qop)<1$ is mean-square stability of the inner steps themselves \citep{mulayoff2024exact}.

\begin{keyresult}
\begin{proposition}[Long blocks and the inner-batch threshold]
\label{thm:blocklength}
Let the inner steps of a block be i.i.d., let $\rho(\bar Q)<1$, and fix $(\alpha,\beta)$.
\begin{enumerate}\setlength{\itemsep}{1pt}\setlength{\parskip}{0pt}
\item[(i)] If $\rho(\Qop)<1$ and $(\alpha,\beta)$ satisfies the mean condition \eqref{eq:mean-condition} at $\lambda=1$, then $\rho(\Lop_K)\to\rho(\T(I))^2<1$, and the recurrence is mean-square stable for all large $K$.
\item[(ii)] If $\rho(\Qop)>1$, then $\rho(\Lop_K)\ge\gamma^2\bigl(\rho(\Qop)^K-\rho(\bar Q)^{2K}\bigr)\to\infty$ with $\gamma=\alpha(1+\beta)$, and the recurrence is mean-square unstable for all large $K$ at every outer rate and momentum.
\item[(iii)] With minibatches of size $b$ drawn with replacement as in \eqref{eq:batch-split},
\begin{equation}
\begin{split} %
&\rho(\Qop_b)<1
\\
&\quad\Longleftrightarrow\quad
b>b_c:=\rho\bigl(\Cop\circ(\mathcal I-\Qop_\infty)^{-1}\bigr),
\end{split}
\label{eq:critical-batch}
\end{equation}
so (i) applies when $b>b_c$ and (ii) when $b<b_c$.
\end{enumerate}
\end{proposition}
\end{keyresult}

That is, in case (i) the blocks behave in the limit like one that lands exactly on the minimizer, $D_K=I$. Figure~\ref{fig:implications}(ii) shows the block noise in the two branches, and the left and middle panels of Figure~\ref{fig:phase} show their stable regions. \citet{mulayoff2024exact} showed that larger inner batches can only help the inner steps, and we show that the threshold $b_c$ grows linearly with $\eta$ at small inner rates (proof in Appendix~\ref{app:proof-blocklength}). Given a mean-stable outer step, whether long enough blocks are stable is thus decided by the inner optimizer alone, so to communicate less often one should raise the per-worker batch size above the threshold $b_c$ or lower the inner rate.

\subsection{Averaging over workers}
\label{sec:cons-W}

Besides the inner batch size, averaging over workers also reduces the block noise, because averaging independent pseudo-gradients keeps their mean and shrinks only their spread around it (Figure~\ref{fig:implications}(iii)).

\begin{keyresult}
\begin{corollary}[Averaging over workers]
\label{cor:averaging}
Let the pseudo-gradient of each outer step be the average over $W$ workers whose inner trajectories start from the same point, are independent given that point, and have block operators as in Assumption~\ref{ass:block}.
\begin{enumerate}\setlength{\itemsep}{1pt}\setlength{\parskip}{0pt}
\item[(a)] The averaged block operator has mean $\bar D_K$ and block noise $\Vop_K/W$, and $\rho(\Lop_{K,W})$ is non-increasing in $W$. The averaged recurrence is mean-square stable if and only if the mean condition holds and $W>\rho(\Vop_K\circ\Gop_{\alpha,\beta})$, so $W^\star=\lfloor\rho(\Vop_K\circ\Gop_{\alpha,\beta})\rfloor+1$ workers are needed and enough. If the mean condition fails, no number of workers suffices.
\item[(b)] Under the hypotheses of Proposition~\ref{thm:blocklength}(ii), mean-square stability with $W$ workers requires\newline \mbox{$W>\gamma^2\bigl(\rho(\Qop)^K-\rho(\bar Q)^{2K}\bigr)$} with $\gamma=\alpha(1+\beta)$, a bound that grows exponentially in $K$.
\end{enumerate}
\end{corollary}
\end{keyresult}

By Proposition~\ref{prop:momentum} the count $W^\star$ never falls as $\beta$ rises, so the largest admissible momentum at a fixed rate is non-decreasing in $W$. \citet{agarwala2026twophase} analyzes averaging over workers in the zero-momentum case, and Table~\ref{tab:toy-w} in Appendix~\ref{app:toy} checks the count $W^\star$ against the second-moment map.

\subsection{Fixed total batch}
\label{sec:cons-total}

More workers at a fixed per-worker batch size cost more gradient computation, so for a fairer comparison we hold the total batch $W\times b$ fixed and ask whether splitting it across workers helps (Figure~\ref{fig:implications}(iv)).

\begin{keyresult}
\begin{proposition}[Fixed total batch]
\label{prop:totalbatch}
Let the inner steps of a block be i.i.d.\ with minibatches drawn with replacement as in \eqref{eq:batch-split}, and let $W$ divide $b$. At fixed $(\alpha,\beta,K)$, $W$ workers with batch $b/W$ each have the same mean block operator as one worker with batch $b$ and at least as much block noise, $\tfrac1W\Vop_K(b/W)\succeq\Vop_K(b)$, with equality at $K=1$. In particular, $\rho\bigl(\Lop_{K,W}(b/W)\bigr)\ge\rho\bigl(\Lop_K(b)\bigr)$.
\end{proposition}
\end{keyresult}

Expanding $\Qop_b^K$ by \eqref{eq:batch-split} writes the block noise as a sum over powers of $1/b$ with positive coefficients, so splitting is nearly free while per-worker batches stay large (proof in Appendix~\ref{app:proof-averaging}), and since workers run in parallel, splitting the batch makes each outer step faster. \citet{agarwala2026twophase} makes the same comparison in the isotropic case without momentum.

\section{Testing the Implications on Neural Networks}
\label{sec:checkpoint}

We now test the implications of Section~\ref{sec:implications} on neural networks. A network's minibatch losses have different minimizers, and its training rarely runs to convergence. We therefore linearize the pseudo-gradient of each block at a checkpoint and measure instability with a quantity that makes no reference to a minimizer. 

\subsection{The block Jacobian}
\label{sec:checkpoint-jacobian}

To connect our model to a network, we begin by asking how the pseudo-gradient of a block changes in response to a parameter perturbation $\delta$. The first-order expansion is
\begin{equation}
P_{K,\omega}(\theta+\delta)=P_{K,\omega}(\theta)+J_\omega(\theta)\,\delta+r_\omega(\delta),
\label{eq:taylor}
\end{equation}
where the remainder $r_\omega(\delta)$ is of order $\opnorm{\delta}^2$ wherever the pseudo-gradient is smooth in $\theta$. We call $J_\omega(\theta)$ the \textit{block Jacobian}. The chain rule gives
\begin{equation}
J_\omega(\theta)=I-\prod_{k=K-1}^{0}\bigl[I-\eta\,\nabla^2\ell_{\xi_k}(\theta_{k,\omega})\bigr].
\label{eq:chain-rule}
\end{equation}

Each Hessian is taken at the point the inner iterate visits, so the block Jacobian depends on the minibatches $\omega$ and, in general, on the starting point $\theta$.

In the model of Section~\ref{sec:setup-model} the block Jacobian $J_\omega(\theta)$ becomes the block operator $D_{K,\omega}$ at every $\theta$. Theorem~\ref{thm:full}, Proposition~\ref{prop:momentum} and Corollary~\ref{cor:averaging}(a) require only that the block Jacobian does not depend on $\theta$ and that blocks are i.i.d.\ across outer steps. We refer to this assumption as the \textit{frozen linearization}. In practice, we obtain this property by taking the block Jacobian of every outer step at a checkpoint $\theta_0$, with the minibatches drawn afresh for each worker at every outer step (as outlined in Algorithm~\ref{alg:frozen}). Note that each inner trajectory still runs on the actual loss from $\theta_0$, and each Hessian in \eqref{eq:chain-rule} is taken at the point it visits, so the inner steps are not replaced by a quadratic model of the loss at $\theta_0$.

\begin{figure}[t]   %
\begin{minipage}[t]{0.485\textwidth}\vspace{0pt}
\centering
\includegraphics[width=\linewidth]{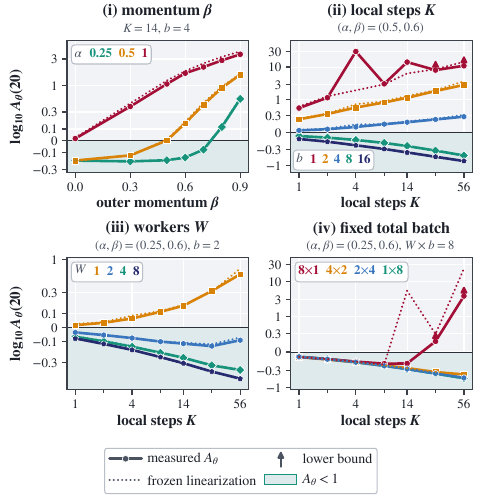}   %
\caption{\textbf{Measured amplification on the CIFAR MLP at step $120{,}000$.} Every panel shows $\log_{10}A_\theta(20)$, the amplification after $20$ outer steps, against the outer momentum in (\uline{i}) and against the number of local steps in (\uline{ii}) to (\uline{iv}). Solid lines are $A_\theta$ measured along actual runs, and dotted lines its prediction under the frozen linearization. Arrows ($\uparrow$) mark points where some runs diverged.}
\label{fig:ckpt-amp}
\end{minipage}\hfill
\begin{minipage}[t]{0.485\textwidth}\vspace{0pt}
\centering
\includegraphics[width=\linewidth]{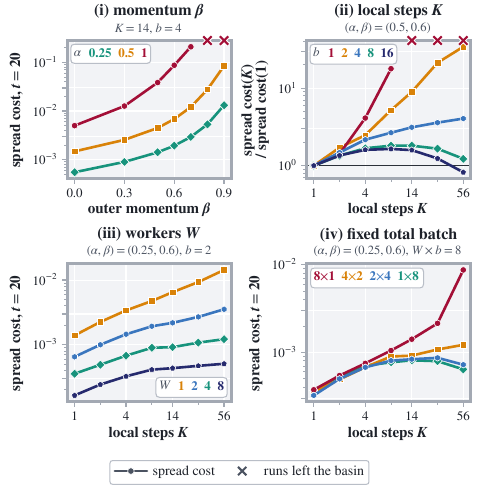}   %
\caption{\textbf{Spread cost on the CIFAR MLP at step $120{,}000$.} The spread cost \eqref{eq:loss-split} of $32$ independent runs (from different seeds) after $20$ outer steps, at the settings of Figure~\ref{fig:ckpt-amp}. A cross ($\times$) on the top edge marks a setting whose runs' average ends with a loss more than $0.1$ above the checkpoint's. Panel (\uline{ii}) shows the spread cost of each inner batch relative to $K=1$.}
\label{fig:ckpt-spread-impl}
\end{minipage}
\end{figure}

Under the frozen linearization, we can interpret the block Jacobian at the checkpoint as the block operator of Section~\ref{sec:setup-model}, and so we write $D_\omega=J_\omega(\theta_0)$. Since the minibatches are drawn afresh at every outer step, the $D_t=J_{\omega_t}(\theta_0)$ are i.i.d., so Assumption~\ref{ass:block} holds whenever they have finite second moments. The pseudo-gradient then becomes affine in $\theta$, $P_{K,\omega_t}(\theta)=p_t+D_t(\theta-\theta_0)$, as in the regular-minimum case of Section~\ref{sec:setup}, with an \textit{offset} $p_t=P_{K,\omega_t}(\theta_0)$ that need not have mean zero because $\theta_0$ is not a minimizer\label{sec:regular}. Importantly, the displacement of a run from $\theta_0$ obeys the forced recurrence (Appendix~\ref{app:proof-forced})
\begin{equation}
z_{t+1}=\T(D_t)\,z_t+\Bm\,p_t
\label{eq:forced-recurrence}
\end{equation}
exactly. The offset leaves $\Lop_K$ unchanged, so by Proposition~\ref{prop:bridge} the covariance of $\theta_t$ across independent streams of minibatches, which we call the \textit{spread} of a run, stays bounded at a noise floor when $\rho(\Lop_K)<1$ and generically grows without bound otherwise.

\paragraph{Shared-minibatch perturbations.} Consider two runs of \eqref{eq:forced-recurrence} that see the same minibatches but start from different points. Because their minibatches are the same, both runs receive the same offset $p_t$, and subtracting one recurrence from the other removes it. Their difference $z_t$ therefore obeys $z_{t+1}=\T(D_t)\,z_t$, which is \eqref{eq:transition}. Unlike the spread, this difference decays exactly when $\rho(\Lop_K)<1$. We call the difference $e_t$ in their parameters a \textit{shared-minibatch perturbation} and measure stability from it (Section~\ref{sec:checkpoint-measures}). At an actual network, the difference follows \eqref{eq:transition} only to first order (Appendix~\ref{sec:frozen}).

\subsection{Amplification and spread cost}
\label{sec:checkpoint-measures}

We read the stability at a checkpoint from the \textit{amplification} of shared-minibatch perturbations, and then measure its consequences on training loss through the \textit{spread cost}.

\paragraph{Amplification.} We define the \textit{amplification} after $T$ outer steps as
\begin{equation}\label{eq:amplification}
A_\theta(T)=\frac{\mathbb E\,\|e_T\|^2}{\|e_0\|^2},
\end{equation}
where $e_t$ is the shared-minibatch perturbation and the expectation is over streams of minibatches, and we say that a setting amplifies when $A_\theta(T)>1$. Following how a small perturbation grows is a common way to study the stability of training viewed as a dynamical system \citep{wu2018sgdselection,dieuleveut2020bridging,chemnitz2025characterizing,li2025coupling}. Notably, \citet{altintas2025butterfly} measure how far two such runs diverge in parameter distance and in functional similarity. In contrast, under the frozen linearization our amplification decays exactly when $\rho(\Lop_K)<1$, and we test whether it changes with the outer momentum, the number of local steps, the inner batch size and the number of workers in the way the four implications predict. As in Section~\ref{sec:setup-stability}, we restrict to the complement of the common flat directions, on which every minibatch Hessian vanishes. We therefore take $e_0=J_{\omega'}(\theta_0)\,g$ for a random vector $g$ and a separate block $\omega'$, which lies in that complement, with an empty buffer (Appendix~\ref{app:measures}). Without ever forming $\Lop_K$ (a map on $2p\times 2p$ matrices for a network with $p$ parameters), we estimate the amplification by propagating $e_t$ with Hessian-vector products (HVPs), along actual runs for the measured values and by Algorithm~\ref{alg:frozen} for the predictions under the frozen linearization. Each Hessian-vector product costs about as much as a gradient evaluation \citep{pearlmutter1994hessian}, so tracking $e_t$ adds about one gradient's worth of computation per inner step. Since $e_t$ evolves linearly, we rescale it to unit length after every outer step and accumulate the logarithm of the scale, as for Lyapunov exponents.
\paragraph{Mean or mean-square instability.} Varying the number of workers lets us tell whether an amplifying setting is unstable in its mean dynamics or only in mean square. This works because averaging over workers divides the block noise by $W$ but leaves the mean transition unchanged (Corollary~\ref{cor:averaging}). Moreover, under the frozen linearization the amplification is non-increasing in $W$ for every $T$ (Appendix~\ref{app:proof-averaging}). Hence, if a setting amplifies with one worker but contracts after many outer steps once enough workers are averaged, this is consistent with stable mean dynamics and mean-square instability with one worker. On the other hand, a setting that still amplifies at the largest $W$ we try is either mean unstable or needs more workers.

\paragraph{Spread cost.} The change in the mean training loss of $S$ independent runs \citep{fort2020deep,frankle2020linear} splits exactly into two parts. The first, the averaged-model change, is driven by the offset even in a stable setting, so a mean-square instability can raise the loss only through the second, the spread cost. With $\hat\mu_t$ the average of the runs, $\bar L_t$ their mean loss and $L$ the full-data loss,
\begin{equation}
\bar L_t-L(\theta_0)=
\underbrace{L(\hat\mu_t)-L(\theta_0)}_{\text{averaged-model change}}
+\underbrace{\bar L_t-L(\hat\mu_t)}_{\text{spread cost}} .
\label{eq:loss-split}
\end{equation}
The spread cost is the loss due to the disagreement between the runs, and replacing the runs by their average would recover exactly this amount. To second order it is a curvature-weighted trace of their covariance, which Proposition~\ref{prop:bridge} propagates (Appendix~\ref{app:measures}). We say that the runs left the checkpoint's basin when their average ends with a loss more than $0.1$ above the checkpoint's.

\subsection{Empirical results}
\label{sec:checkpoint-measure}

We conduct experiments at checkpoints of an MLP and a ResNet-20, both trained by plain SGD on CIFAR-10 \citep{krizhevsky2009cifar}, and of Pythia-70m (Appendix~\ref{app:net-setup}), and at one checkpoint of Pythia-160m, Pythia-410m and Pythia-1b (Appendix~\ref{app:all-larger}). Figures~\ref{fig:ckpt-amp} and~\ref{fig:ckpt-spread-impl} show a single checkpoint, the CIFAR MLP at step $120{,}000$, and experiments with different checkpoints and other models are in Appendices~\ref{app:implications-all} to~\ref{app:continuations}. The 95\% bootstrap confidence intervals over the 48 minibatch streams (from different seeds) have a median width of 0.027 in \(\log_{10} A_\theta(20)\). We take the checkpoints after the early phase of training, as it has been shown that linear approximations about an early checkpoint do not track the actual training well \citep{fort2020deep,meterez2026quadratic}.

Figure~\ref{fig:ckpt-amp} is consistent with each of the four implications. As $\beta$ grows at a fixed outer rate, a setting turns from contracting to amplifying at most once (i). As $K$ grows, a setting contracts more at inner batches $8$ and $16$ and amplifies more at batches $1$ and $2$ (ii), matching the two behaviors of Proposition~\ref{thm:blocklength}. More workers never raise the amplification (iii), and splitting a fixed total batch did not lower it in any setting of Figure~\ref{fig:ckpt-amp} (iv). The setting of panel (iii) amplifies with one worker and contracts with two or more, so this instability appears to be mean-square (Section~\ref{sec:checkpoint-measures}).

As with the amplification, the spread cost rises with $\beta$, falls with $W$ and is not lowered by splitting a fixed total batch (Figure~\ref{fig:ckpt-spread-impl}). As $K$ grows, it first rises at every inner batch, because the noise that a block injects accumulates over its $K$ steps (Proposition~\ref{prop:bridge}). The batches then separate as in Proposition~\ref{thm:blocklength}. The spread cost keeps rising at $b\le4$ and turns down at $b=8$ and $16$. On ResNet-20, on the Pythia models and at other checkpoints, the amplification is also largely consistent with the implications (Appendix~\ref{app:implications-all}).

\FloatBarrier
\section{Discussion}
\label{sec:limitations}

Our results show that the inner trajectory acts on the outer iteration as a random operator, the block operator, and that once the mean condition holds, stability is decided by whether the noise of this operator fits within the noise budget of the outer step. The number of local steps, the inner batch size and the number of workers change the block operator, while the outer learning rate and momentum shape the noise budget, so the effect of any one hyperparameter depends on the others. Stability can also be read at a checkpoint of a network from short continuations of training that share their minibatches. In distributed training these hyperparameters are often constrained, since limited bandwidth sets a minimum number of local steps, the memory of each worker caps the inner batch size, and hardware failures can reduce the number of workers during training. Under such constraints, the checkpoint measurement suggests a way to check whether a planned change keeps training stable before committing to it. However, the spread cost need not follow the trends of the amplification at every checkpoint (Appendix~\ref{app:implications-all}), as it also depends on the noise that each block injects and on the curvature of the loss (Proposition~\ref{prop:bridge}; Appendix~\ref{app:measures}).

\paragraph{Extensions.} Our criterion also restricts the minima that two-phase optimizers can settle into, since the outer iteration cannot settle in mean square at a minimum where the criterion fails, and certain desirable properties of such minima might be valuable to the community \citep{wu2023implicit}. In this direction, \citet{mulayoff2024exact} study the dynamical stability of optimization methods in the vicinity of the minima of the loss, and there is also a rich literature on which minima are dynamically accessible to an optimization algorithm \citep{wu2018sgdselection}. Although it has yet to be shown, a two-phase system with outer momentum may also display catapults \citep{lewkowycz2020catapult} or train at the edge of stability, as gradient descent \citep{cohen2021eos}, SGD \citep{andreyev2024edge}, SGD with momentum \citep{andreyev2026momentum} and federated averaging \citep{khandelwal2026scaffold} do.

\paragraph{Limitations.} The mean-square analysis focuses on stability rather than on the rate of convergence, and it does not capture the fact that training can, in the longer term, recover from a local instability \citep{lewkowycz2020catapult,damian2023selfstabilization}. Our scope is limited to SGD for the inner steps, whereas modern two-phase methods can use AdamW or Muon inside the block. Such settings are especially difficult to study when an adaptive inner optimizer carries state across blocks and preconditions the response, so its block Jacobian is not the product \eqref{eq:chain-rule} and Assumption~\ref{ass:block} then has to be re-examined.

\section*{Acknowledgments}
This research was enabled in part by compute resources provided by Mila (mila.quebec), and by support provided by Calcul Qu\'ebec (calculquebec.ca), Compute Ontario (computeontario.ca), the BC DRI Group and the Digital Research Alliance of Canada (alliancecan.ca). We also thank the Masason Foundation for providing compute resources. Sarath Chandar is supported by the Canada CIFAR AI Chairs program, the Canada Research Chair in Lifelong Machine Learning, and an NSERC Discovery Grant.

\clearpage
\appendix

\raggedbottom   %
\makeatletter
\setlength{\@fptop}{0pt}
\setlength{\@fpsep}{14pt}
\setlength{\@fpbot}{0pt plus 1fil}
\makeatother

\renewcommand{\contentsname}{Contents of the Appendix}
\addtocontents{toc}{\protect\setcounter{tocdepth}{2}}
\tableofcontents

\clearpage   %
\section{The Two-Phase Optimizer with Nesterov Momentum}
\label{app:algorithm}
Algorithm~\ref{alg:outer} gives one run of the two-phase optimizer \eqref{eq:snoo} with $W$ workers. Most of Sections~\ref{sec:setup} and~\ref{sec:theory} treat the case $W=1$, and Section~\ref{sec:cons-W} the general case. At each outer step the buffer is updated first, and the outer step then uses the updated buffer. This makes the outer step Nesterov momentum, since heavy-ball momentum would instead step along the buffer alone, $\theta_{t+1}=\theta_t-\alpha u_t$.

Algorithm~\ref{alg:frozen} gives the frozen linearization of Section~\ref{sec:checkpoint-jacobian}, which produces the dotted lines in Figure~\ref{fig:ckpt-amp} and in the appendix figures. It is Algorithm~\ref{alg:outer} applied to a perturbation $e_t$, with every inner trajectory started at the checkpoint $\theta_0$ instead of the current iterate.

\begin{algorithm}[H]   %
\caption{The frozen linearization at a checkpoint $\theta_0$}
\label{alg:frozen}
\begin{minipage}{0.83\linewidth}
\begin{algorithmic}[1]
\Require checkpoint $\theta_0$, perturbation $e_0$, inner rate $\eta$, inner batch size $b$, outer rate $\alpha$, outer momentum $\beta$, number of local steps $K$, number of workers $W$, number of outer steps $T$
\State $u_{-1}\gets0$
\For{$t=0,\ldots,T-1$}
  \For{$w=1,\ldots,W$} \Comment{per worker}\tikzmark{algFrozenInA}
    \State \parbox[t]{0.66\textwidth}{run $K$ inner SGD steps at rate $\eta$ with batch size $b$ from $\theta_0$ on fresh minibatches $\omega_t^{(w)}$}
    \State $d_t^{(w)}\gets J_{\omega_t^{(w)}}(\theta_0)\,e_t$ \Comment{\eqref{eq:chain-rule}, $K$ HVPs}
  \EndFor\hfill\tikzmark{algFrozenInB}
  \State $d_t\gets\frac1W\sum_{w=1}^W d_t^{(w)}$\hfill\tikzmark{algFrozenOutA}
  \State $u_t\gets\beta u_{t-1}+d_t$
  \State $e_{t+1}\gets e_t-\alpha\bigl(\beta u_t+d_t\bigr)$\hfill\tikzmark{algFrozenOutB}
\EndFor
\end{algorithmic}
\end{minipage}
\begin{tikzpicture}[remember picture,overlay]
\phasebracket{algFrozenInA}{algFrozenInB}{1}{Linearized\\inner phase}
\phasebracket{algFrozenOutA}{algFrozenOutB}{2}{Outer\\step}
\end{tikzpicture}
\end{algorithm}

\section{Mean and Mean-Square Stability}
\label{app:mean-vs-ms}

Here the error $e_t=\theta_t-\theta^*$ is the displacement from the minimizer $\theta^*$ of the random quadratic model of Section~\ref{sec:setup-model}, and $z_t=(e_t,u_{t-1})$ stacks it with the momentum buffer. By Definition~\ref{def:mss}, the recurrence is mean stable if $\E[z_t]\to0$ and mean-square stable if $\E\|z_t\|^2\to0$. Since $\E\|z_t\|^2=\|\E[z_t]\|^2+\operatorname{tr}\Cov(z_t)$, mean-square stability also requires the fluctuations of $z_t$ around its mean to vanish.

\begin{figure}[H]   %
\centering
\includegraphics[width=0.42\linewidth]{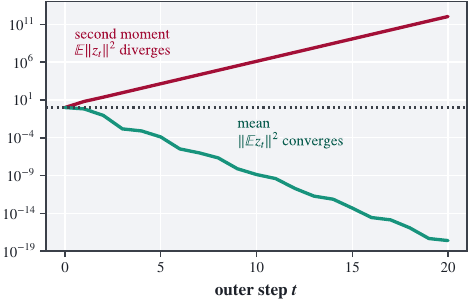}
\caption{\textbf{Mean stable but mean-square unstable.} The mean $\|\E[z_t]\|^2$ and the second moment $\E\|z_t\|^2$ for the linear-regression model above ($\eta=2.5$, $b=1$, $K=4$, $\alpha=1$, $\beta=0.5$), starting from a unit error with an empty buffer. Both curves are computed exactly from the moment recurrences \eqref{eq:moment-recurrences}. The dotted line marks the initial value.}
\label{fig:mean-vs-ms}
\end{figure}

Figure~\ref{fig:mean-vs-ms} shows a setting in which the mean of $z_t$ converges while its second moment diverges. The model used is linear regression in eight dimensions with unit-norm features, so each sample loss has curvature one along its own feature, while the average curvature is about $1/8$ in every direction (Appendix~\ref{app:toy}). With the inner rate $\eta=2.5$, a step on a single sample therefore overshoots along that sample's feature, while the average step shrinks the error by a factor of about $0.69$. The pseudo-gradient thus points towards the minimizer on average, but its fluctuations are proportional to the current error; so they compound over the outer steps while the mean vanishes. The divergence of the second moment does not come from rare runs, since every one of $4{,}000$ simulated runs diverges. In the terms of Section~\ref{sec:theory}, we see that the mean condition holds ($\rho(A)\approx0.37$) while the block noise does not fit within the noise budget ($\rho(\Lop_K)\approx3.9$, Theorem~\ref{thm:full}(ii)). An analysis of the mean dynamics alone would therefore call this setting stable.

\section{Notation and Conventions}
\label{app:notation}

The table below lists the notation of the paper and Appendix~\ref{app:glossary} lists the terms the paper defines. Symbols used inside a single proof are defined there.

\medskip
{\small
\setlength{\LTpre}{0pt}\setlength{\LTpost}{0pt}
\renewcommand{\theHtable}{notation.\arabic{table}}%
\renewcommand{\arraystretch}{1.35}
\settowidth{\tabtermwidth}{shared-minibatch perturbation}\addtolength{\tabtermwidth}{1pt}
\begin{xltabular}{\linewidth}{@{}p{\tabtermwidth}>{\raggedright\arraybackslash}X@{}}
\toprule
\textsc{Notation} & \textsc{Description} \\
\midrule
\endfirsthead
\multicolumn{2}{@{}l}{\small\itshape Notation, continued.}\\[2pt]
\toprule
\textsc{Notation} & \textsc{Description} \\
\midrule
\endhead
\midrule
\multicolumn{2}{r@{}}{\small\itshape continued on the next page}\\
\endfoot
\bottomrule
\endlastfoot
\rowcolor{gray!15}\multicolumn{2}{@{}l}{\textsc{Two-phase optimizer}} \\*\addlinespace[2pt]
$\theta$, $\theta^*$ & parameters; common minimizer of the sample losses \\
$\theta^{(w)}_{t,k}$ & parameters of worker $w$ after $k$ inner steps of outer step $t$ \\
$K$, $W$ & number of local steps; number of workers \\
$t$, $k$, $T$ & outer step; inner step; number of outer steps \\
$\alpha$, $\beta$ & outer learning rate; outer momentum \\
$\eta$, $b$ & inner learning rate; inner batch size \\
$\omega$, $\xi_k$ & minibatches of one block; minibatch of inner step $k$ \\
$P_{K,\omega}$, $P_t$ & pseudo-gradient of one block; its average over the workers \\
$u_t$ & momentum buffer \\
$e_t$ & error; at a checkpoint, the perturbation \\
$z_t$ & state $(e_t,u_{t-1})$ \\
\midrule
\rowcolor{gray!15}\multicolumn{2}{@{}l}{\textsc{Model}} \\*\addlinespace[2pt]
$\ell_\xi$, $L$ & minibatch loss; full loss \\
$H_\xi$, $h_i$, $\bar H$ & Hessian of $\ell_\xi$; of sample $i$; mean Hessian \\
$Q_\xi$, $\bar Q$ & one inner step $I-\eta H_\xi$; its mean \\
$\Qop_b$, $\Qop_\infty$, $\Cop$ & second moment of one inner step at batch $b$; its large-batch limit; the batch-noise part \\
$M_{K,\omega}$ & product of the $K$ inner steps \\
$D_{K,\omega}$ & block operator $I-M_{K,\omega}$ \\
$\bar D_K$, $\Sm_K$ & mean and second moment of the block operator \\
$b_c$ & inner-batch threshold \\
$\kron$, $\vecop$ & Kronecker product; vectorization \\
\midrule
\rowcolor{gray!15}\multicolumn{2}{@{}l}{\textsc{Stability}} \\*\addlinespace[2pt]
$\T(D)$ & one outer step on the state \\
$\Sigma_t$ & second moment of the state \\
$\Lop_K$, $A$ & second-moment map; mean transition \\
$\Vop_K$, $\Gop_{\alpha,\beta}$ & block noise; noise gain \\
$\Bm$, $\Em$ & input and output matrices \\
$\gamma$ & $\alpha(1+\beta)$ \\
$\Sop$, $\kappa_K$ & the map $X\mapsto\bar D_KX+X\bar D_K$; block noise relative to the mean block operator \\
$\Lop_{K,W}$, $W^\star$ & second-moment map with $W$ workers; workers needed \\
$\rho(\cdot)$, $\preceq$, $\mathcal I$ & spectral radius; order of positive maps and matrices; identity map \\
$\rho_K$ & growth per inner gradient \\
$d_1$, $d_2$, $\Nk$, $x$ & one-dimensional quantities of Appendix~\ref{app:proof-scalar} \\
\midrule
\rowcolor{gray!15}\multicolumn{2}{@{}l}{\textsc{Network checkpoint}} \\*\addlinespace[2pt]
$\theta_0$, $J_\omega$, $r_\omega$ & checkpoint; block Jacobian; remainder of the expansion \\
$p_t$ & offset \\
$A_\theta(T)$ & amplification over $T$ outer steps \\
$S$, $\hat\mu_t$, $\bar L_t$ & number of independent runs; their average; their mean loss \\
\end{xltabular}
}

\section{Glossary of Terms}
\label{app:glossary}

The table below lists the terms the paper defines, grouped in the same way as the notation in Appendix~\ref{app:notation}.

\medskip
{\small
\setlength{\LTpre}{0pt}\setlength{\LTpost}{0pt}
\renewcommand{\theHtable}{glossary.\arabic{table}}%
\renewcommand{\arraystretch}{1.35}
\settowidth{\tabtermwidth}{shared-minibatch perturbation}\addtolength{\tabtermwidth}{1pt}
\begin{xltabular}{\linewidth}{@{}p{\tabtermwidth}>{\raggedright\arraybackslash}X@{}}
\toprule
\textsc{Term} & \textsc{Meaning} \\
\midrule
\endfirsthead
\multicolumn{2}{@{}l}{\small\itshape Glossary, continued.}\\[2pt]
\toprule
\textsc{Term} & \textsc{Meaning} \\
\midrule
\endhead
\midrule
\multicolumn{2}{r@{}}{\small\itshape continued on the next page}\\
\endfoot
\bottomrule
\endlastfoot
\rowcolor{gray!15}\multicolumn{2}{@{}l}{\textsc{Two-phase optimizer}} \\*\addlinespace[2pt]
inner trajectory & one worker's $K$ steps of minibatch SGD \\
block & an inner trajectory, seen as a map from its start to its displacement \\
block length & the number $K$ of local steps \\
pseudo-gradient & the displacement of an inner trajectory; the outer step uses its average over the workers \\
outer step & one update of the outer optimizer \\
workers & the $W$ copies that run inner trajectories from the same parameters \\
stream & the sequence of minibatches that one run draws \\
\midrule
\rowcolor{gray!15}\multicolumn{2}{@{}l}{\textsc{Model and stability}} \\*\addlinespace[2pt]
interpolating minimum & a minimizer shared by all sample losses \citep{mulayoff2024exact} \\
regular minimum & the minimizer of the full loss when the sample losses have different minimizers \\
block operator & the random matrix that maps the error at the start of a block to its pseudo-gradient \\
common flat directions & directions that no block changes \\
mean stable & $\E[z_t]\to0$ \\
mean-square stable & $\E\opnorm{z_t}^2\to0$ \\
positive map & a linear map that keeps positive-semidefinite matrices positive semidefinite \\
second-moment map & carries the second moment of the state one outer step forward \\
mean transition & the outer step applied to the mean of the state \\
block noise & the covariance of the block operator \\
mean-stability analysis & an analysis that replaces the block operator by its mean \\
mean condition & $\rho(A)<1$ \\
noise gain & the error that the mean dynamics build up from injected noise \\
loop gain & the growth around the loop from error to block noise and back, $\rho(\Vop_K\circ\Gop_{\alpha,\beta})$ \\
noise budget & one over the loop gain \\
noise scale & the small-step quantity $\alpha\kappa_K/((1-\beta)W)$, which has the form of the noise scale of \citet{smith2018bayesian} \\
inner-batch threshold & the batch size $b_c$ above which inner SGD is mean-square stable \\
mean and mean-square boundaries & edges of the mean-stable and mean-square-stable regions in the $(\alpha,\beta)$ plane \\
\midrule
\rowcolor{gray!15}\multicolumn{2}{@{}l}{\textsc{Network checkpoint}} \\*\addlinespace[2pt]
offset & the pseudo-gradient at the reference point \\
spread & the covariance of the parameters across independent runs \\
block Jacobian & the Jacobian of the pseudo-gradient with respect to its starting point \\
frozen linearization & the block Jacobian taken at the checkpoint for every outer step \\
shared-minibatch perturbation & the difference between two runs that see the same minibatches \\
amplification & the growth of a shared-minibatch perturbation over $T$ outer steps; a setting amplifies when it exceeds one \\
contracts & said of a setting whose amplification is below one \\
spread cost & the part of the change in mean loss caused by the runs' disagreement \\
averaged-model change & the change in loss of the runs' average \\
\end{xltabular}
}

\section{Proofs}
\label{app:proofs}

Throughout these proofs, we mainly use two facts about positive maps. The first is that the spectral radius is monotone in the order of Section~\ref{sec:block-moments}, $\rho(\mathcal M)\le\rho(\mathcal N)$ whenever $0\preceq\mathcal M\preceq\mathcal N$, which is proved in Appendix~\ref{app:proof-averaging}. The second is that the spectral radius of a positive map is an eigenvalue with a positive-semidefinite eigenmatrix, which is the Krein--Rutman theorem for the cone of positive-semidefinite matrices \citep[Chapter~1, Theorem~3.2]{berman1994nonnegative}.

\subsection{Derivation of the block moments}
\label{app:proof-moments}

Let $M_0=I$ and $M_{k+1}=Q_{\xi_k}M_k$. Under independence $Q_{\xi_k}$ is independent of $M_k$, so $\E[M_{k+1}]=\bar Q\,\E[M_k]$, and since $M_{k+1}\kron M_{k+1}=(Q_{\xi_k}\kron Q_{\xi_k})(M_k\kron M_k)$, also $\E[M_{k+1}\kron M_{k+1}]=\mathcal Q\,\E[M_k\kron M_k]$. A simple induction then gives $\E[M_K]=\bar Q^K$ and $\E[M_K\kron M_K]=\mathcal Q^K$, and expanding $(I-M_K)\kron(I-M_K)$ gives \eqref{eq:block-moments}.\footnote{If instead one minibatch is drawn per block and reused for all $K$ steps, then $M_K=Q_\xi^K$. This gives $\bar D_K=I-\E[Q_\xi^K]$ and $\Sm_K=\E[(I-Q_\xi^K)\kron(I-Q_\xi^K)]$. The moment condition of Assumption~\ref{ass:block} must then be imposed on the block directly, because a finite second moment of $Q_\xi$ does not imply one for $Q_\xi^K$.}

\subsection{Proof of the moment recurrences, the decomposition and Theorem~\ref{thm:full}(i)}
\label{app:proof-full}

\paragraph{The moment recurrences.} Recall that under Assumption~\ref{ass:block} the block operator $D_t$ is independent of $z_t$. Conditioning on $z_t$ in $z_{t+1}=\T(D_t)z_t$ gives $\E[z_{t+1}\mid z_t]=\T(\bar D_K)z_t$, and conditioning in $z_{t+1}z_{t+1}^\top=\T(D_t)z_tz_t^\top\T(D_t)^\top$ gives $\E[z_{t+1}z_{t+1}^\top\mid z_t]=\Lop_K(z_tz_t^\top)$. Taking expectations on both sides then proves \eqref{eq:moment-recurrences}, since $\Lop_K$ is linear.

\paragraph{Part (i).} As a first step, column-wise vectorization turns the second recurrence of \eqref{eq:moment-recurrences} into
\[
\vecop\Sigma_{t+1}=\mathcal M_K\vecop\Sigma_t,
\qquad
\mathcal M_K=\E[\T(D_K)\kron\T(D_K)],
\]
because $\vecop(\T X\T^\top)=(\T\kron\T)\vecop X$. If $\rho(\mathcal M_K)<1$ then $\mathcal M_K^t\to0$, so $\Sigma_t\to0$ and $\E\opnorm{z_t}^2=\operatorname{tr}\Sigma_t\to0$ for every admissible initial state. Conversely, suppose $\E\opnorm{z_t}^2\to0$ for every admissible $z_0$. Every positive-semidefinite matrix is the second moment of some admissible $z_0$, so $\Lop_K^t(\Sigma_0)\to0$ for every positive-semidefinite $\Sigma_0$; these matrices span the symmetric matrices, hence $\Lop_K^t\to0$ on symmetric matrices and $\rho(\Lop_K)<1$.

It remains to check that the condition $\rho(\Lop_K)<1$ is the same as $\rho(\mathcal M_K)<1$. One direction is immediate, since $\Lop_K$ is the restriction of $\mathcal M_K$ to the invariant subspace of symmetric matrices. For the other, let $Y_t=\T(D_{t-1})\cdots\T(D_0)$, so that $\mathcal M_K^t\vecop X=\vecop\E[Y_tXY_t^\top]$ for every matrix $X$. If $\rho(\Lop_K)<1$ then $\E[Y_te_ke_k^\top Y_t^\top]\to0$ for every basis vector $e_k$, so every $\E[(Y_t)_{ik}^2]\to0$, and the Cauchy--Schwarz inequality bounds each entry of $\E[Y_tXY_t^\top]$ by $\sum_{k,l}|X_{kl}|\,(\E[(Y_t)_{ik}^2]\,\E[(Y_t)_{jl}^2])^{1/2}\to0$. Hence $\mathcal M_K^t\to0$ and $\rho(\mathcal M_K)<1$. This is the standard mean-square stability criterion for linear recurrences with i.i.d.\ random coefficients \citep{costa2005mjls}.

\paragraph{The decomposition.} Note that the operator $D$ enters $\T(D)$ only through its first block column, so
\[
\T(D)=\T(\bar D_K)+\Bm(D-\bar D_K)\Em,
\qquad
\Bm=\begin{bmatrix}-\alpha(1+\beta)I\\ I\end{bmatrix},
\quad
\Em=\begin{bmatrix}I&0\end{bmatrix}.
\]
Substituting into \eqref{eq:lift} with $A=\T(\bar D_K)$ and $\Delta=D_K-\bar D_K$,
\[
\Lop_K(X)=AXA^\top+AX\Em^\top\E[\Delta]^\top\Bm^\top+\Bm\E[\Delta]\Em XA^\top+\Bm\,\E[\Delta\,\Em X\Em^\top\Delta^\top]\,\Bm^\top,
\]
and the two middle terms vanish simply because $\E[\Delta]=0$. The last term is $\Bm\Vop_K(\Em X\Em^\top)\Bm^\top$. Under vectorization, $\vecop\Vop_K(Y)=\E[\Delta\kron\Delta]\vecop Y=(\Sm_K-\bar D_K\kron\bar D_K)\vecop Y$, so this is exactly where the block moments \eqref{eq:block-moments} enter the criterion.

\paragraph{Starting from an empty buffer.} The necessity argument above lets the initial buffer be arbitrary, while in practice the algorithm starts with $u_{-1}=0$. Fortunately, the conclusion still holds when the buffer starts empty. Let $\Lop_K^*(\Phi)=\E[\T(D_K)^\top\Phi\,\T(D_K)]$ be the adjoint of $\Lop_K$ for the trace inner product $\langle X,Y\rangle=\operatorname{tr}(X^\top Y)$; it is positive, has the same spectral radius $\rho$, and therefore has an eigenmatrix $\Phi\succeq0$ at $\rho$. Write $\Phi$ in the blocks of $z=(e,u)$. Now suppose that the error block $\Phi_{ee}$ vanished. Then positive semidefiniteness would force $\Phi_{eu}=0$, so $\Phi=\operatorname{diag}(0,\Phi_{uu})$, and a direct multiplication gives
\[
\T(D)^\top\begin{bmatrix}0&0\\0&\Phi_{uu}\end{bmatrix}\T(D)
=\begin{bmatrix}D^\top\Phi_{uu}D&\beta D^\top\Phi_{uu}\\ \beta\Phi_{uu}D&\beta^2\Phi_{uu}\end{bmatrix},
\]
whose buffer block is $\beta^2\Phi_{uu}$ for every $D$. The eigenvalue equation on that block reads $\beta^2\Phi_{uu}=\rho\,\Phi_{uu}$ with $\Phi_{uu}\ne0$, so $\rho=\beta^2<1$. Hence if $\rho\ge1$ then $\Phi_{ee}\ne0$, and for $\Sigma_0=\operatorname{diag}(I,0)$, an isotropic error with an empty buffer, $\langle\Phi,\Sigma_0\rangle=\operatorname{tr}\Phi_{ee}>0$ and
\[
\langle\Phi,\Lop_K^t(\Sigma_0)\rangle=\langle(\Lop_K^*)^t\Phi,\Sigma_0\rangle=\rho^t\langle\Phi,\Sigma_0\rangle
\]
does not vanish. Since $\langle\Phi,\Sigma_t\rangle\le\lambda_{\max}(\Phi)\operatorname{tr}\Sigma_t$, neither does $\E\opnorm{z_t}^2=\operatorname{tr}\Sigma_t$. So an isotropic initial error with an empty buffer is already not mean-square stable whenever $\rho(\Lop_K)\ge1$.

\paragraph{Almost-sure and mean-square stability.} If $\rho(\Lop_K)<1$ then $\E\opnorm{z_t}^2=\operatorname{tr}\Lop_K^t(\Sigma_0)$ decays geometrically, so $\E\sum_t\opnorm{z_t}^2<\infty$, the sum is finite almost surely, and $z_t\to0$ almost surely. The converse, however, fails in general.

\subsection{Proof of Theorem~\ref{thm:full}(ii) and (iii)}
\label{app:proof-budget}

\paragraph{Part (ii).} We start by writing the two terms of \eqref{eq:decomposition} as maps on symmetric matrices,
\[
\mathcal A(X)=AXA^\top,
\qquad
\mathcal N(X)=\Bm\,\Vop_K(\Em X\Em^\top)\,\Bm^\top ,
\]
so that $\Lop_K=\mathcal A+\mathcal N$ and both are positive. Since $\mathcal A\preceq\Lop_K$, $\rho(A)^2=\rho(\mathcal A)\le\rho(\Lop_K)$, so mean stability is necessary and we may assume $\rho(\mathcal A)<1$, in which case $(\mathcal I-\mathcal A)^{-1}=\sum_{t\ge0}\mathcal A^t$ is a positive map. Let $\mathcal P=\mathcal N(\mathcal I-\mathcal A)^{-1}$.

Suppose first $\rho(\Lop_K)<1$ and, for contradiction, $\rho(\mathcal P)\ge1$. The spectral radius of a positive map is an eigenvalue with a positive-semidefinite eigenmatrix, so $\mathcal PX=\rho(\mathcal P)X$ for some $X\succeq0$, $X\neq0$. Put $Y=(\mathcal I-\mathcal A)^{-1}X\succeq0$, which is nonzero. Then $\mathcal NY=\rho(\mathcal P)X=\rho(\mathcal P)(\mathcal I-\mathcal A)Y$, hence $\bigl(\mathcal A+\rho(\mathcal P)^{-1}\mathcal N\bigr)Y=Y$. As $\rho(\mathcal P)\ge1$ we have $\mathcal A+\rho(\mathcal P)^{-1}\mathcal N\preceq\Lop_K$, and the spectral radius is monotone on the cone, so $\rho(\Lop_K)\ge1$, a contradiction.

For the converse, let $\rho(\mathcal A)<1$ and $\rho(\mathcal P)<1$, and suppose $\mu=\rho(\Lop_K)\ge1$ with eigenmatrix $Y\succeq0$. Then $(\mu\mathcal I-\mathcal A)Y=\mathcal NY$, and $\mu\mathcal I-\mathcal A$ is invertible with positive inverse because $\mu\ge1>\rho(\mathcal A)$. Set $X=\mathcal NY\succeq0$; it is nonzero, since $X=0$ would give $\mathcal AY=\mu Y$ against $\rho(\mathcal A)<1$. Comparing the Neumann series termwise gives $(\mu\mathcal I-\mathcal A)^{-1}\preceq(\mathcal I-\mathcal A)^{-1}$ for $\mu\ge1$, so
\[
X=\mathcal N(\mu\mathcal I-\mathcal A)^{-1}X\preceq\mathcal N(\mathcal I-\mathcal A)^{-1}X=\mathcal PX .
\]
Then $\mathcal P^tX\succeq X$ for every $t$, so $\mathcal P^t$ does not converge to zero and $\rho(\mathcal P)\ge1$, a contradiction.

It remains to identify $\mathcal P$. Because $\mathcal N=X\mapsto\Bm\Vop_K(\Em X\Em^\top)\Bm^\top$ factors through the error block, and swapping the two factors of a composition leaves its spectral radius unchanged,
\[
\rho(\mathcal P)
=\rho\Bigl(\Vop_K\circ\bigl[Y\mapsto\Em(\mathcal I-\mathcal A)^{-1}(\Bm Y\Bm^\top)\Em^\top\bigr]\Bigr)
=\rho\bigl(\Vop_K\circ\Gop_{\alpha,\beta}\bigr),
\]
since $\Em\sum_tA^t(\Bm Y\Bm^\top)(A^\top)^t\Em^\top=\sum_t(\Em A^t\Bm)Y(\Em A^t\Bm)^\top$, which is \eqref{eq:noise-gain}.

\paragraph{Part (iii).} For the mean condition, let $\bar D_K=URU^{*}$ be a Schur form, with $U$ unitary and $R$ upper triangular with the eigenvalues $\lambda_i$ on its diagonal. Every block of $A=\T(\bar D_K)$ is a polynomial in $\bar D_K$, so $(U\oplus U)^{*}A\,(U\oplus U)=\T(R)$ has upper-triangular blocks, and reordering the coordinates into the pairs $(e_i,u_i)$ makes it block upper triangular with the $2\times2$ diagonal blocks $\T(\lambda_i)$. The eigenvalues of $A$ are therefore those of the $\T(\lambda_i)$, and $\rho(A)<1$ if and only if every $\T(\lambda_i)$ is stable. The characteristic polynomial of $\T(\lambda)$ is $\mu^2-(1+\beta)(1-\alpha\lambda)\mu+\beta(1-\alpha\lambda)$; for real $\lambda$ its stability is the Jury computation of Appendix~\ref{app:proof-scalar} at $d_1=\lambda$, which gives \eqref{eq:mean-condition}.

\subsection{The criterion in one dimension}
\label{app:proof-scalar}

Write $d_1=\E[D]$ and $d_2=\E[D^2]$, $x=\alpha d_1$ for the outer step in units of the mean block operator, and $\Nk=(d_2-d_1^2)/d_1^2$ for the relative variance of the block operator. In one dimension \eqref{eq:budget-criterion} can be solved explicitly. Under \eqref{eq:mean-condition} the recurrence is mean-square stable if and only if
\begin{equation}
\Nk<r(x,\beta):=\frac{[\,2(1+\beta)-x(1+2\beta)\,]\,(1-\beta+\beta x)}{x\,[\,(1+\beta)+\beta(1+2\beta)x\,]}.
\label{eq:budget-general}
\end{equation}
At $x=1$ the bound is $r_\beta:=r(1,\beta)=1/(1+2\beta+2\beta^2)$, which falls from $1$ at $\beta=0$ towards $1/5$ as $\beta\to1$. This is \eqref{eq:budget-criterion} written out in one dimension. \citet{andreyev2026momentum} write the same scalar recurrence for a single SGD step. The rest of this subsection proves the criterion \eqref{eq:budget-general}.

\paragraph{The second-moment matrix.} In one dimension \eqref{eq:transition} reads
\[
\begin{bmatrix}e_{t+1}\\u_t\end{bmatrix}
=\T(D_t)
\begin{bmatrix}e_t\\u_{t-1}\end{bmatrix},
\qquad
\T(D)=
\begin{bmatrix}
1-\alpha(1+\beta)D&-\alpha\beta^2\\
D&\beta
\end{bmatrix}.
\]
To keep the formulas short, let $\gamma=\alpha(1+\beta)$ and $h=-\alpha\beta^2$. Expanding $e_{t+1}^2$, $e_{t+1}u_t$ and $u_t^2$ and taking expectations, independence of $D_t$ from the past gives $\smom_{t+1}=G\smom_t$ for the second-moment vector $\smom_t=(\E[e_t^2],\E[e_tu_{t-1}],\E[u_{t-1}^2])^\top$, with
\begin{equation}
G=
\begin{bmatrix}
1-2\gamma d_1+\gamma^2d_2 & 2h(1-\gamma d_1) & h^2\\
d_1-\gamma d_2 & \beta(1-\gamma d_1)+hd_1 & h\beta\\
d_2 & 2\beta d_1 & \beta^2
\end{bmatrix},
\label{eq:snoo-second-moment-matrix}
\end{equation}
the matrix of $\Lop_K$ in the coordinates of a symmetric $2\times2$ matrix, so that $\rho(G)<1$ is Theorem~\ref{thm:full} in one dimension.

\paragraph{The mean condition.} The deterministic mean transition $\T(d_1)$ has trace
$(1+\beta)(1-\alpha d_1)$ and determinant $\beta(1-\alpha d_1)$. The three
Jury inequalities reduce to
\[
\alpha d_1>0,
\qquad
2(1+\beta)-\alpha d_1(1+2\beta)>0,
\qquad
1-\beta+\alpha\beta d_1>0.
\]
The remaining condition, $\beta(1-\alpha d_1)>-1$, follows from the second. For $0\leq\beta<1$ these are equivalent to \eqref{eq:mean-condition}, which is therefore necessary and sufficient for mean stability. Mean-square stability implies mean stability, so \eqref{eq:mean-condition} is necessary for it as well, and we assume it from now on.

\paragraph{The critical second moment.} Write $D=d_1+\varepsilon$, where $\E[\varepsilon]=0$ and
$\E[\varepsilon^2]=d_2-d_1^2$. Since $\T(D)$ is affine in $D$,
\[
\T(D)=\T(d_1)+\varepsilon C,
\qquad
C=\Bm\Em=\begin{bmatrix}-\alpha(1+\beta)&0\\1&0\end{bmatrix}.
\]
The second-moment map \eqref{eq:decomposition} is therefore the positive linear map
\begin{equation}
\mathcal L_{d_2}(X)
=\T(d_1)X\T(d_1)^\top
+(d_2-d_1^2)CXC^\top.
\label{eq:positive-lift}
\end{equation}
Note that the map is continuous
and monotone in $d_2$ on the positive-semidefinite cone. At $d_2=d_1^2$ its
spectral radius is $\rho(\T(d_1))^2<1$. The spectral radius of a positive finite-dimensional covariance map is a
real eigenvalue with a positive-semidefinite eigenmatrix. Hence its first
unit-circle crossing as $d_2$ increases occurs at eigenvalue $+1$, where $\det(I-G)=0$. A direct, if slightly tedious, determinant calculation
from \eqref{eq:snoo-second-moment-matrix} gives
\begin{align}
\det(I-G)=\alpha\bigl[&2d_1(1-\beta^2)
+2\alpha\beta(1+2\beta)d_1^2\nonumber\\
&-\alpha d_2\{(1+\beta)
+\alpha\beta(1+2\beta)d_1\}\bigr].
\label{eq:det-one-minus-G}
\end{align}
Setting it to zero and solving for $d_2$ gives the critical value
\begin{equation}
\dcrit:=\frac{2d_1\,[\,1-\beta^2+\alpha\beta d_1(1+2\beta)\,]}{\alpha\,[\,1+\beta+\alpha\beta d_1(1+2\beta)\,]},
\label{eq:d2crit}
\end{equation}
and writing $x=\alpha d_1$, direct rearrangement gives
$(\dcrit-d_1^2)/d_1^2=r(x,\beta)$ with $r$ as in \eqref{eq:budget-general}.
The mean condition makes every factor of $r(x,\beta)$ positive, so the
critical value is feasible and strictly exceeds $d_1^2$.

Putting the pieces together, equation~\eqref{eq:det-one-minus-G} is affine and strictly decreasing in
$d_2$, so it has the unique root \eqref{eq:d2crit}; below the root $\rho(G)<1$, and above it monotonicity in $d_2$ keeps $\rho(G)\ge1$. This proves
that, under \eqref{eq:mean-condition}, the recurrence is mean-square stable if and only if $d_2<\dcrit$, which is the statement $\Nk<r(x,\beta)$ of \eqref{eq:budget-general}.

Setting $\alpha d_1=1$ in the critical value and subtracting $d_1^2$ gives
$r(1,\beta)=r_\beta=1/(1+2\beta+2\beta^2)$.

\paragraph{The bound as an impulse-response sum.} With the scalar impulse response $\varphi_t(d_1)=\Em\T(d_1)^t\Bm$, the bound reads $r(x,\beta)^{-1}=d_1^2\sum_{t\ge0}\varphi_t(d_1)^2$, since in one dimension $\rho(\Vop_K\circ\Gop_{\alpha,\beta})=(d_2-d_1^2)\sum_t\varphi_t(d_1)^2$ and the boundary of Theorem~\ref{thm:full}(ii) is the same root.

\subsection{The criterion at small outer steps}
\label{app:smallstep}

At a small outer step the criterion of Theorem~\ref{thm:full} takes the form of a noise scale; Section~\ref{sec:stability} states it informally.

\begin{proposition}[Small outer steps]
\label{prop:smallstep}
Let the pseudo-gradient of each outer step be the average over $W\ge1$ workers of inner trajectories that start from the same point, are independent given that point, and have block operators as in Assumption~\ref{ass:block}, let the mean condition hold, and set
\begin{equation}
\nu:=\frac{\alpha}{(1-\beta)\,W}\,\kappa_K,
\qquad
\kappa_K=\rho\bigl(\Vop_K\circ\Sop^{-1}\bigr),
\qquad
\Sop(X)=\bar D_KX+X\bar D_K .
\label{eq:smallstep}
\end{equation}
\begin{enumerate}\setlength{\itemsep}{1pt}\setlength{\parskip}{0pt}
\item[(a)] If $\bar D_K$ is symmetric, then on the complement of the common flat directions the criterion \eqref{eq:budget-criterion} reads $\nu<1$ to leading order in $\alpha$.
\item[(b)] In one dimension $\kappa_K=\Var(D_K)/(2\E[D_K])$, and the recurrence \eqref{eq:transition} is mean-square stable if and only if
\begin{equation}
\nu<F(x,\beta):=\frac{x\,r(x,\beta)}{2(1-\beta)},
\qquad
x=\alpha\E[D_K],
\label{eq:noise-scale-criterion}
\end{equation}
with $r$ the bound \eqref{eq:budget-general}. The factor satisfies $F(0,\beta)=1$; it is below $1$ at every mean-stable $x>0$ when $\beta\le\beta_0\approx0.603$, and above $1$ on an interval $0<x<x_c(\beta)$ when $\beta>\beta_0$.
\end{enumerate}
\end{proposition}

The quantity $\nu$ has the form of the noise scale $g\approx\epsilon N/(B(1-m))$ that \citet{smith2018bayesian} derive for one heavy-ball SGD step and that \citet{smith2018increase} use to trade learning rate, batch size and momentum against one another, with the outer step $\alpha$ in place of $\epsilon$, the number of workers $W$ in place of the batch size $B$, the outer momentum $\beta$ in place of $m$, and the block constant $\kappa_K$ in place of the training-set size $N$. The quantity $\nu$ matches $g$ only in form, since it is a rewriting of Theorem~\ref{thm:full}(ii) for a $K$-step block under outer Nesterov momentum. Part (b) says when the familiar rule $\nu<1$ is right, and shows that below $\beta_0$ the rule admits settings that are mean-square unstable while above $\beta_0$ it rejects settings that are stable at small outer steps. The closed form of $F-1$, the root $\beta_0$, the interval $x_c(\beta)$ and the case $\beta=0$ are derived below, but the results of Section~\ref{sec:implications} use the exact criterion in every dimension. The paragraphs below prove the proposition, part (a) from the scalar series, the kernel and the matrix statement, and part (b) from the exact scalar factor.

\paragraph{The one-dimensional bound at small steps.} In \eqref{eq:budget-general} the numerator tends to $2(1+\beta)(1-\beta)$ and the bracket of the denominator to $1+\beta$ as $x\to0$, so
\[
r(x,\beta)=\frac{2(1-\beta)}{x}+c_0(\beta)+O(x),
\qquad
c_0(\beta)=\frac{4\beta^3+2\beta^2-\beta-1}{1+\beta} .
\]
The criterion $\Nk/W<r(x,\beta)$ therefore reads $\Nk\,x/\bigl(2(1-\beta)W\bigr)<1$ to leading order, which is \eqref{eq:smallstep} since $x=\alpha\E[D_K]$ and $\Nk=\Var(D_K)/\E[D_K]^2$, and the relative correction to the threshold is $c_0(\beta)\,x/(2(1-\beta))$.

\paragraph{The noise gain in an eigenbasis.} For two eigenvalues $\lambda,\mu$ of a symmetric $\bar D_K$ the gain acts entrywise in an eigenbasis through $k(\lambda,\mu)=\sum_{t\ge0}\varphi_t(\lambda)\varphi_t(\mu)$, which is the error-block entry of the solution $X$ of the Stein equation $X-\T(\lambda)X\T(\mu)^\top=\Bm\Bm^\top$ with the scalar $\T$ and $\Bm$ of Appendix~\ref{app:proof-scalar}. Solving it and expanding in $\alpha$,
\[
k(\lambda,\mu)=\frac{\alpha}{(1-\beta)(\lambda+\mu)}\bigl(1+O(\alpha)\bigr).
\]
The leading term is the kernel of $\tfrac{\alpha}{1-\beta}\Sop^{-1}$, since $\Sop$ acts entrywise by $\lambda+\mu$.

\paragraph{The loop gain at small outer steps.} For a symmetric $\bar D_K$ the expansion holds entry by entry in an eigenbasis, so on the complement of the common flat directions, where $\Sop$ is invertible, $\Gop_{\alpha,\beta}=\tfrac{\alpha}{1-\beta}\Sop^{-1}(1+o(1))$ and, by the continuity of the spectral radius, $\rho(\Vop_K\circ\Gop_{\alpha,\beta})=\tfrac{\alpha}{1-\beta}\kappa_K(1+o(1))$.

\paragraph{The exact criterion in one dimension.} Averaging over $W$ workers leaves $d_1$ unchanged and replaces $d_2-d_1^2$ by $(d_2-d_1^2)/W$ (Appendix~\ref{app:proof-averaging}), so under \eqref{eq:mean-condition} the recurrence is mean-square stable if and only if $\Nk/W<r(x,\beta)$. The mean condition gives $x>0$, and $0\le\beta<1$, so multiplying both sides by $x/(2(1-\beta))>0$ is an equivalence and gives $\nu<F(x,\beta)$ with $F=xr/(2(1-\beta))$, since $x\Nk/(2(1-\beta)W)=\alpha(d_2-d_1^2)/(2d_1(1-\beta)W)=\nu$.

\paragraph{The sign of $F-1$.} To find the closed form of $F-1$, we put $F$ and $1$ over the common denominator $2(1-\beta)[(1+\beta)+\beta(1+2\beta)x]$; the numerator is
\[
\bigl[2(1+\beta)-(1+2\beta)x\bigr](1-\beta+\beta x)-2(1-\beta)\bigl[(1+\beta)+\beta(1+2\beta)x\bigr],
\]
whose constant term vanishes, whose coefficient of $x$ is $2\beta(1+\beta)-(1+2\beta)(1-\beta)-2\beta(1-\beta)(1+2\beta)=p(\beta)$, and whose coefficient of $x^2$ is $-\beta(1+2\beta)$. Hence
\begin{equation}
F(x,\beta)-1=\frac{x\,\bigl[\,p(\beta)-\beta(1+2\beta)\,x\,\bigr]}{2(1-\beta)\bigl[(1+\beta)+\beta(1+2\beta)x\bigr]},
\qquad
p(\beta)=4\beta^3+2\beta^2-\beta-1,
\label{eq:F-minus-one}
\end{equation}
and its first-order term $p(\beta)x/(2(1-\beta^2))$ is the correction $c_0(\beta)x/(2(1-\beta))$ of the one-dimensional expansion above.

Since $p'(\beta)=(2\beta+1)(6\beta-1)$, $p$ decreases on $[0,1/6]$ from $p(0)=-1$ to $p(1/6)=-59/54$ and increases on $[1/6,1]$ to $p(1)=4$, so it has one root $\beta_0$ in $[0,1)$, at $\beta_0=0.60278\ldots$, and is negative on $[0,\beta_0)$ and positive on $(\beta_0,1)$. Consider first $\beta\le\beta_0$. Then the bracket $p(\beta)-\beta(1+2\beta)x$ is negative for every $x>0$, so $F<1$ and mean-square stability forces $\nu<F<1$. For $\beta>\beta_0$ the bracket is positive if and only if $x<x_c(\beta)=p(\beta)/(\beta(1+2\beta))$, and
\[
\frac{2(1+\beta)}{1+2\beta}-x_c(\beta)=\frac{(1-\beta)(1+2\beta)}{\beta}>0 ,
\]
so this range lies inside the mean-stable range, $F>1$ on it, and $\nu<1<F$ gives mean-square stability. At $\beta=0$, \eqref{eq:F-minus-one} gives $F=1-x/2$, and $\nu<1-x/2$ with $\nu=\alpha(d_2-d_1^2)/(2d_1W)$ and $x=\alpha d_1$ rearranges at $W=1$ to $\alpha d_2<2d_1$, the threshold \eqref{eq:d2crit} at $\beta=0$.

\paragraph{$\kappa_K$ along $K$ in one dimension.} With $\bar q=\E[Q_\xi]$ and $s_q=\E[Q_\xi^2]$, \eqref{eq:block-moments} gives $\E[D_K]=1-\bar q^K$ and $\Var(D_K)=s_q^K-\bar q^{2K}$, so
\[
\kappa_K=\frac{s_q^K-\bar q^{2K}}{2\,(1-\bar q^K)},
\qquad
\frac{\kappa_2}{\kappa_1}=\frac{s_q+\bar q^2}{1+\bar q},
\]
and $\kappa_2<\kappa_1$ if and only if $s_q<1+\bar q(1-\bar q)$. With $\lvert\bar q\rvert<1$ the limits are $\kappa_K\to0$ when $s_q<1$, $\kappa_K\to\tfrac12$ when $s_q=1$ and $\kappa_K\to\infty$ when $s_q>1$; the middle case is the marginal case $\rho(\Qop)=1$ of Appendix~\ref{app:proof-blocklength} at small $\alpha$, where $1/r(\alpha,\beta)=\tfrac{\alpha}{2(1-\beta)}\,(1+O(\alpha))$.

\subsection{Proof of Proposition~\ref{prop:momentum}}
\label{app:proof-momentum}

Throughout, $\alpha$ and the block mean $\bar D:=\bar D_K$ are fixed, and $A_\beta=\T(\bar D)$, $\Bm_\beta$ and $\Gop_{\alpha,\beta}$ carry the momentum in their index. The proof first sets up a generating function in which $\beta$ enters through one quadratic polynomial, then relates two such polynomials by subordination, applies Littlewood's principle, and finally passes from rank-one matrices to the cone.

\paragraph{The generating function.} Let $\Phi_t=\Em A_\beta^t\Bm_\beta$, so that $\Gop_{\alpha,\beta}(Y)=\sum_{t\ge0}\Phi_tY\Phi_t^\top$ whenever the series converges. For $|w|<1/\rho(A_\beta)$,
\begin{equation}
\sum_{t\ge0}\Phi_t\,w^{t+1}=w\,\Em\,(I-wA_\beta)^{-1}\Bm_\beta=\Psi\bigl(P_\beta(w)\bigr),
\label{eq:transfer}
\end{equation}
where
\[
\Psi(\sigma)=-\alpha\sigma\,\bigl[I-(I-\alpha\bar D)\sigma\bigr]^{-1},
\qquad
P_\beta(w)=(1+\beta)w-\beta w^2 .
\]
The key observation is that every block of $A_\beta$ and of $\Bm_\beta$ is a polynomial in $\bar D$, so both sides are power series in $w$ whose coefficients are polynomials in $\bar D$, and it suffices to verify the identity when $\bar D$ is a scalar $\lambda$. There it is the resolvent of the $2\times2$ matrix $\T(\lambda)$, and a direct computation gives
\[
w\,\Em\,\bigl(I-w\T(\lambda)\bigr)^{-1}\Bm_\beta=\frac{-\alpha\,P_\beta(w)}{1-(1-\alpha\lambda)P_\beta(w)},
\qquad
\det\bigl(I-w\T(\lambda)\bigr)=1-(1-\alpha\lambda)P_\beta(w).
\]
The second identity, applied eigenvalue by eigenvalue through a Schur form as in Appendix~\ref{app:proof-budget}, gives $\det(I-wA_\beta)=\det\bigl(I-(I-\alpha\bar D)P_\beta(w)\bigr)$. Hence $w\mapsto\Psi(P_\beta(w))$ is analytic on a neighborhood of the closed unit disc if and only if no eigenvalue of $A_\beta$ has modulus at least one, that is, if and only if $(\alpha,\beta)$ is mean stable; the coefficients $\Phi_t$ then decay geometrically and $\Gop_{\alpha,\beta}$ is well defined.

\paragraph{The gain as an $H^2$ norm.} Fix real vectors $v$ and $y$ and put $f_\beta(w)=v^\top\Psi(P_\beta(w))\,y=\sum_t(v^\top\Phi_ty)\,w^{t+1}$. Its coefficients are real, so
\[
v^\top\Gop_{\alpha,\beta}(yy^\top)\,v=\sum_{t\ge0}(v^\top\Phi_ty)^2=\lVert f_\beta\rVert_{H^2}^2 ,
\]
the squared norm of the Hardy space of the unit disc. Writing $g(\sigma)=v^\top\Psi(\sigma)\,y$, which does not depend on $\beta$, we have $f_\beta=g\circ P_\beta$.

\paragraph{Relating a lower momentum to a higher one.} This is the main step of the proof. Let $0\le\beta\le\beta'<1$. We show that $P_\beta=P_{\beta'}\circ\omega$ for an analytic $\omega:\mathbb D\to\mathbb D$ with $\omega(0)=0$. Parametrize the family by $\tau=\log(1+\beta)$ and write $P(w,\tau)=e^\tau w-(e^\tau-1)w^2$. Then, with $\beta=e^\tau-1$,
\[
\partial_\tau P=(1+\beta)\,w(1-w)=w\,\partial_wP\cdot p(w,\tau),
\qquad
p(w,\tau)=\frac{(1+\beta)(1-w)}{1+\beta-2\beta w},
\]
since $\partial_wP=(1+\beta)-2\beta w$. The map $p(\cdot,\tau)$ is a M\"obius transformation with real coefficients whose pole $(1+\beta)/(2\beta)$ lies outside the closed disc for $\beta<1$; it sends $1$ to $0$ and $-1$ to $2(1+\beta)/(1+3\beta)>0$, so it maps the unit circle onto the circle with that real diameter and the open disc onto its interior, and $\operatorname{Re}p>0$ on $\mathbb D$. Let $s=\log(1+\beta)$ and $s'=\log(1+\beta')$, and for $z\in\mathbb D$ let $w(\cdot)$ solve
\[
\dot w(\tau)=-w(\tau)\,p\bigl(w(\tau),\tau\bigr),
\qquad w(s)=z,
\qquad \tau\in[s,s'] .
\]
Along a solution, $\tfrac{d}{d\tau}|w|^2=-2|w|^2\operatorname{Re}p(w,\tau)\le0$, so $|w(\tau)|\le|z|<1$ and the solution stays in the disc, where the right-hand side is analytic in $w$ and continuous in $\tau$, uniformly on $[s,s']$ because the pole of $p$ stays outside the closed disc for every $\beta\le\beta'$; hence the solution exists on the whole interval and depends analytically on $z$. Moreover
\[
\frac{d}{d\tau}P\bigl(w(\tau),\tau\bigr)=\partial_wP\cdot\dot w+\partial_\tau P=-w\,\partial_wP\,p+w\,\partial_wP\,p=0 ,
\]
so $P(w(\tau),\tau)$ is constant along the solution. Setting $\omega(z)=w(s')$ gives an analytic self-map of the disc with $\omega(0)=0$, since $w\equiv0$ is a solution, and $P_\beta(z)=P(z,s)=P(\omega(z),s')=P_{\beta'}(\omega(z))$. Note that no univalence of $P_\beta$ is used.

\paragraph{The gain does not decrease with momentum.} Suppose $(\alpha,\beta')$ is mean stable. The inclusion $P_\beta(\mathbb D)=P_{\beta'}(\omega(\mathbb D))\subseteq P_{\beta'}(\mathbb D)$ shows that $(\alpha,\beta)$ is then mean stable as well. By the determinant identity, mean stability at $\beta'$ says that $1/(1-\alpha\lambda)$ lies outside the closure of $P_{\beta'}(\mathbb D)$ for every eigenvalue $\lambda$ of $\bar D$ with $\alpha\lambda\ne1$, hence outside the closure of $P_\beta(\mathbb D)$, which is mean stability at $\beta$. In particular $\Gop_{\alpha,\beta}$ is well defined. Moreover, $f_{\beta'}=g\circ P_{\beta'}$ is analytic on a neighborhood of the closed disc, hence in $H^2$, and $f_\beta=g\circ P_\beta=f_{\beta'}\circ\omega$ is analytic on the disc. Littlewood's subordination principle states that composition with an analytic self-map of the disc fixing the origin does not increase the $H^2$ norm \citep{duren1970hp}, so
\[
\lVert f_\beta\rVert_{H^2}=\lVert f_{\beta'}\circ\omega\rVert_{H^2}\le\lVert f_{\beta'}\rVert_{H^2}
\qquad\text{for every real } v \text{ and } y .
\]
By the $H^2$ identity above, the two sides are $v^\top\Gop_{\alpha,\beta}(yy^\top)v$ and $v^\top\Gop_{\alpha,\beta'}(yy^\top)v$.

\paragraph{From rank-one inputs to all inputs.} The inequality says $v^\top\bigl[\Gop_{\alpha,\beta'}-\Gop_{\alpha,\beta}\bigr](yy^\top)\,v\ge0$ for all $v$, that is, $\bigl[\Gop_{\alpha,\beta'}-\Gop_{\alpha,\beta}\bigr](yy^\top)\succeq0$ for every $y$. Every positive-semidefinite $Y$ is a sum of such rank-one matrices and the maps are linear, so $\Gop_{\alpha,\beta}\preceq\Gop_{\alpha,\beta'}$. Since $\Vop_K$ is positive, $0\preceq\Vop_K\circ\Gop_{\alpha,\beta}\preceq\Vop_K\circ\Gop_{\alpha,\beta'}$, and the monotonicity of the spectral radius on the cone (Appendix~\ref{app:proof-averaging}) gives $\rho(\Vop_K\circ\Gop_{\alpha,\beta})\le\rho(\Vop_K\circ\Gop_{\alpha,\beta'})$. With Theorem~\ref{thm:full}, mean-square stability at $(\alpha,\beta')$ implies it at $(\alpha,\beta)$. The argument does not diagonalize $\bar D$ or require its spectrum to be real, since the identity \eqref{eq:transfer} is a polynomial identity in $\bar D$ and the determinant identity handles complex eigenvalues.

In one dimension the same conclusion also follows from a much simpler sign computation, since differentiating \eqref{eq:budget-general} gives
\[
\partial_\beta\,r(x,\beta)^{-1}
=\frac{2x\bigl[(1+\beta)^2-\beta^2x\bigr]}
{\bigl[2(1+\beta)-(1+2\beta)x\bigr]^2\bigl[1-\beta+\beta x\bigr]^2} ,
\]
whose denominator is a square and whose numerator is positive because $x<2\le(1+\beta)^2/\beta^2$ under \eqref{eq:mean-condition}.

\subsection{Proof of Proposition~\ref{thm:blocklength}}
\label{app:proof-blocklength}

By \eqref{eq:block-moments}, $\Sm_K=I\kron I-I\kron\bar Q^K-\bar Q^K\kron I+\Qop^K$ and $\bar D_K\kron\bar D_K$ expands the same way with $\bar Q^K\kron\bar Q^K$ in place of $\Qop^K$, so
\[
\Vop_K=\Qop^K-\Qop_\infty^K,
\qquad
\Qop(Y)=\E[Q_\xi YQ_\xi],
\quad
\Qop_\infty(Y)=\bar QY\bar Q .
\]
Both are self-adjoint for the trace inner product because $Q_\xi$ is symmetric.

\paragraph{Part (i).} If $\rho(\Qop)<1$ then $\lVert\Vop_K\rVert\le\lVert\Qop^K\rVert+\lVert\Qop_\infty^K\rVert\to0$, while $\bar D_K=I-\bar Q^K\to I$ because $\rho(\bar Q)<1$. Hence $\Lop_K$ converges to $X\mapsto\T(I)X\T(I)^\top$, and the spectral radius is continuous, so $\rho(\Lop_K)\to\rho(\T(I))^2$, which is below one exactly under \eqref{eq:mean-condition} at $\lambda=1$, that is $0<\alpha<2(1+\beta)/(1+2\beta)$.

\paragraph{Part (ii).} If $\rho(\Qop)>1$, use $\Lop_K\succeq X\mapsto\Bm\Vop_K(\Em X\Em^\top)\Bm^\top$ and $\Em\Bm=-\gamma I$. Since swapping the two factors of a composition leaves its spectral radius unchanged, these give $\rho(\Lop_K)\ge\gamma^2\rho(\Vop_K)$. Let $X\succeq0$ with $\lVert X\rVert_F=1$ be the eigenmatrix of $\Qop$ at $\rho(\Qop)$. Then
\[
\langle X,\Vop_KX\rangle=\rho(\Qop)^K-\langle X,\bar Q^KX\bar Q^K\rangle\ge\rho(\Qop)^K-\rho(\bar Q)^{2K},
\]
and $\Vop_K$ is self-adjoint, so its spectral radius is at least this Rayleigh quotient.

\paragraph{Part (iii).} By \eqref{eq:batch-split}, $\Qop_b=\Qop_\infty+\tfrac1b\Cop$ with both terms positive, and $\rho(\Qop_\infty)=\rho(\bar Q)^2<1$ on the complement of the flat directions. The criterion for positive maps in the proof of Theorem~\ref{thm:full}(ii) (Appendix~\ref{app:proof-budget}) gives $\rho(\Qop_b)<1$ if and only if $\rho\bigl(\tfrac1b\Cop\circ(\mathcal I-\Qop_\infty)^{-1}\bigr)<1$, that is, $b>b_c$. For $b<b_c$ the inequality is strict, since the same criterion applied to $\Qop_b/\mu$ for $\mu\ge1$ gives $\rho(\Qop_b)<\mu$ if and only if $\rho\bigl(\tfrac1b\Cop\circ(\mu\mathcal I-\Qop_\infty)^{-1}\bigr)<1$, and this spectral radius is $b_c/b>1$ at $\mu=1$, hence above $1$ for some $\mu>1$ by continuity, so $\rho(\Qop_b)>1$. The last clause of (iii) is then (i) and (ii) with $\Qop=\Qop_b$.

\paragraph{Sampling without replacement.} The form \eqref{eq:batch-split} is for sampling with replacement; for batches of $b$ distinct samples out of $n$ the same split holds with $1/b$ replaced by $(n-b)/(b(n-1))$, the weight in the second-moment matrix of \citet{mulayoff2024exact}, and (iii) reads $(n-b)/(b(n-1))<1/b_c$.

\paragraph{Small inner rates.} For small $\eta$ write $\Cop=\eta^2\Cop_1$ with $\Cop_1(Y)=\E_i[(h_i-\bar H)Y(h_i-\bar H)]$; since $\mathcal I-\Qop_\infty=\eta\,\mathcal H-\eta^2\,\bar H(\cdot)\bar H$ with $\mathcal H(Y)=\bar HY+Y\bar H$, which is invertible on the complement of the flat directions, $\Cop\circ(\mathcal I-\Qop_\infty)^{-1}=\eta\,\Cop_1\circ\mathcal H^{-1}+O(\eta^2)$ and, the spectral radius being continuous, $b_c=\eta\,\rho(\Cop_1\circ\mathcal H^{-1})+o(\eta)$. The threshold $b_c$ therefore grows linearly with the inner rate for small rates.

\paragraph{Growth per inner gradient.} Since a block of length $K$ spends $K$ gradients, the growth per inner gradient is
\begin{equation}
\rho_K=\rho(\Lop_K)^{1/(2K)},
\label{eq:rho-k}
\end{equation}
the square root turning the factor of the second moment into a factor of the error itself and the $K$-th root converting the factor per block into a factor per gradient. Because $\rho(\Lop_K)<1$ exactly when $\rho_K<1$, the growth per inner gradient $\rho_K$ and the spectral radius $\rho(\Lop_K)$ agree on stability at every $K$, but only $\rho_K$ counts the $K$ gradients that a stable block spends. In terms of $\rho_K$, the two cases of Proposition~\ref{thm:blocklength} become $\rho_K\to1$ in case (i) when $\alpha\ne1$ and $\rho_K\to\rho(\Qop)^{1/2}$ in case (ii). Per gradient, lengthening a stable block stops helping once the mean block operator has saturated, because once $\bar D_K\approx I$ the contraction per block is fixed at $\rho(\T(I))^2$ while each block costs $K$ gradients, so $\rho_K$ returns towards one and a finite block length is best (Figure~\ref{fig:toy-blocklength}).

In the expanding case the lower bound gives $\liminf_K\rho_K\ge\rho(\Qop)^{1/2}$. For the matching upper bound, $\rho(\Lop_K)\le\lVert\Lop_K\rVert\le\lVert A_K\rVert^2+\lVert\Bm\rVert^2\lVert\Em\rVert^2\lVert\Vop_K\rVert$ with $A_K=\T(\bar D_K)$ bounded, since $\bar D_K\to I$, and $\lVert\Vop_K\rVert\le\lVert\Qop^K\rVert+\lVert\bar Q^K\rVert^2$; Gelfand's formula gives $\lVert\Qop^K\rVert\le C_\varepsilon(\rho(\Qop)+\varepsilon)^K$ for every $\varepsilon>0$, so $\limsup_K\rho_K\le\rho(\Qop)^{1/2}$. In the contracting case $\rho(\Lop_K)\to\rho(\T(I))^2$, and $\T(I)$ has trace $(1+\beta)(1-\alpha)$ and determinant $\beta(1-\alpha)$, so $\rho(\T(I))>0$ whenever $\alpha\ne1$; then $\log\rho(\Lop_K)$ converges to a finite limit and $\log\rho_K=\log\rho(\Lop_K)/(2K)\to0$. At $\alpha=1$ the matrix $\T(I)$ is nilpotent and the limit of $\rho_K$ is set by how fast $\Vop_K$ vanishes rather than by the mean transition. At $\beta=0$, for example, the outer step only groups the $K$ inner steps, so $\rho(\Lop_K)=\rho(\Qop)^K$ and $\rho_K=\rho(\Qop)^{1/2}$ at every $K$.

The expanding case can occur even when a single inner step contracts in mean square along every fixed direction, as the following example shows. Let $H_\xi=c\,uu^\top$ with $u$ uniform on the unit sphere of $\mathbb R^n$ and write $a=\eta c$. Using $\E[uu^\top]=I/n$ and $\E[(u^\top Yu)uu^\top]=(2Y+\operatorname{tr}(Y)I)/(n(n+2))$,
\[
\Qop(I)=\Bigl(1-\frac{2a}{n}+\frac{a^2}{n}\Bigr)I,
\qquad
\Qop(Y)=\Bigl(1-\frac{2a}{n}+\frac{2a^2}{n(n+2)}\Bigr)Y\ \ \text{for }\operatorname{tr}Y=0,
\]
so $\rho(\Qop)=1-2a/n+a^2/n$, which exceeds one for $a>2$. Along a fixed direction, $\E[(v^\top Q_\xi v)^2]=1-2a/n+3a^2/(n(n+2))$, which stays below one for $a<2(n+2)/3$. Finally, $\rho(\bar Q)=\lvert1-a/n\rvert<1$ keeps the mean inner step contracting.

\paragraph{The marginal case $\rho(\Qop)=1$.} Let $\rho(\Qop)=1$ with $1$ the only eigenvalue of $\Qop$ on the unit circle, let $\rho(\bar Q)<1$ and $0<\alpha<2(1+\beta)/(1+2\beta)$, and restrict everything to the complement of the common flat directions. Since $Q_\xi$ is symmetric, $\Qop$ is self-adjoint for the trace inner product, hence diagonalizable with real eigenvalues, so the eigenvalue $1$ is semisimple and the hypothesis excludes only $-1$. Let $\Pi$ be the orthogonal projector onto the eigenspace of $\Qop$ at $1$. Then $\Qop^K=\Pi+(\Qop-\Pi)^K\to\Pi$, because $\rho(\Qop-\Pi)<1$, while $\Qop_\infty^K\to0$ because $\rho(\bar Q)<1$; hence $\Vop_K=\Qop^K-\Qop_\infty^K\to\Pi$, a nonzero projector, with $\rho(\Pi)=1$. Meanwhile $\bar D_K\to I$, and $\T(I)$ is stable under $0<\alpha<2(1+\beta)/(1+2\beta)$, so $\T(\bar D_K)$ is stable for all large $K$ and $\Gop_{\alpha,\beta}$, which depends continuously on $\bar D_K$ near $I$ because its Neumann series converges uniformly there, tends to its value at $\bar D_K=I$. At $\bar D_K=I$ every $\Phi_t=\Em\T(I)^t\Bm$ is the scalar $\varphi_t(1)$ times the identity, so that value is multiplication by $\sum_t\varphi_t(1)^2=1/r(\alpha,\beta)$, by the identity of Appendix~\ref{app:proof-scalar} at $d_1=1$. Therefore $\Vop_K\circ\Gop_{\alpha,\beta}\to r(\alpha,\beta)^{-1}\Pi$ in norm, and the spectral radius is continuous, so $\rho(\Vop_K\circ\Gop_{\alpha,\beta})\to1/r(\alpha,\beta)$. Theorem~\ref{thm:full} then gives mean-square stability for all large $K$ when $r(\alpha,\beta)>1$ and instability for all large $K$ when $r(\alpha,\beta)<1$. With $W$ workers $\Vop_K$ is divided by $W$ (Corollary~\ref{cor:averaging}) and the limit is $1/(W\,r(\alpha,\beta))$. The scalar model gives the same limit through \eqref{eq:k-law} below, since at $s_q=1$ one has $\Nk\to1$ while $r(x(K),\beta)\to r(\alpha,\beta)$.

\paragraph{The scalar case along $K$.} In one dimension, with $\bar q=\E[Q_\xi]$ and $s_q=\E[Q_\xi^2]$, \eqref{eq:block-moments} gives
\begin{equation}
d_1(K)=1-\bar q^{\,K},
\qquad
d_2(K)=1-2\bar q^{\,K}+s_q^{\,K},
\qquad
\Nk=\frac{s_q^{\,K}-\bar q^{\,2K}}{(1-\bar q^{\,K})^2}.
\label{eq:k-law}
\end{equation}
The mean block operator saturates, so $x(K)=\alpha(1-\bar q^{\,K})$ tends to $\alpha$ when $|\bar q|<1$, while $\Nk$ tends to $0$, to $1$ or to infinity according as $s_q$ is below, equal to or above one, the two cases of Proposition~\ref{thm:blocklength} and the marginal case between them.

\subsection{Proof of Corollary~\ref{cor:averaging} and Proposition~\ref{prop:totalbatch}}
\label{app:proof-averaging}

\paragraph{Monotonicity of the spectral radius.} For two positive maps $0\preceq\mathcal M\preceq\mathcal N$ one has $0\preceq\mathcal M^t(I)\preceq\mathcal N^t(I)$ for every $t$, and since every symmetric $X$ with $\opnorm{X}\le1$ satisfies $-I\preceq X\preceq I$, the trace norm of $\mathcal M^t(X)$ is at most $2\operatorname{tr}\mathcal M^t(I)\le2\operatorname{tr}\mathcal N^t(I)$; Gelfand's formula then gives $\rho(\mathcal M)\le\rho(\mathcal N)$.

\paragraph{Proof of Corollary~\ref{cor:averaging}(a).} Let $D^{(1)},\ldots,D^{(W)}$ be independent copies of the block operator and $\widehat D=W^{-1}\sum_iD^{(i)}$ their average; the averaged pseudo-gradient is $\widehat De$ because all inner trajectories start at the same point. Then $\E[\widehat D]=\bar D_K$, and with $\Delta_i=D^{(i)}-\bar D_K$,
\[
\Vop_{K,W}(Y)=\E\bigl[(\widehat D-\bar D_K)Y(\widehat D-\bar D_K)^\top\bigr]
=\frac1{W^2}\sum_{i,j}\E[\Delta_iY\Delta_j^\top]
=\frac1W\Vop_K(Y),
\]
since the terms with $i\neq j$ vanish by independence and $\E[\Delta_i]=0$. The averaged operators are i.i.d.\ across outer steps with finite second moments, so the decomposition \eqref{eq:decomposition} and Theorem~\ref{thm:full} apply to $\widehat D$, and the decomposition gives $\Lop_{K,W}$, which is \eqref{eq:decomposition} with $\Vop_K$ replaced by $\Vop_K/W$. For $W'\ge W$ and $X\succeq0$,
\[
\Lop_{K,W}(X)-\Lop_{K,W'}(X)=\Bigl(\frac1W-\frac1{W'}\Bigr)\Bm\Vop_K(\Em X\Em^\top)\Bm^\top\succeq0,
\]
so $0\preceq\Lop_{K,W'}\preceq\Lop_{K,W}$, and by the monotonicity of the spectral radius above, $\rho(\Lop_{K,W'})\le\rho(\Lop_{K,W})$. In the scalar case $d_1$ is unchanged and $d_2$ becomes $d_1^2+(d_2-d_1^2)/W$, so $\Nk$ becomes $\Nk/W$, and the scalar form \eqref{eq:budget-general} gives the stated criterion. The exact count of Corollary~\ref{cor:averaging}(a) follows from Theorem~\ref{thm:full}, since $\Gop_{\alpha,\beta}$ depends on the distribution of the block operator only through $\bar D_K$, which averaging leaves unchanged. The criterion with $W$ workers is therefore $\rho(\Vop_K\circ\Gop_{\alpha,\beta})/W<1$.

\paragraph{Proof of Corollary~\ref{cor:averaging}(b).} On the cone, $\Lop_{K,W}\succeq X\mapsto\tfrac1W\Bm\Vop_K(\Em X\Em^\top)\Bm^\top$, and as in the proof of Proposition~\ref{thm:blocklength}(ii), $\Em\Bm=-\gamma I$ gives $\rho(\Lop_{K,W})\ge\gamma^2\rho(\Vop_K)/W\ge\gamma^2\bigl(\rho(\Qop)^K-\rho(\bar Q)^{2K}\bigr)/W$. Mean-square stability needs $\rho(\Lop_{K,W})<1$, which is the bound of Corollary~\ref{cor:averaging}(b).

\paragraph{Proof of Proposition~\ref{prop:totalbatch}.} With replacement, the one-step mean $\bar Q=I-\eta\bar H$ does not depend on the batch, so $\bar D_K$, $A$ and the mean condition are the same in both configurations. Write $\Qop_b=\Qop_\infty+\tfrac1b\Cop$ and expand its $K$-th power by distributivity, keeping the order of the factors because $\Qop_\infty$ and $\Cop$ need not commute. Collecting the terms in which $\Cop$ appears $s$ times gives
\[
\Qop_b^K=\Qop_\infty^K+\sum_{s=1}^Kb^{-s}\,\mathcal M_s,
\]
where $\mathcal M_s$ is the sum of the $\binom{K}{s}$ compositions of $K$ maps in which $\Cop$ appears $s$ times and $\Qop_\infty$ appears $K-s$ times. Each such composition is positive because its factors are, so each $\mathcal M_s$ is positive, and none depends on $b$. By the proof of Proposition~\ref{thm:blocklength}, $\Vop_K(b)=\Qop_b^K-\Qop_\infty^K=\sum_{s=1}^Kb^{-s}\mathcal M_s$. An inner trajectory with batch $b/W$ has $\Vop_K(b/W)=\sum_sW^sb^{-s}\mathcal M_s$, and averaging $W$ such trajectories divides the map by $W$ (Corollary~\ref{cor:averaging}), so
\[
\tfrac1W\,\Vop_K(b/W)-\Vop_K(b)=\sum_{s=2}^K\bigl(W^{s-1}-1\bigr)b^{-s}\mathcal M_s\succeq0 ,
\]
with equality at $K=1$, where only $s=1$ occurs. For $K\ge2$ the difference is nonzero whenever $\Cop\ne0$ and $\bar Q$ is invertible, because the composition $\Cop\circ\Cop\circ\Qop_\infty^{K-2}$ applied to $I$ has trace $\eta^4\operatorname{tr}\bigl(\bar Q^{2(K-2)}\,\E_j[(h_j-\bar H)S(h_j-\bar H)]\bigr)$ with $S=\E_i[(h_i-\bar H)^2]\ne0$, and the matrix in the expectation is positive semidefinite with trace $\operatorname{tr}(S^2)>0$. The two second-moment maps share $\mathcal A$ and have ordered noise terms, so $\Lop_K(b)\preceq\Lop_{K,W}(b/W)$, and the monotonicity of the spectral radius on the cone gives the conclusion. The isotropic zero-momentum case, with an outer relaxation in place of the Nesterov step, corresponds to a result of \citet{agarwala2026twophase} for a high-dimensional limit of linear regression, where the expansion in powers of the inverse batch also appears.

\paragraph{The second moment decreases with $W$ at every horizon.} Write $\Lop_{K,W}=\mathcal A+\tfrac1W\mathcal N$ with $\mathcal A(X)=AXA^\top$ and $\mathcal N(X)=\Bm\Vop_K(\Em X\Em^\top)\Bm^\top$, both positive. Expanding, as before, the $T$-fold composition and collecting the terms in which $\mathcal N$ appears $j$ times,
\[
\Lop_{K,W}^T=\sum_{j=0}^TW^{-j}\,\mathcal R_j ,
\]
each $\mathcal R_j$ a positive map that does not depend on $W$. For positive-semidefinite $\Psi$ and $\Sigma_0$, $\langle\Psi,\Lop_{K,W}^T(\Sigma_0)\rangle=\sum_jW^{-j}\langle\Psi,\mathcal R_j(\Sigma_0)\rangle$ is a polynomial in $1/W$ with non-negative coefficients, so it is non-increasing in $W$ and convex in $1/W$. With $\Psi=I$ it is $\E\opnorm{z_T}^2$; with $\Psi=\Em^\top\Em$ it is $\E\opnorm{e_T}^2$, the numerator of \eqref{eq:amplification}. For the ordering of Proposition~\ref{prop:totalbatch} at a finite horizon, let $0\preceq\mathcal M\preceq\mathcal N$ be positive maps; then
\[
\mathcal N^T-\mathcal M^T=\sum_{j=0}^{T-1}\mathcal N^{\,j}\,(\mathcal N-\mathcal M)\,\mathcal M^{\,T-1-j}\succeq0 ,
\]
a telescoping sum of compositions of positive maps, so $\mathcal M^T\preceq\mathcal N^T$ and every readout $\langle\Psi,\cdot\,(\Sigma_0)\rangle$ with $\Psi,\Sigma_0\succeq0$ is ordered at every horizon $T$.

\subsection{Proof of Proposition~\ref{prop:bridge} (no common minimizer)}
\label{app:proof-forced}

\paragraph{Sample losses with different minimizers.} The common minimizer of Section~\ref{sec:setup-model} enters only through the definition of the error, since the block operator does not depend on it. Let each sample loss keep its curvature but have its own minimizer, a \textit{regular minimum} in the terminology of \citet{mulayoff2024exact}. With $\bar\theta$ the minimizer of the full loss and $e_t=\theta_t-\bar\theta$, the pseudo-gradient is affine in $\theta$,
\begin{equation}
P_{K,\omega}(\theta)=p_\omega+D_{K,\omega}(\theta-\bar\theta),
\qquad
p_\omega=P_{K,\omega}(\bar\theta),
\label{eq:regular-block}
\end{equation}
with the same block operator as in \eqref{eq:block-linear}, because the response of the pseudo-gradient to a displacement depends on the curvatures alone. Hence $\bar D_K$, $\Vop_K$ and $\Lop_K$ are unchanged, and so is every statement of Sections~\ref{sec:theory} and \ref{sec:implications} about $\rho(\Lop_K)$. The offset has mean zero at $\bar\theta$ (the last paragraph of this subsection), so by Proposition~\ref{prop:bridge}(b) below, when $\rho(\Lop_K)<1$, the covariance across streams converges to $C_\infty=(\mathcal I-\Lop_K)^{-1}(\Bm S_p\Bm^\top)$ with $S_p=\E[p_\omega p_\omega^\top]$. With a common minimizer the offset vanishes and the error is that of Section~\ref{sec:setup-model}.

Section~\ref{sec:regular} uses the following statement.
\begin{proposition}[The second-moment map without a common minimizer]
\label{prop:bridge}
Let the pairs $(D_t,p_t)$ of \eqref{eq:forced-recurrence} be i.i.d.\ across outer steps and independent of $z_0$, with $\E\opnorm{D_t}_F^2$, $\E\opnorm{p_t}^2$ and $\E\opnorm{z_0}^2$ finite.
\begin{enumerate}\setlength{\itemsep}{1pt}\setlength{\parskip}{0pt}
\item[(a)] $\rho(\Lop_K)<1$ if and only if every shared-minibatch perturbation is mean-square stable (Definition~\ref{def:mss}).
\item[(b)] If $\rho(\Lop_K)<1$, then \eqref{eq:forced-recurrence} is mean-square bounded, $\sup_t\E\opnorm{z_t}^2<\infty$, and the covariance of the state across independent streams converges to $C_\infty=(\mathcal I-\Lop_K)^{-1}(F_\infty)$, with $F_\infty=\Bm\Cov(D\Em m_\infty+p)\Bm^\top$ and $m_\infty=(I-A)^{-1}\Bm\E[p]$; when $\E[p]=0$, $C_\infty=(\mathcal I-\Lop_K)^{-1}(\Bm S_p\Bm^\top)$ with $S_p=\E[p_\omega p_\omega^\top]$.
\item[(c)] Let $\Phi\succeq0$ be an eigenmatrix of the adjoint $\Lop_K^*$ at $\rho(\Lop_K)$. If $\rho(\Lop_K)\ge1$ and the forcing has weight on it, $\liminf_t\langle\Phi,F_t\rangle>0$, then $\operatorname{tr}C_t\to\infty$, at least linearly when $\rho(\Lop_K)=1$ and at least geometrically above; when the mean is stable this condition reads $\langle\Phi,F_\infty\rangle>0$.
\end{enumerate}
\end{proposition}

\paragraph{The recurrence with an offset.} Substituting $P_{\omega_t}(\theta_t)=p_t+D_te_t$ into \eqref{eq:snoo} gives
\[
u_t=\beta u_{t-1}+p_t+D_te_t,
\qquad
e_{t+1}=\bigl(I-\alpha(1+\beta)D_t\bigr)e_t-\alpha\beta^2u_{t-1}-\alpha(1+\beta)p_t,
\]
that is, $z_{t+1}=\T(D_t)z_t+\Bm p_t$, which is \eqref{eq:forced-recurrence}. Write $\T(D_t)=A+\Bm\Delta_t\Em$ with $A=\T(\bar D)$ and $\Delta_t=D_t-\bar D$, so that $z_{t+1}=Az_t+\Bm(\Delta_t\Em z_t+p_t)$. The pair $(D_t,p_t)$ is independent of $z_t$, hence $\E[\Delta_t\Em z_t]=0$ and $m_{t+1}=Am_t+\Bm\E[p]$.

For the covariance put $\zeta_t=z_t-m_t$ and $\chi_t=\Delta_t\Em z_t+(p_t-\E[p])$, so that $z_{t+1}-m_{t+1}=A\zeta_t+\Bm\chi_t$. The cross term vanishes, since $\E[\zeta_t\chi_t^\top]=\E[\zeta_tz_t^\top]\Em^\top\E[\Delta_t]^\top+\E[\zeta_t]\,\E[p_t-\E p]^\top=0$, so $C_{t+1}=AC_tA^\top+\Bm\E[\chi_t\chi_t^\top]\Bm^\top$. Using $\E[z_tz_t^\top]=C_t+m_tm_t^\top$ and the independence of $z_t$ from $(D_t,p_t)$,
\begin{align*}
\E[\chi_t\chi_t^\top]
&=\Vop\bigl(\Em C_t\Em^\top\bigr)
+\E\bigl[(\Delta_t\Em m_t+p_t-\E p)(\Delta_t\Em m_t+p_t-\E p)^\top\bigr]\\
&=\Vop\bigl(\Em C_t\Em^\top\bigr)+\Cov\bigl(D\,\Em m_t+p\bigr),
\end{align*}
and combining with $AC_tA^\top$ through \eqref{eq:decomposition} gives, with $m_t=\E[z_t]$ and $C_t=\Cov(z_t)$,
\begin{equation}
m_{t+1}=A\,m_t+\Bm\,\E[p],
\qquad
C_{t+1}=\Lop_K(C_t)+F_t,
\qquad
F_t=\Bm\,\Cov\bigl(D_t\Em m_t+p_t\bigr)\,\Bm^\top\succeq0 .
\label{eq:forced}
\end{equation}
The forcing $F_t$ is a covariance conjugated by $\Bm$, hence positive semidefinite.

\paragraph{Part (a).} A shared-minibatch perturbation obeys \eqref{eq:transition}, and Theorem~\ref{thm:full}(i) says that it is mean-square stable for every admissible start if and only if $\rho(\Lop_K)<1$.

\paragraph{Part (b).} Let $\rho(\Lop_K)<1$. Then $\rho(A)^2\le\rho(\Lop_K)<1$ (Appendix~\ref{app:proof-budget}), so $m_t\to m_\infty=(I-A)^{-1}\Bm\E[p]$ and $m_t$ is bounded. Expanding the covariance in the forcing, $\Cov(D\Em m+p)=\Vop_K(\Em mm^\top\Em^\top)+\E[\Delta\Em m\,(p-\E p)^\top]+\E[(p-\E p)\,m^\top\Em^\top\Delta^\top]+\Cov(p)$ is quadratic in $m$ with coefficients that are finite by the moment clause, so $F_t$ is bounded and $F_t\to F_\infty=\Bm\Cov(D\Em m_\infty+p)\Bm^\top$. Iterating $C_{t+1}=\Lop_K(C_t)+F_t$,
\[
C_t=\Lop_K^t(C_0)+\sum_{j=0}^{t-1}\Lop_K^{\,j}(F_{t-1-j}),
\]
and since $\lVert\Lop_K^{\,j}\rVert\le c\,\vartheta^{\,j}$ for some $\vartheta<1$, the first term vanishes, the sum is bounded by $c\sup_s\lVert F_s\rVert/(1-\vartheta)$, and it converges to $\sum_{j\ge0}\Lop_K^{\,j}(F_\infty)=(\mathcal I-\Lop_K)^{-1}(F_\infty)$, each term converging and the series being dominated. So $\Sigma_t=C_t+m_tm_t^\top$ is bounded, which is the boundedness of Proposition~\ref{prop:bridge}(b), and $C_t\to C_\infty$. In the model \eqref{eq:regular-block} with i.i.d.\ inner steps, $\E[p]=0$ (below), so $m_\infty=0$ and $F_\infty=\Bm\Cov(p)\Bm^\top=\Bm S_p\Bm^\top$.

\paragraph{Part (c).} Let $\rho=\rho(\Lop_K)\ge1$ and let $\Phi\succeq0$ be an eigenmatrix of the adjoint $\Lop_K^*$ at $\rho$, which exists because $\Lop_K^*$ is positive with the same spectral radius (Appendix~\ref{app:proof-full}). Taking the inner product of $C_{t+1}=\Lop_K(C_t)+F_t$ with $\Phi$ gives $\langle\Phi,C_{t+1}\rangle=\rho\langle\Phi,C_t\rangle+\langle\Phi,F_t\rangle$, and induction gives
\[
\langle\Phi,C_t\rangle=\rho^t\langle\Phi,C_0\rangle+\sum_{s<t}\rho^{\,t-1-s}\langle\Phi,F_s\rangle ,
\]
every term of which is non-negative because $C_0,F_s\succeq0$. If $\rho>1$ and $\langle\Phi,C_0\rangle>0$ the first term grows geometrically. If $\liminf_t\langle\Phi,F_t\rangle=2\epsilon>0$, then $\langle\Phi,F_s\rangle\ge\epsilon$ for $s\ge s_0$ and the sum is at least $\epsilon\sum_{s_0\le s<t}\rho^{t-1-s}\ge\epsilon\,(t-s_0)$, which is linear growth at $\rho=1$ and geometric growth above it. In either case $\operatorname{tr}C_t\ge\langle\Phi,C_t\rangle/\lambda_{\max}(\Phi)\to\infty$, and $\E\opnorm{z_t}^2=\operatorname{tr}\Sigma_t\ge\operatorname{tr}C_t$. When the mean is stable, $F_t\to F_\infty$ and the condition on the forcing is $\langle\Phi,F_\infty\rangle>0$, that is, the stationary forcing has weight on the unstable mode.

\paragraph{The offset and its second moment.} In the model \eqref{eq:regular-block} let an inner trajectory start at $\bar\theta$ and write $w_k=\theta_k-\bar\theta$, $w_0=0$. One inner step on the minibatch $\xi_k$ is $w_{k+1}=Q_{\xi_k}w_k-\eta\,g_{\xi_k}$ with $g_\xi=\nabla\ell_\xi(\bar\theta)=H_\xi(\bar\theta-\theta^*_\xi)$, whose mean $\E_\xi[g_\xi]=\nabla\E_\xi\ell_\xi(\bar\theta)$ vanishes because $\bar\theta$ minimizes the full loss. With i.i.d.\ inner steps, $\xi_k$ is independent of $w_k$, so $\E[w_{k+1}]=\bar Q\,\E[w_k]$ and $\E[w_k]=0$ for all $k$; hence $\E[p_\omega]=-\E[w_K]=0$. For the second moment $W_k=\E[w_kw_k^\top]$ the cross terms $\E[Q_{\xi_k}w_kg_{\xi_k}^\top]$ vanish for the same reason, and
\[
W_{k+1}=\Qop_b(W_k)+\eta^2\Gamma,
\qquad \Gamma=\E_\xi[g_\xi g_\xi^\top],
\qquad\text{so}\qquad
S_p=W_K=\sum_{j=0}^{K-1}\Qop_b^{\,j}\bigl(\eta^2\Gamma\bigr),
\]
where $\Gamma$ is the second moment of the minibatch gradient at $\bar\theta$, equal to $\tfrac1b\E_i[g_ig_i^\top]$ for $b$ samples drawn with replacement. The formula for $S_p$ is the analogue of \eqref{eq:block-moments} for the offset, and together with $\bar D_K$ and $\Sm_K$ it determines the floor of part (b) from the one-step law.

\paragraph{Common flat directions without a common minimizer.} In the quadratic model the forcing has no component on directions that every block leaves unchanged, and the analysis is carried out on their complement, as in Section~\ref{sec:setup-stability}. The affine construction of Example~\ref{ex:counterexample} is an instance of Proposition~\ref{prop:bridge}(b) in one dimension.

\section{The Exact Model Used in the Figures}
\label{app:toy}

The figures of Sections~\ref{sec:intro} and~\ref{sec:implications} and of Appendices~\ref{app:mean-vs-ms} and~\ref{app:toy} evaluate the criterion of Theorem~\ref{thm:full} on a model in which every expectation is a finite sum. The sample losses are $\tfrac12(a_i^\top(\theta-\theta^*))^2$ over a fixed set of $N$ feature vectors, so they share a minimizer and each sample curvature is $a_ia_i^\top$, which is the model of Section~\ref{sec:setup-model}. Inner SGD draws minibatches of size $b$ with replacement, so that
\[
\E[H_\xi\kron H_\xi]=\frac1b\,\E_i[h_i\kron h_i]+\Bigl(1-\frac1b\Bigr)\bar H\kron\bar H ,
\]
and \eqref{eq:block-moments} turns this into $\bar D_K$ and $\Sm_K$. The second-moment map is then formed and its spectral radius computed by dense eigenvalues, so nothing in the figures is estimated. We use Gaussian features with a power-law covariance, for which $\bar D_K$ is strongly anisotropic, in Figure~\ref{fig:exact-beta} and Table~\ref{tab:toy-w}, and features of equal norm, for which each sample curvature is a rank-one projector, in Figure~\ref{fig:toy-blocklength} at inner rates $\eta=1.5$ and $3$ and in Figures~\ref{fig:phase}, \ref{fig:implications} and~\ref{fig:mean-vs-ms} at $\eta=2.5$ ($\eta=2$ in the right panel of Figure~\ref{fig:phase}). All three use the eight-dimensional rank-one model with $1{,}024$ samples. In Figure~\ref{fig:phase}, the left panel has $b=2$ and $K\in\{1,8,16,32\}$, the middle panel $b=1$ and $K\in\{1,4,8,16\}$, both with one worker, and the right panel $b=1$, $K=16$ and $W\in\{1,2,4,8\}$. In Figure~\ref{fig:implications}, panel (i) has $(\alpha,K,b)=(1.4,8,2)$; panel (ii) has $b\in\{1,2\}$; panel (iii) has $(\alpha,\beta,K,b)=(1,0.5,8,1)$; and panel (iv) has $(\alpha,\beta,K)=(1,0.5,8)$ with $(W,b)\in\{(1,4),(2,2),(4,1)\}$.

\begin{figure}[t]
\centering
\includegraphics{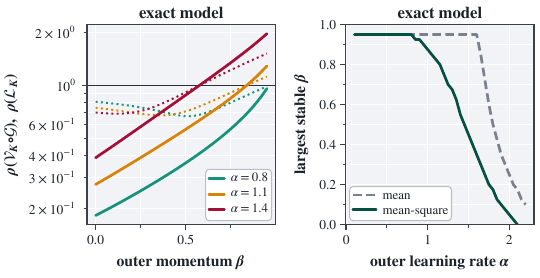}
\caption{\textbf{Proposition~\ref{prop:momentum} on the exact model.} Realizable linear regression in $12$ dimensions with a power-law feature covariance, $\eta=0.2$, $b=1$, $K=8$. \uline{Left}: the loop gain $\rho(\Vop_K\circ\Gop_{\alpha,\beta})$ (solid) rises with $\beta$ at every outer rate, but $\rho(\Lop_K)$ (dotted) does not, so Proposition~\ref{prop:momentum} is stated for the loop gain. At each rate the stable momenta run from zero to the value where the loop gain crosses one, $\beta=0.85$ for $\alpha=1.1$ and $\beta=0.575$ for $\alpha=1.4$. \uline{Right}: the largest stable momentum against the outer rate, for mean and for mean-square stability.}
\label{fig:exact-beta}
\end{figure}

\begin{figure}[t]
\centering
\includegraphics[width=\linewidth]{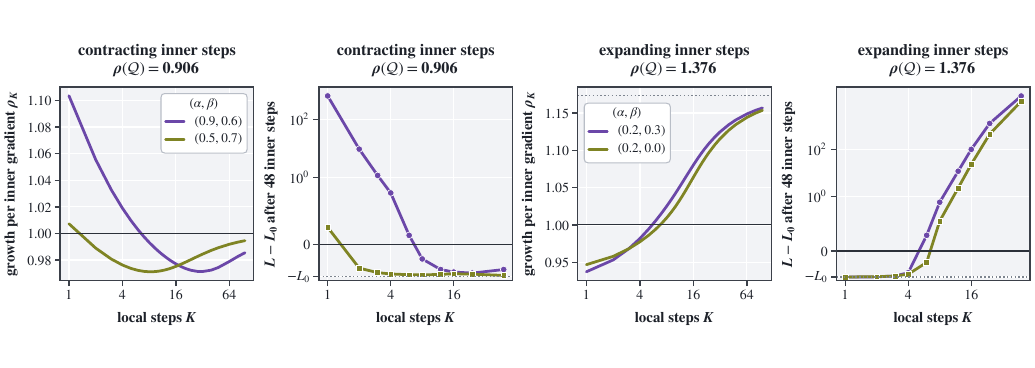}   %
\caption{\textbf{Long blocks are eventually stable or unstable according to the mean-square stability of the inner steps.} Two eight-dimensional models with rank-one sample curvatures and $b=1$, each at two outer settings whose stability changes with $K$. Each pair of panels shows the growth per inner gradient $\rho_K$ of \eqref{eq:rho-k}, and beside it the change in expected population loss over a budget of $48$ inner steps at the block lengths dividing it, from the loss $L_0=1/16$, with the largest possible fall dotted. \uline{Left pair}: $\rho(\Qop)=0.906$; longer blocks make both settings stable, and $\rho_K$ returns towards one past its best value, since the contraction per block saturates while each block spends more gradients. \uline{Right pair}: $\rho(\Qop)=1.376$; longer blocks make both settings unstable, and $\rho_K$ rises towards $\rho(\Qop)^{1/2}=1.173$ (dotted). The mean dynamics stay stable throughout, so these changes come from the block noise, and on a grid of outer settings no cell turns the other way.}
\label{fig:toy-blocklength}
\end{figure}

\paragraph{Limits of the exact model.} Inner steps drawn i.i.d.\ within a block make $\bar D_K$ symmetric, so the block means in the figures have real spectra. The figures do not include blocks whose draws are correlated, for which $\bar D_K$ can have complex eigenvalues. The proof of Proposition~\ref{prop:momentum} covers these blocks, because it does not need a real spectrum.

\paragraph{Checking the worker count $W^\star$.} Table~\ref{tab:toy-w} evaluates the count $W^\star$ of Corollary~\ref{cor:averaging}(a) on the model of Figure~\ref{fig:exact-beta}, against the second-moment map.

\begin{table}[h]
\centering
\small
\caption{\textbf{The number of workers predicted from the distribution of a single inner trajectory, on the model of Figure~\ref{fig:exact-beta}.} $W^\star$ is the count of Corollary~\ref{cor:averaging}(a), computed from the loop gain of a single worker, and the last two columns give $\rho(\Lop_{K,W})$ at $W^\star-1$ and $W^\star$ workers.}
\label{tab:toy-w}
\setlength{\tabcolsep}{5pt}
\begin{tabular}{@{}cccccc@{}}
\toprule
& & & & \multicolumn{2}{c}{$\rho(\Lop_{K,W})$} \\
\cmidrule(lr){5-6}
$\alpha$ & $\beta$ & $\rho(\Vop_K\circ\Gop_{\alpha,\beta})$ & $W^\star$ & $W^\star-1$ & $W^\star$ \\
\midrule
$1.3$ & $0.9$ & $1.486$ & $2$ & $1.274$ & $0.877$ \\
$1.5$ & $0.9$ & $2.386$ & $3$ & $1.090$ & $0.908$ \\
$1.6$ & $0.8$ & $3.391$ & $4$ & $1.036$ & $0.959$ \\
\bottomrule
\end{tabular}
\end{table}

\section{The Measures at a Checkpoint}
\label{app:measures}

\paragraph{Frozen predictions and actual runs.}\label{sec:frozen} The frozen linearization only approximates an actual run, because it uses the block Jacobian at the checkpoint $\theta_0$ for every outer step, whereas an actual run uses the block Jacobian at the point $\theta_t$ it has reached. We therefore compare the frozen linearization, to which Theorem~\ref{thm:full} applies exactly, with measurements along actual runs, the dotted and solid lines of Figure~\ref{fig:ckpt-amp}. The frozen linearization and the actual runs agree as long as the block Jacobians change little along the run, at least over short horizons. Neglecting the remainder of \eqref{eq:taylor} causes no error in the measured ratio, because the ratio is computed from a tangent along the run, an infinitesimal perturbation for which the remainder vanishes.

\paragraph{Which results hold at a network.} The results of Sections~\ref{sec:theory} and~\ref{sec:implications} fall into two groups. The moment recurrences \eqref{eq:moment-recurrences}, the decomposition \eqref{eq:decomposition}, Theorem~\ref{thm:full}, Proposition~\ref{prop:momentum} and Corollary~\ref{cor:averaging}(a) need only that the block operators of different outer steps are i.i.d. Under the frozen linearization the block Jacobians $D_t=J_{\omega_t}(\theta_0)$ are i.i.d., so these results hold exactly for shared-minibatch perturbations at a network, with $\Vop_K$ the covariance of $D_\omega$. Proposition~\ref{thm:blocklength}, Proposition~\ref{prop:totalbatch} and Corollary~\ref{cor:averaging}(b) also use the block moments \eqref{eq:block-moments}, which require the block operator to be a product of $K$ i.i.d.\ one-step maps. This holds when the Hessians are constant, as in the quadratic model, but not for the block Jacobian of a network, whose one-step factors take the Hessian at the point the inner trajectory has reached and so depend on the earlier minibatches. At a network we therefore test these three results rather than derive them.

\paragraph{Amplification of the full state.} Theorem~\ref{thm:full} concerns the state $z_t=(e_t,u_{t-1})$, and the finite-horizon ratio on it is
\begin{equation}
A(T)=\frac{\E\opnorm{z_T}^2}{\opnorm{z_0}^2}=\frac{\operatorname{tr}\Lop_K^T(z_0z_0^\top)}{\opnorm{z_0}^2},
\label{eq:amplification-full}
\end{equation}
while the parameter ratio \eqref{eq:amplification} is $\operatorname{tr}\bigl(\Em\Lop_K^T(z_0z_0^\top)\Em^\top\bigr)/\opnorm{e_0}^2$ with $z_0=(e_0,0)$; at a network both identities hold only under the frozen linearization. If $\rho(\Lop_K)<1$, both ratios tend to zero for large $T$ from every start, by Theorem~\ref{thm:full}(i). If $\rho(\Lop_K)\ge1$, they fail to decay whenever the start excites a mode that does not decay, as an isotropic start does by the empty-buffer lemma of Appendix~\ref{app:proof-full}. With momentum $A$ is not normal, so the ratio can rise for a while before it decays. We therefore always report the horizon $T$. We use the term amplification for growth over $T$ outer steps, which differs from the noise amplification of momentum methods, the steady-state variance of the iterates \citep{mohammadi2021robustness}.

\paragraph{Finite-horizon interpretation.} The stability criterion of Theorem~\ref{thm:full} is asymptotic, whereas $A_\theta(T)$ is measured over a finite continuation. Finite-horizon separation of coupled or perturbed trajectories is a common way to study training dynamics \citep{li2025coupling,altintas2025butterfly}, and we use $A_\theta(T)$ in the same way, as a finite-horizon measure of the mean-square dynamics. The empirical conclusions are largely unchanged when the horizon is varied from $T=10$ to $T=20$. Reading the amplification at $T=10$ instead of $T=20$ changes whether a setting amplifies or contracts at only $16$ of the $1{,}404$ settings we measured, and at $T=15$ at only $5$. More importantly for the implications, the predicted orderings are already present at $T=10$: $241$ of the $244$ orderings that the four implications predict hold at both $T=10$ and $T=20$.

\paragraph{Choice of the initial perturbation.} We draw the perturbation in the range of a block operator, $e_0\propto J_{\omega'}(\theta_0)\,g$ for a block $\omega'$ disjoint from the stream and an isotropic $g$, with an empty buffer.

\paragraph{Spread cost to second order.} To second order the spread cost in \eqref{eq:loss-split} is $\tfrac12\operatorname{tr}(H_t\hat C_t)$, where $\hat C_t$ is the covariance of the runs about their average, normalized by $S$, and $H_t$ is the full-data Hessian at $\hat\mu_t$. In expectation this is proportional to $\tfrac12\operatorname{tr}(H_t\Cov(\theta_t))$, and $\Cov(\theta_t)$ is the covariance that Proposition~\ref{prop:bridge} propagates (derivation in the next paragraph). Each direction of the spread therefore costs loss in proportion to the curvature along it. The spread cost is the loss analogue of the interpolation gap of \citet{wortsman2022soups}, and $L(\hat\mu_t)$ is meaningful because the runs share a start \citep{frankle2020linear}.

\paragraph{Derivation of the second-order spread cost.} For $S$ independent runs $\theta_1,\ldots,\theta_S$ from a common start, expand the full-data loss to second order about $\mu=\E[\theta]$, with $H_\mu$ the full-data Hessian there and $\hat\mu_S$ the average of the runs. The first-order terms cancel and
\[
\tfrac1S\sum_sL(\theta_s)-L(\hat\mu_S)\approx\tfrac12\Bigl[\tfrac1S\sum_s(\theta_s-\mu)^\top H_\mu(\theta_s-\mu)-(\hat\mu_S-\mu)^\top H_\mu(\hat\mu_S-\mu)\Bigr],
\]
whose expectation is $\tfrac12\bigl(1-\tfrac1S\bigr)\operatorname{tr}(H_\mu\Cov(\theta))$ because $\Cov(\hat\mu_S)=\Cov(\theta)/S$ for independent runs. Expanding about $\hat\mu_S$ instead makes the first-order term vanish exactly and gives $\tfrac12\operatorname{tr}(H_t\hat C_t)$ with $H_t$ the Hessian at $\hat\mu_S$, the form stated above; the two Hessians differ at higher order.

\section{Setup of the Network Experiments}
\label{app:net-setup}

Section~\ref{sec:checkpoint-measure} and Appendices~\ref{app:implications-all} and~\ref{app:continuations} measure an MLP, a ResNet-20 and Pythia language models at several checkpoints of their training. This appendix describes the networks and the checkpoints, and the table below summarizes them.

\begin{center}
\small
\setlength{\tabcolsep}{4.5pt}
\begin{tabular}{@{}llll@{}}
\toprule
network & training & checkpoints (step) & inner steps \\
\midrule
CIFAR MLP & SGD, rate $0.01$, batch $4$, $150{,}000$ steps & $60$k, $90$k, $120$k, $150$k & SGD at rate $0.01$ \\
ResNet-20 & SGD, rate $0.005$, batch $4$, $150{,}000$ steps & $90$k, $120$k, $150$k & SGD at rate $0.005$ \\
Pythia-70m & Adam, $143{,}000$ steps & $50$k, $100$k, $143$k & \begin{tabular}[t]{@{}l@{}}SGD at rate $10^{-3}$,\\ $3\times10^{-5}$, $6\times10^{-6}$\end{tabular} \\
\bottomrule
\end{tabular}
\end{center}

\paragraph{The networks.} The CIFAR MLP is a fully connected network of widths $3{,}072$, $512$, $512$ and $10$. The ResNet is the CIFAR ResNet-20 of \citet{he2016deep}, with $269{,}722$ parameters, in which every batch normalization is replaced by a group normalization of eight groups \citep{wu2018group}, so that the output for an image does not depend on the rest of its minibatch. Both are trained by plain SGD with batch $4$ and the squared loss on the same $8{,}192$ CIFAR-10 images \citep{krizhevsky2009cifar}. Pythia-70m is the language model of \citet{biderman2023pythia} at its public checkpoints. It was trained with Adam, and our inner steps are plain SGD. Its transformer blocks and final layer norm move, $18.9$ of its $70.4$ million parameters, while the input embedding and the output layer stay fixed. A minibatch is $b$ sequences of $2{,}048$ tokens drawn from one million sequences of its training data, the Pile \citep{gao2020pile}, and the loss is the mean next-token cross-entropy on $128$ held-out sequences. All experiments ran on NVIDIA H100 GPUs. The Pythia models are released under the Apache 2.0 license, CIFAR-10 asks only to be cited, and each component of the Pile keeps its own license.

\paragraph{Settings.} Unless a panel states otherwise, $K=14$, $b=4$ and $W=1$. Section~\ref{sec:checkpoint-measure} shows the MLP at step $120{,}000$, and Figures~\ref{fig:pair-mlp60k} to~\ref{fig:pair-pythia143k} show the other checkpoints in the same panels. On Pythia the momentum panel uses $b=8$, and the workers and split panels use $(\alpha,\beta)=(0.5,0.6)$.

\paragraph{Runs that diverge.} When some runs diverge, we average over the streams that survive and mark these points with an arrow. In Figure~\ref{fig:ckpt-amp} this happens within the $20$ outer steps for $b=1$ at $K=28$ and $56$ in panel (ii), where $41$ and $37$ of the $48$ runs survive, and for $(W,b)=(8,1)$ in panel (iv), where $47$ runs survive at $K=28$ and $39$ at $K=56$. For the frozen linearization we redrew the inner trajectories that diverged, at $b=1$ with $K=56$ and at $(8,1)$ with $K\ge28$, so these points average only over the inner trajectories that stayed bounded. In Figure~\ref{fig:ckpt-spread-impl}, runs diverged only at settings where they also left the basin, whose spread cost the figure does not plot, and at $(W,b)=(8,1)$ with $K=56$ in panel (iv), where $4$ of the $32$ runs diverged and we report the spread cost of the other $28$.

\section{Experiments across Networks and Checkpoints}
\label{app:implications-all}

Figures~\ref{fig:pair-mlp60k} to~\ref{fig:pair-pythia143k} show the other checkpoints as Figures~\ref{fig:ckpt-amp} and~\ref{fig:ckpt-spread-impl} show the MLP at step $120{,}000$.

\subsection{The CIFAR MLP}\label{app:all-mlp}

\begin{figure}[!htbp]
\centering
\includegraphics[scale=0.96]{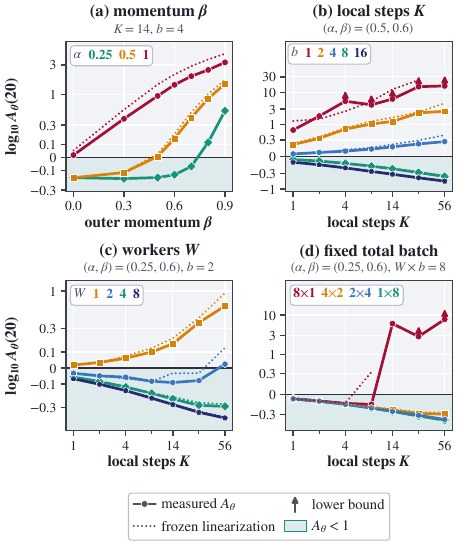}\hfill
\includegraphics[scale=0.96]{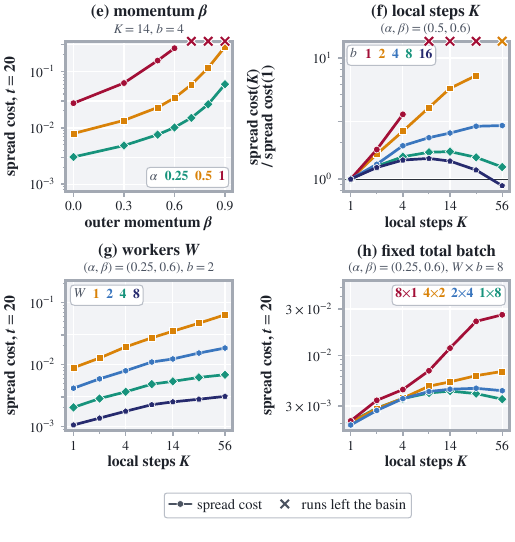}
\caption{\textbf{The CIFAR MLP at step $60{,}000$.} \uline{Left}: the amplification, as Figure~\ref{fig:ckpt-amp}. \uline{Right}: the spread cost, as Figure~\ref{fig:ckpt-spread-impl}.}
\label{fig:pair-mlp60k}
\end{figure}

\begin{figure}[tbp]
\centering
\includegraphics[scale=0.96]{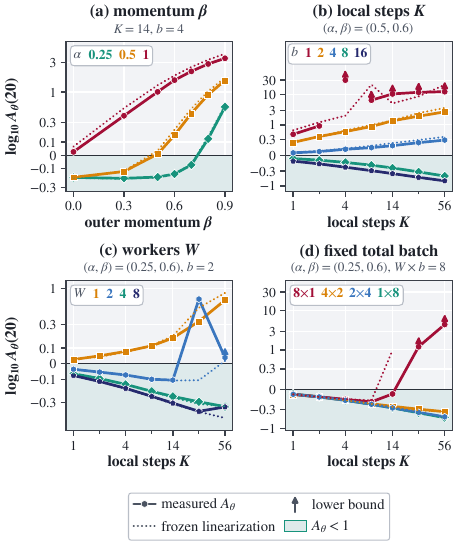}\hfill
\includegraphics[scale=0.96]{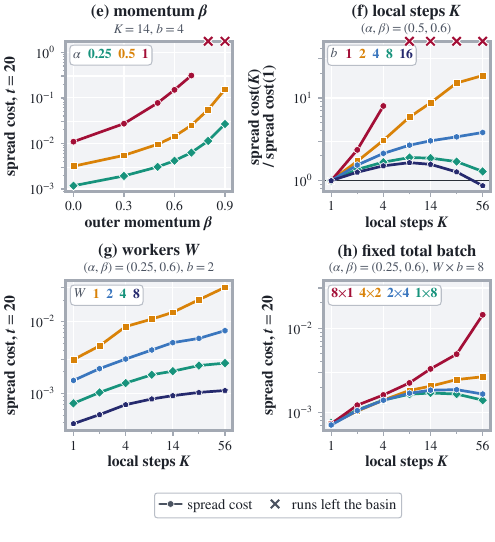}
\caption{\textbf{The CIFAR MLP at step $90{,}000$.} \uline{Left}: the amplification, as Figure~\ref{fig:ckpt-amp}. \uline{Right}: the spread cost, as Figure~\ref{fig:ckpt-spread-impl}.}
\label{fig:pair-mlp90k}
\end{figure}

\begin{figure}[tbp]
\centering
\includegraphics[scale=0.96]{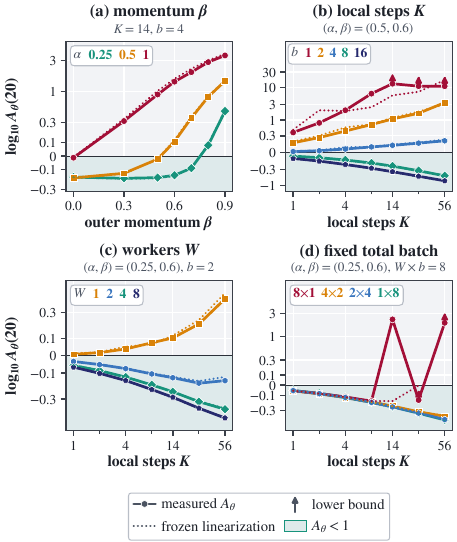}\hfill
\includegraphics[scale=0.96]{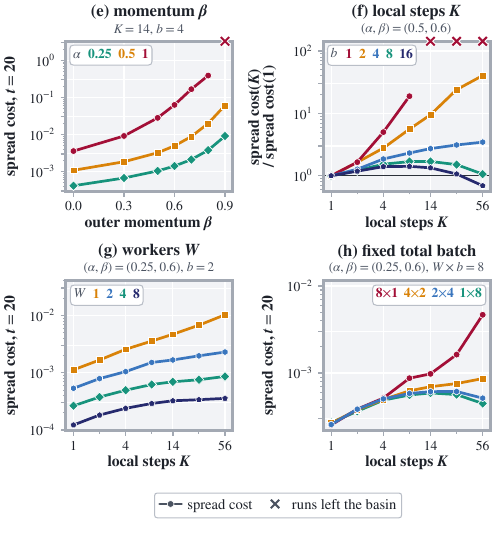}
\caption{\textbf{The CIFAR MLP at step $150{,}000$.} \uline{Left}: the amplification, as Figure~\ref{fig:ckpt-amp}. \uline{Right}: the spread cost, as Figure~\ref{fig:ckpt-spread-impl}.}
\label{fig:pair-mlp150k}
\end{figure}

\FloatBarrier

\subsection{ResNet-20}\label{app:all-resnet}

\begin{figure}[!htbp]
\centering
\includegraphics[scale=0.96]{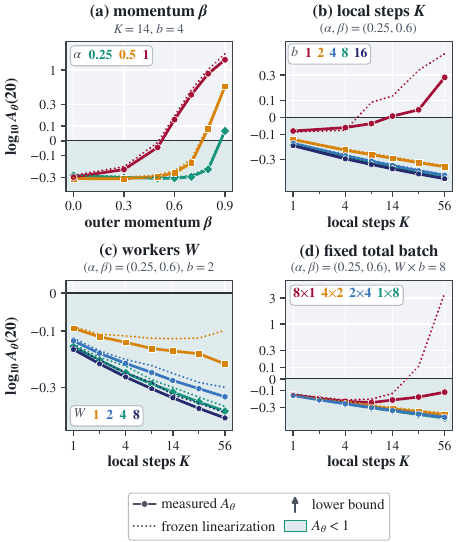}\hfill
\includegraphics[scale=0.96]{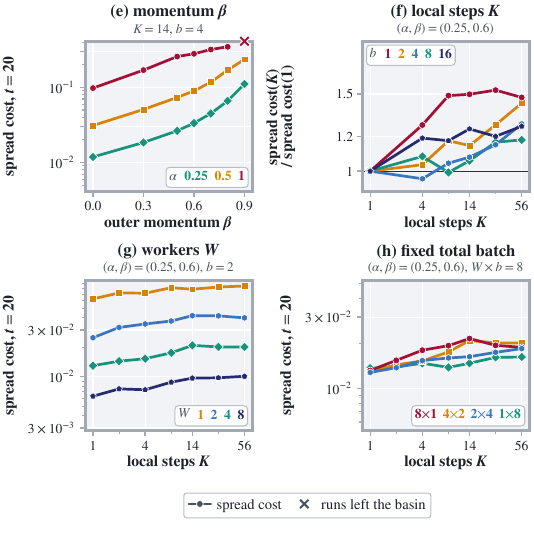}
\caption{\textbf{ResNet-20 at step $90{,}000$.} \uline{Left}: the amplification, as Figure~\ref{fig:ckpt-amp}. \uline{Right}: the spread cost, as Figure~\ref{fig:ckpt-spread-impl}.}
\label{fig:pair-resnet90k}
\end{figure}

\begin{figure}[tbp]
\centering
\includegraphics[scale=0.96]{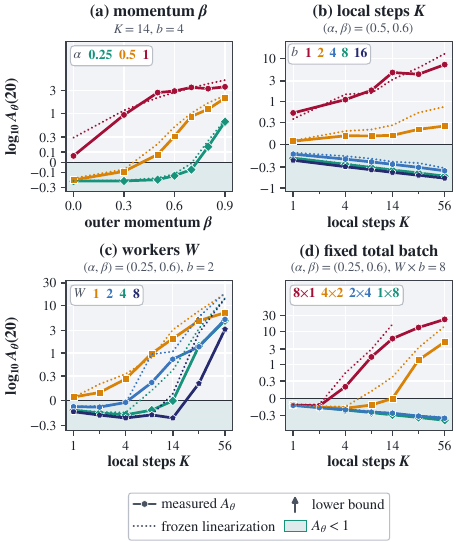}\hfill
\includegraphics[scale=0.96]{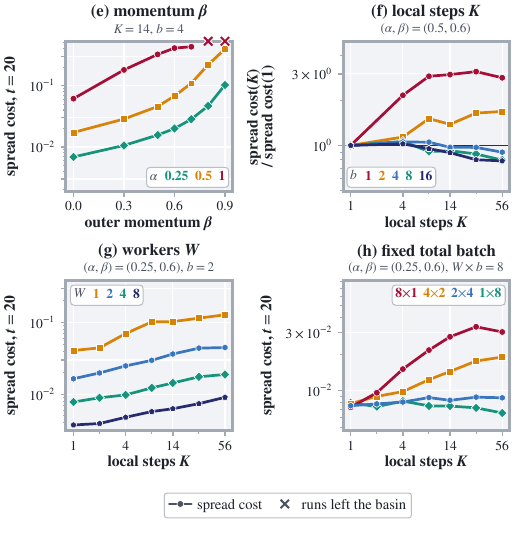}
\caption{\textbf{ResNet-20 at step $120{,}000$.} \uline{Left}: the amplification, as Figure~\ref{fig:ckpt-amp}. \uline{Right}: the spread cost, as Figure~\ref{fig:ckpt-spread-impl}.}
\label{fig:pair-resnet120k}
\end{figure}

\begin{figure}[tbp]
\centering
\includegraphics[scale=0.96]{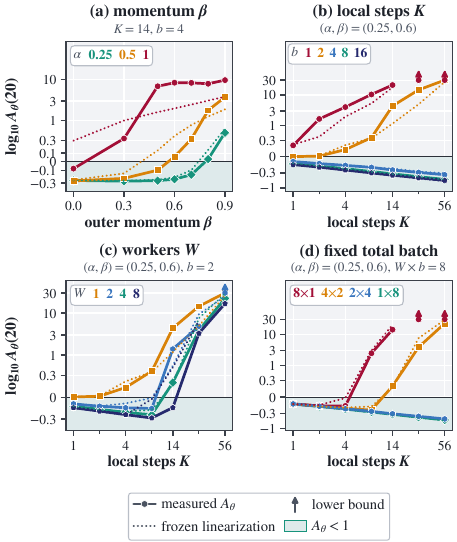}\hfill
\includegraphics[scale=0.96]{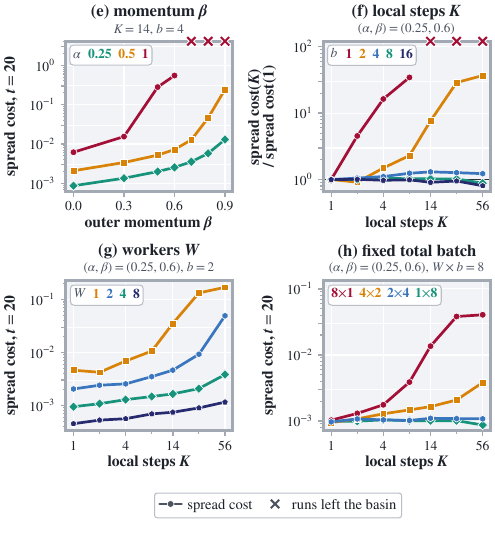}
\caption{\textbf{ResNet-20 at step $150{,}000$.} \uline{Left}: the amplification, as Figure~\ref{fig:ckpt-amp}. \uline{Right}: the spread cost, as Figure~\ref{fig:ckpt-spread-impl}.}
\label{fig:pair-resnet150k}
\end{figure}

\FloatBarrier

\subsection{Pythia-70m}\label{app:all-pythia}
On Pythia-70m, unlike on the other two networks, the spread cost does not separate the inner batches (panels (f) and (h) of Figures~\ref{fig:pair-pythia50k} to~\ref{fig:pair-pythia143k}). Two runs that share their minibatches do not differ in the noise of their draws, and their loss gap, the spread cost \eqref{eq:loss-split} of such a pair, isolates the effect of the amplification. The loss gap follows the amplification (Figures~\ref{fig:gap-pythia50k} to~\ref{fig:gap-pythia143k}). It grows with $K$ at $b=1$ and shrinks at $b=4$, $8$ and $16$, and it grows for the split $8\times1$ while it shrinks for $1\times8$, as the amplification does in panels (b) and (d).

\begin{figure}[!htbp]
\centering
\includegraphics[scale=0.96]{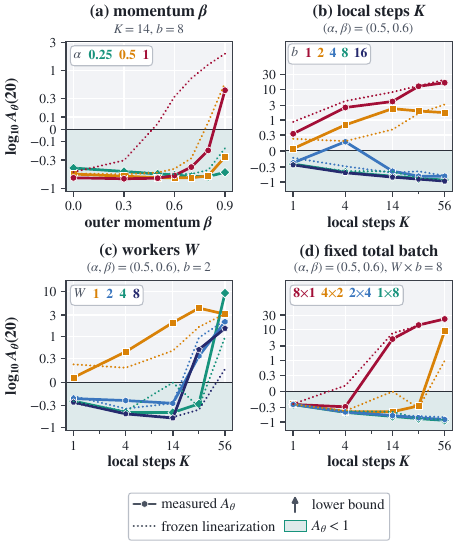}\hfill
\includegraphics[scale=0.96]{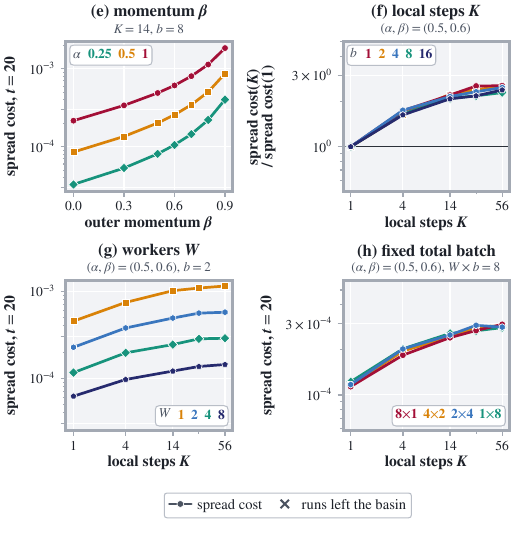}
\caption{\textbf{Pythia-70m at step $50{,}000$.} \uline{Left}: the amplification, as Figure~\ref{fig:ckpt-amp}. \uline{Right}: the spread cost, as Figure~\ref{fig:ckpt-spread-impl}.}
\label{fig:pair-pythia50k}
\end{figure}

\begin{figure}[tbp]
\centering
\includegraphics[scale=0.96]{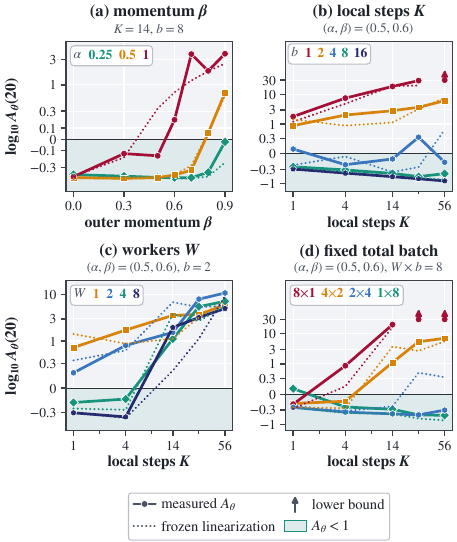}\hfill
\includegraphics[scale=0.96]{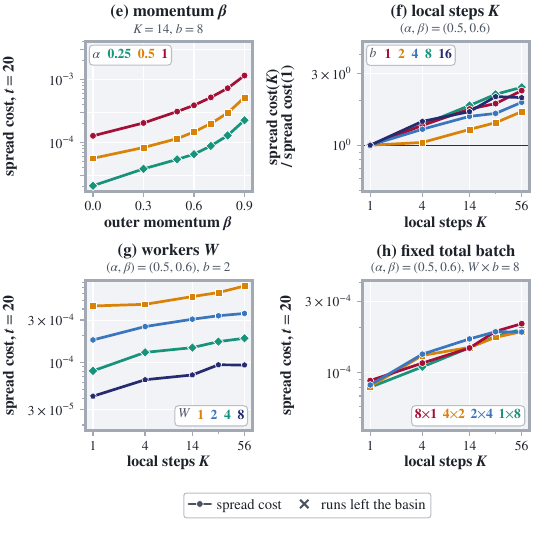}
\caption{\textbf{Pythia-70m at step $100{,}000$.} \uline{Left}: the amplification, as Figure~\ref{fig:ckpt-amp}. \uline{Right}: the spread cost, as Figure~\ref{fig:ckpt-spread-impl}.}
\label{fig:pair-pythia100k}
\end{figure}

\begin{figure}[tbp]
\centering
\includegraphics[scale=0.96]{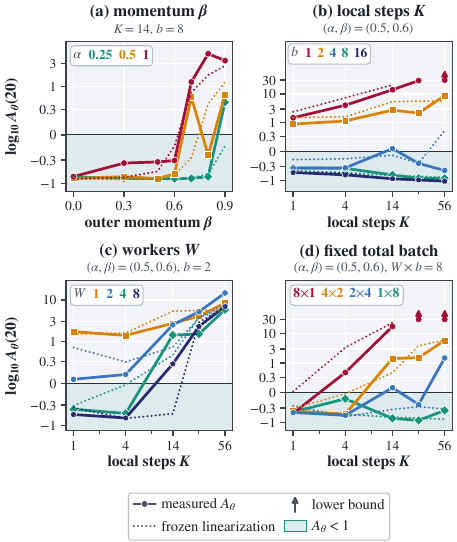}\hfill
\includegraphics[scale=0.96]{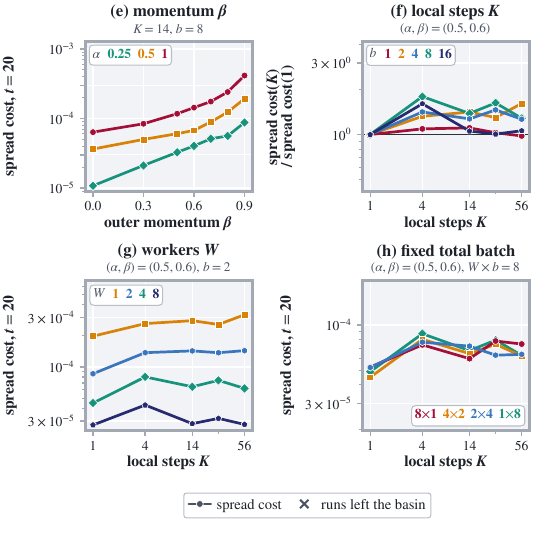}
\caption{\textbf{Pythia-70m at step $143{,}000$.} \uline{Left}: the amplification, as Figure~\ref{fig:ckpt-amp}. \uline{Right}: the spread cost, as Figure~\ref{fig:ckpt-spread-impl}.}
\label{fig:pair-pythia143k}
\end{figure}

\begin{figure}[tbp]
\centering
\includegraphics[width=0.49\linewidth]{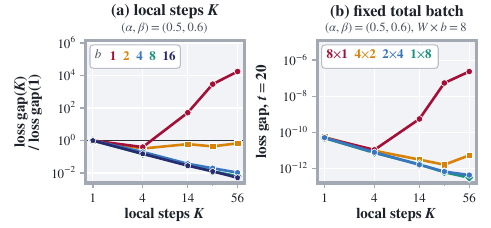}
\caption{\textbf{The loss gap of pairs of runs that share their minibatches, Pythia-70m at step $50{,}000$.} Two runs start $10^{-5}$ apart and see the same minibatches, and the loss gap is their mean loss minus the loss at their midpoint after $20$ outer steps (median over pairs).}
\label{fig:gap-pythia50k}
\end{figure}

\begin{figure}[tbp]
\centering
\includegraphics[width=0.49\linewidth]{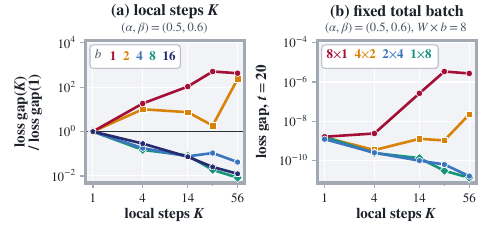}
\caption{\textbf{The loss gap of pairs of runs that share their minibatches, Pythia-70m at step $100{,}000$.} As Figure~\ref{fig:gap-pythia50k}.}
\label{fig:gap-pythia100k}
\end{figure}

\begin{figure}[tbp]
\centering
\includegraphics[width=0.49\linewidth]{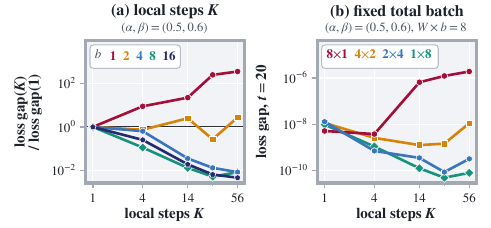}
\caption{\textbf{The loss gap of pairs of runs that share their minibatches, Pythia-70m at step $143{,}000$.} As Figure~\ref{fig:gap-pythia50k}.}
\label{fig:gap-pythia143k}
\end{figure}

\FloatBarrier
\subsection{Larger language models}\label{app:all-larger}
Figure~\ref{fig:pythia-larger} repeats the amplification panels on Pythia-160m, Pythia-410m and Pythia-1b at step $100{,}000$, with inner SGD at rates $1.8\times10^{-5}$, $2.4\times10^{-4}$ and $1.2\times10^{-2}$. As on the smaller networks, small inner batches amplify more as $K$ grows and large ones less (panels (b), (f) and (j)), and the most split total batch amplifies most at long blocks (panels (d), (h) and (l)).

\begin{figure}[!htbp]
\centering
\includegraphics[scale=0.96]{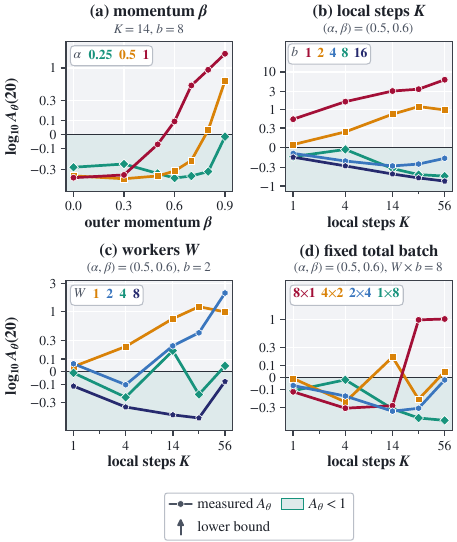}\hspace{2em}
\includegraphics[scale=0.96]{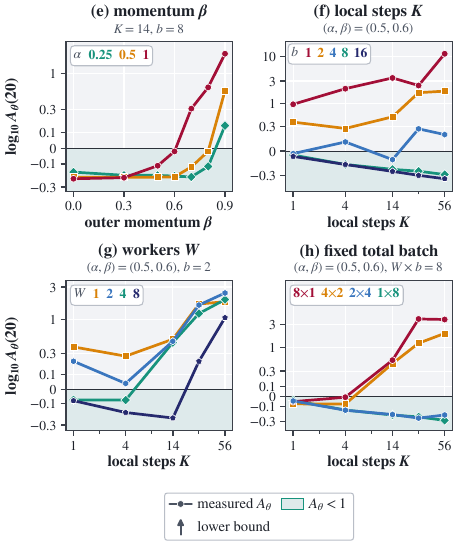}\\[6pt]
\includegraphics[scale=0.96]{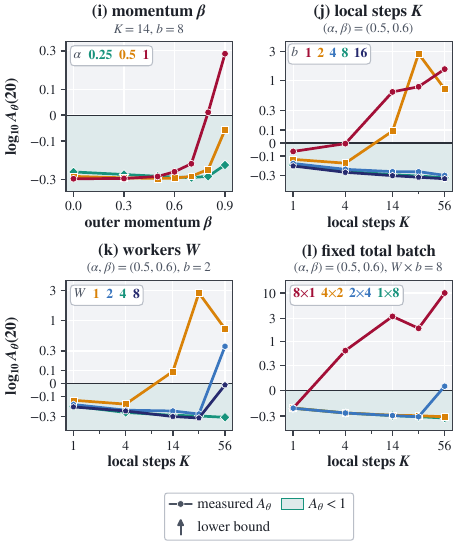}
\caption{\textbf{Pythia-160m, Pythia-410m and Pythia-1b at step $100{,}000$.} \uline{Top left}: Pythia-160m. \uline{Top right}: Pythia-410m. \uline{Bottom}: Pythia-1b. The amplification, as Figure~\ref{fig:ckpt-amp}.}
\label{fig:pythia-larger}
\end{figure}

\FloatBarrier
\section{Phase Diagrams at a Checkpoint}\label{app:phase-diagrams}

Figure~\ref{fig:phase} draws the mean-square stable region of the exact model in the $(\alpha,\beta)$ plane. Figures~\ref{fig:phase-mlp} and~\ref{fig:phase-batch} measure the same plane at the CIFAR MLP checkpoint, as the amplification $\log_{10}A_\theta(20)$ at each setting, and Figure~\ref{fig:phase-mlp} also shows the spread cost \eqref{eq:loss-split} after the same $20$ outer steps.

\begin{figure}[!htbp]
\centering
\includegraphics[width=\linewidth]{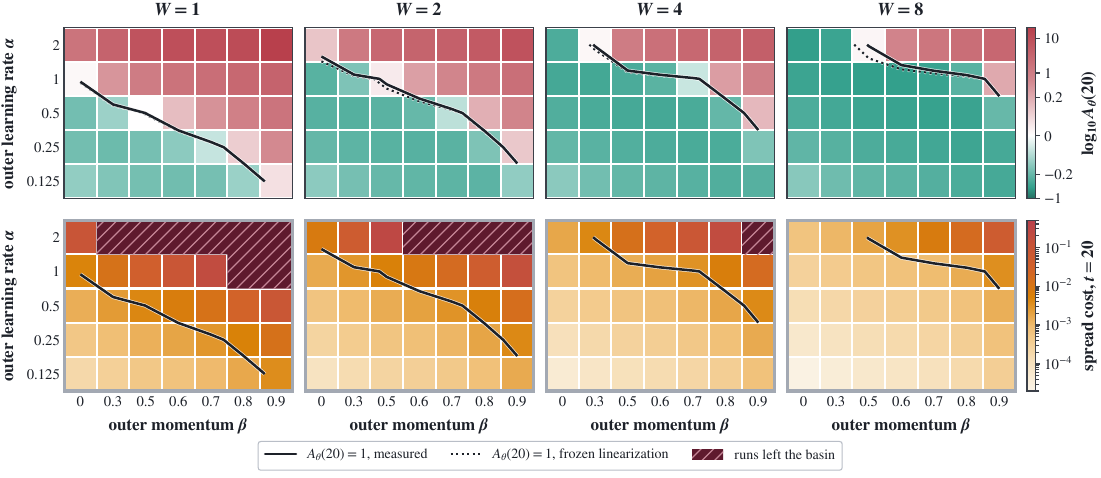}   %
\caption{\textbf{The $(\alpha,\beta)$ plane at the CIFAR MLP checkpoint.} $K=14$, $b=4$, step $120{,}000$, with $W=1,2,4,8$ workers. \uline{Top}: $\log_{10}A_\theta(20)$. The solid line is the measured boundary $A_\theta(20)=1$, below which a setting contracts, and the dotted line is that of the frozen linearization. \uline{Bottom}: the spread cost \eqref{eq:loss-split} after $20$ outer steps, with the boundary of the top row drawn over it. Hatching marks the settings whose runs left the checkpoint's basin.}
\label{fig:phase-mlp}
\end{figure}

\begin{figure}[tbp]
\centering
\includegraphics[width=\linewidth]{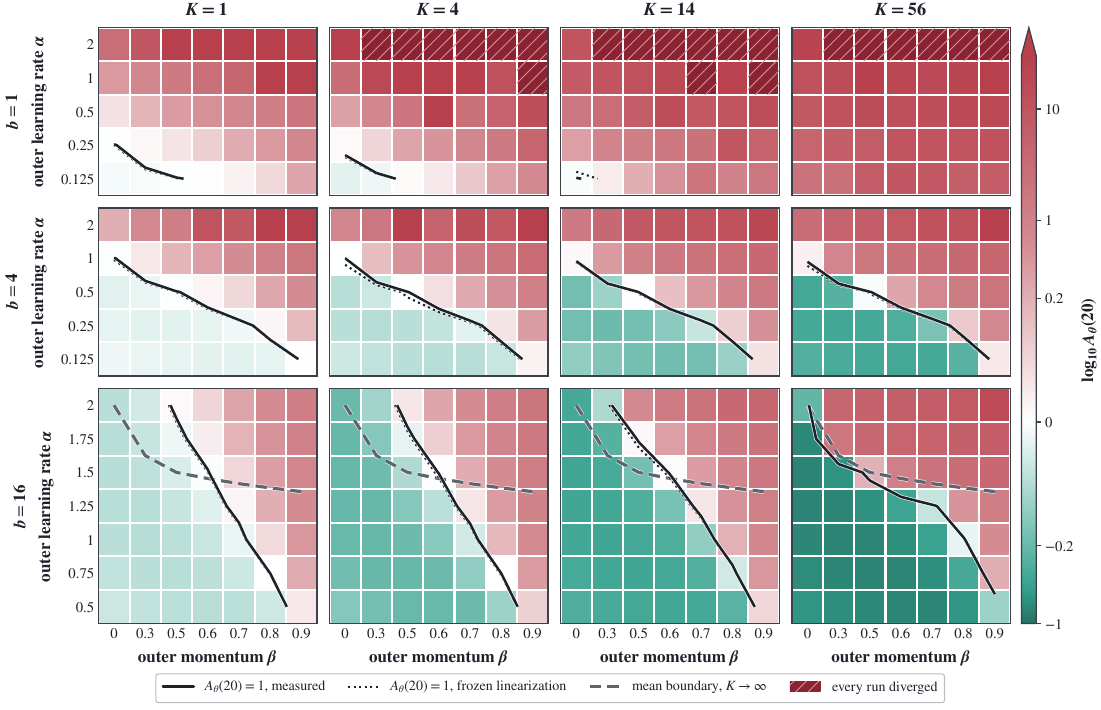}   %
\caption{\textbf{The $(\alpha,\beta)$ plane by inner batch and block length at the CIFAR MLP checkpoint.} $\log_{10}A_\theta(20)$ at step $120{,}000$ with one worker, for $b=1$, $4$ and $16$ (rows) and $K=1$, $4$, $14$ and $56$ (columns). The $b=16$ row samples $\alpha$ from $0.5$ to $2$ in steps of $0.25$. The solid line is the measured boundary $A_\theta(20)=1$, the dotted line that of the frozen linearization, and the gray dashed line in the $b=16$ row the mean condition~\eqref{eq:mean-condition} at $\lambda=1$, $\alpha=2(1+\beta)/(1+2\beta)$, the boundary for $K\to\infty$ (Proposition~\ref{thm:blocklength}(i)). At a hatched setting every run diverged within the $20$ outer steps.}
\label{fig:phase-batch}
\end{figure}

\paragraph{The CIFAR MLP.} The contracting region grows with the number of workers, as in the right panel of Figure~\ref{fig:phase}, and at every $W$ its boundary falls from large $\alpha$ at small $\beta$ to small $\alpha$ at large $\beta$. The frozen linearization has the sign of the measured amplification at every setting. Leaving out the settings whose runs left the checkpoint's basin, the spread cost rises with $\beta$ and with $\alpha$ and falls with $W$.

\paragraph{Inner batch and block length.} Figure~\ref{fig:phase-batch} draws the plane with one worker for inner batches $b=1$, $4$ and $16$ and block lengths $K=1$, $4$, $14$ and $56$, the two quantities varied in the left and middle panels of Figure~\ref{fig:phase}. As $K$ grows, the contracting region shrinks to nothing at $b=1$ and barely moves at $b=4$. At $b=16$ it keeps its size, and by $K=56$ its boundary runs along the mean boundary for $\beta$ up to $0.5$, as in the left panel of Figure~\ref{fig:phase}, where the mean-square boundary approaches a mean boundary that itself moves in as $K$ grows. The frozen linearization has the sign of the measured amplification at almost every setting.

\FloatBarrier

\section{Continuations}
\label{app:continuations}

Figure~\ref{fig:other-spread} continues training from a checkpoint of each network at three outer settings with $K=14$ and $b=4$, for $400$ outer steps on the CIFAR networks and $100$ on Pythia, once with one worker and once with two. It shows whether the amplification measured at the checkpoint reaches the training loss.

\begin{figure}[!htbp]
\centering
\includegraphics[width=\linewidth]{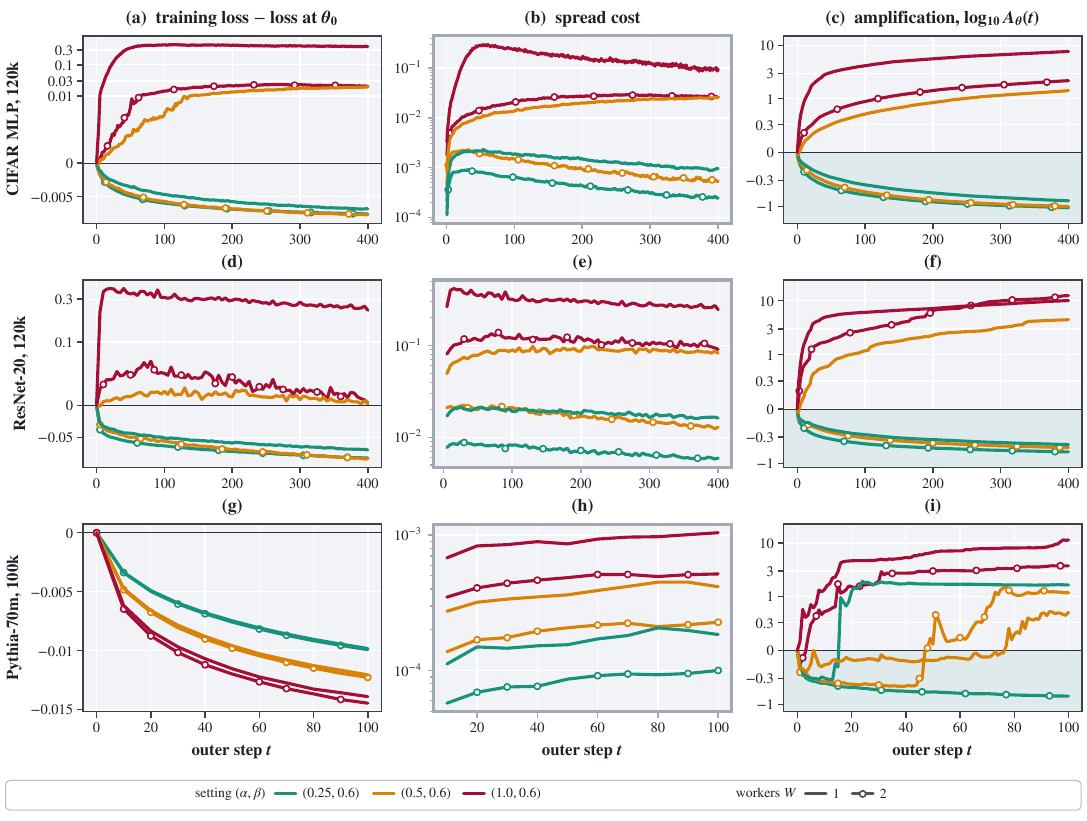}   %
\caption{\textbf{Training loss, spread cost and amplification over a continuation, on the CIFAR MLP (\uline{top}), ResNet-20 (\uline{middle}) and Pythia-70m (\uline{bottom}).} $K=14$ and $b=4$ at three outer settings, with one worker as a plain line and two with hollow circles. The columns show the training loss relative to its value at the checkpoint, the spread cost and $\log_{10}A_\theta(t)$.}
\label{fig:other-spread}
\end{figure}

Where a setting contracts, training carries on as usual. At $(0.25,0.6)$ the loss of both CIFAR networks keeps falling and the spread cost stays small, with one worker and with two.

On the CIFAR networks, amplification reaches the loss through the spread cost. At $(0.5,0.6)$ one worker amplifies, and the mean loss of the runs ends above the loss at the checkpoint, three times above it on the MLP, while the loss at the runs' average keeps falling. The rise comes from the runs moving apart, which is what the spread cost measures.

A second worker undoes this. With two workers the same setting contracts on both networks and the loss falls as in a contracting setting, and averaging over workers lowers the spread cost at every setting, as implication (iii) predicts.

On Pythia the loss falls at every setting over the $100$ outer steps. This is the normal course of training from a checkpoint taken before the end of training, at step $100{,}000$ of $143{,}000$. Over this horizon, the spread cost that the amplification causes stays far below the progress of training.

\FloatBarrier

\section{Examples of What Mean-Stability and Bounded-Noise Analyses Miss}
\label{app:related}

\subsection{Mean-stability analyses}
\label{app:affine-example}

Some analyses keep the mean of the block operator and the covariance of the pseudo-gradient at a reference point but leave out the block noise. Example~\ref{ex:counterexample} shows that these quantities alone cannot decide stability.

\begin{example}[Same observables, opposite stability]
\label{ex:counterexample}
For every $\sigma > 0$ there are two scalar two-step inner laws with
$\E[D_\omega] = 1/2$, $\E[p_\omega] = 0$ and
$\Var(p_\omega) = \sigma^2$, such that outer relaxation with $\alpha = 18/5$
has a finite stationary second moment for one law and an unbounded second moment
for the other.
\end{example}

One law has $D_\omega \equiv 1/2$, and the other has $D_\omega \in \{1/4, 3/4\}$ with equal probability. The two laws have the same mean block operator and the same pseudo-gradient covariance, so no analysis based on those quantities can distinguish them at the reference point. The outer update is relaxation ($\beta = 0$). The sample losses below do not share a minimizer, so the pseudo-gradient is affine rather than linear in the iterate, and the construction lies outside the exact model of Section~\ref{sec:setup-model}.

Set the scalar inner step size to one. At each outer iteration draw fresh
$q_\omega\in(0,1)$ and $\zeta_\omega\in\{-\sigma,+\sigma\}$, independently
of the current iterate. Use the two strongly convex inner losses
\[
f_{0,\omega}(x)=\frac{1-q_\omega}{2}x^2,
\qquad
f_{1,\omega}(x)=\frac{1-q_\omega}{2}x^2+\zeta_\omega x.
\]
The two gradient steps are
\[
x_1=q_\omega x,
\qquad
x_2=q_\omega x_1-\zeta_\omega
=q_\omega^2x-\zeta_\omega,
\]
so
\begin{equation}
P_\omega(x)=x-x_2=D_\omega x+\zeta_\omega,
\qquad D_\omega=1-q_\omega^2.
\label{eq:counterexample-map}
\end{equation}
For the first law take $q_\omega=2^{-1/2}$, hence $D_\omega=1/2$. For the
second law take $D_\omega\in\{1/4,3/4\}$ with equal probability,
equivalently $q_\omega\in\{\sqrt3/2,1/2\}$. In both cases draw
$\zeta_\omega$ independently and with equal probability from
$\{-\sigma,+\sigma\}$. Both laws satisfy
\[
\E[D_\omega]=\frac12,
\quad
\E[P_\omega(x)]=\frac12x,
\quad
\Var(P_\omega(0))=\sigma^2.
\]

Under outer relaxation,
$x_{t+1}=(1-\alpha D_{\omega_t})x_t-\alpha\zeta_{\omega_t}$.
Since $D_{\omega_t}$ and $\zeta_{\omega_t}$ are independent of $x_t$ and of each other, and $\E[\zeta_\omega]=0$, the cross term vanishes and
\begin{equation}
\E[x_{t+1}^2]
=\kappa\E[x_t^2]+\alpha^2\sigma^2,
\qquad
\kappa=\E[(1-\alpha D_\omega)^2].
\end{equation}
At $\alpha=18/5$ the first law has $\kappa=16/25<1$ and the second has
$\kappa=16/25+\alpha^2\Var(D_\omega)=29/20>1$. The first law therefore has
stationary second moment $\alpha^2\sigma^2/(1-\kappa)=36\sigma^2$, and the
second diverges in mean square.

The first law is an instance of Proposition~\ref{prop:bridge}(b) in one dimension. At $\beta=0$ the buffer drops out of the error recurrence, $\Lop$ is multiplication by $\kappa=16/25$, $\Bm$ is $-\alpha=-18/5$ on the error block, $S_p=\sigma^2$, and the floor $(\mathcal I-\Lop)^{-1}(\Bm S_p\Bm^\top)=\alpha^2\sigma^2/(1-\kappa)=36\sigma^2$ is the stationary second moment found above. For the second law $\rho(\Lop)=\kappa=29/20>1$ and the forcing has weight on the only mode, so part (c) gives the divergence.

\subsection{Bounded-noise guarantees}
\label{app:bounded-noise-example}

\citet{khaled2025outer} bound the expected suboptimality of local SGD with an outer learning rate, with heavy-ball momentum, and with Nesterov's accelerated gradient method with round-dependent coefficients. They assume a convex and smooth objective with uniformly bounded gradient variance, and their worst-case guarantees use the gradient noise only through this uniform bound. The stability boundary, however, depends on how the noise of the next block grows with the current error, and the block noise of \eqref{eq:decomposition} describes this growth.

\begin{example}[Same guarantee, opposite local stability]
\label{ex:bounded-noise}
For $c>0$ let $\zeta_c$ equal $-1$ with probability $c/(c+1)$ and $c$ with probability $1/(c+1)$, and let the sample loss and its gradient be
\[
\ell_c(x)=\frac{x^2}{2}+\frac{\zeta_c}{c}\log\cosh(\sqrt c\,x),
\qquad
g_c(x)=\ell_c'(x)=x+\frac{\zeta_c}{\sqrt c}\tanh(\sqrt c\,x).
\]
Then, for every $c$: (i) $\E[g_c(x)]=x$, so the expected objective is $x^2/2$, convex with $1$-Lipschitz gradient and minimizer $0$; (ii) $\Var(g_c(x))=\tanh^2(\sqrt c\,x)\le1$ for every $x$; (iii) every sample loss is convex with minimizer $0$. Consequently the worst-case guarantees of \citet{khaled2025outer} take the same value for $c=1$ and $c=31$ at any common initialization. Yet one SGD step at rate $\eta$ multiplies an error at the minimizer by $1-\eta(1+\zeta_c)$, whose mean is $1-\eta$ for every $c$ while its second moment is
\[
q_c=\E\bigl[(1-\eta(1+\zeta_c))^2\bigr]=1-2\eta+\eta^2(1+c),
\]
equal to $0.78125$ for $c=1$ and $1.25$ for $c=31$ at $\eta=1/8$. The first problem therefore contracts errors in mean square at the minimizer, and the second amplifies them.
\end{example}

\begin{proof}
$\E[\zeta_c]=-c/(c+1)+c/(c+1)=0$ gives (i), and $\E[\zeta_c^2]=c/(c+1)+c^2/(c+1)=c$ gives $\Var(g_c(x))=\E[\zeta_c^2]\tanh^2(\sqrt c\,x)/c$, which is (ii). For (iii), $\ell_c''(x)=1+\zeta_c\operatorname{sech}^2(\sqrt c\,x)$ is at least $1-\operatorname{sech}^2(\sqrt c\,x)\ge0$ when $\zeta_c=-1$ and at least $1$ when $\zeta_c=c$, and $g_c(0)=0$ for both values. At the minimizer $g_c'(0)=1+\zeta_c$, with $\E[1+\zeta_c]=1$ and $\E[(1+\zeta_c)^2]=1+c$, which gives $q_c$.
\end{proof}

For a block of $K$ steps at the minimizer, \eqref{eq:block-moments} gives $d_1(K)=1-(1-\eta)^K$ for both problems and $d_2(K)=1-2(1-\eta)^K+q_c^K$, so the two problems have the same mean block operator at every $K$ but different second moments. Their guarantee for the outer update with a learning rate and no momentum applies here, since its step-size condition $\eta L\bigl(1+(\gamma-1)_+H\bigr)\le1/4$ holds with $L=1$, $\gamma=1$ and $\eta=1/8$. At $(\alpha,\beta)=(1,0)$ the outer step is $\theta_{t+1}=\theta_{K,\omega_t}$ and $\rho(\Lop_K)=q_c^K$, which is $0.37$ for $c=1$ and $2.44$ for $c=31$ at $K=4$, and $0.14$ and $5.96$ at $K=8$.

The same split appears under outer Nesterov momentum with constant coefficients, as in \eqref{eq:snoo}, once the blocks are long enough. At $\alpha=0.15$, $\beta=0.6$ and $\eta=1/8$, $\rho(G)$ of \eqref{eq:snoo-second-moment-matrix} is $0.637$ at $K=4$ and $0.545$ at $K=8$ for $c=1$, and $0.990$ and $1.196$ for $c=31$. Both problems are therefore stable at $K=4$, but the second becomes unstable at $K=8$. This instability is local, because the noise saturates away from the minimizer. It does not contradict that guarantee, which allows an error level that depends on the noise.

\end{document}